%% file: main-arxiv.tex
\documentclass[11pt]{article}
\pdfoutput=1

\usepackage{mathrsfs}
\usepackage{amsmath,amssymb}
\usepackage{bm}
\usepackage{natbib}
\usepackage[usenames]{color}
\usepackage{amsthm}
\usepackage{algorithm}
\usepackage{algorithmic}

\usepackage{multirow} 
\usepackage{enumitem}

\usepackage{microtype}
\usepackage{graphicx}
\usepackage{subfigure}
\usepackage{booktabs,threeparttable,multirow}

\usepackage{amsmath}
\usepackage{amssymb,bm}
\usepackage{mathtools}
\usepackage{caption}
\usepackage{subcaption,color}
\usepackage{indentfirst}
\usepackage{amsfonts}
\usepackage{float,url}
\usepackage{bbm}
\usepackage{lipsum}
\usepackage{nicefrac}
\usepackage{pifont}

\usepackage[colorlinks,
linkcolor=red,
anchorcolor=blue,
citecolor=blue
]{hyperref}

\usepackage{mylatexstyle}
\usepackage[left=1in, right=1in, top=1in, bottom=1in]{geometry}

\newcommand{\disableaddcontentsline}{%
  \let\savedaddcontentsline\addcontentsline 
  \renewcommand{\addcontentsline}[3]{}
}
\newcommand{\enableaddcontentsline}{%
  \let\addcontentsline\savedaddcontentsline
}

\newcommand{\op}{{op}}

\newcommand{\Tr}{\operatorname{Tr}}
\def\Trarg#1{\Tr\left({#1}\right)} 

\usepackage{setspace}
\usepackage{xcolor}

\usepackage{graphicx,eso-pic}
\AddToShipoutPictureFG*{%
  \AtPageUpperLeft{%
    \raisebox{-2cm}{%
      \hspace{2.5cm}
      \includegraphics[width=1cm]{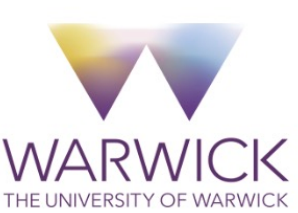}%
      \hspace{0.3cm}%
      \includegraphics[width=0.8cm]{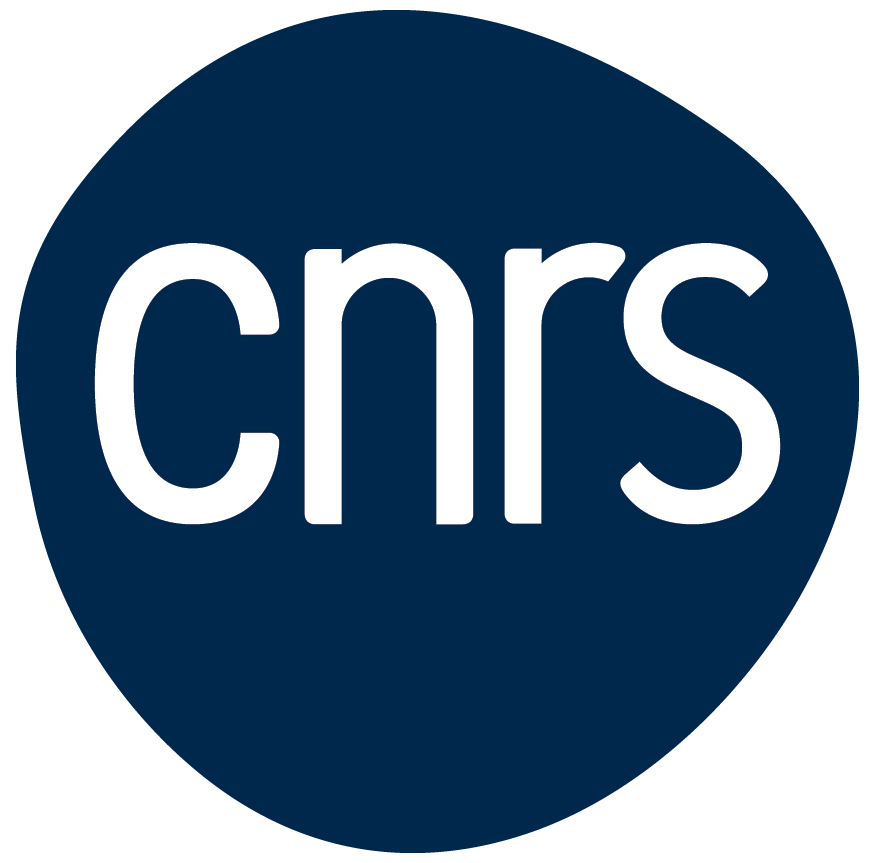}%
      \hspace{0.3cm}%
      \includegraphics[width=0.8cm]{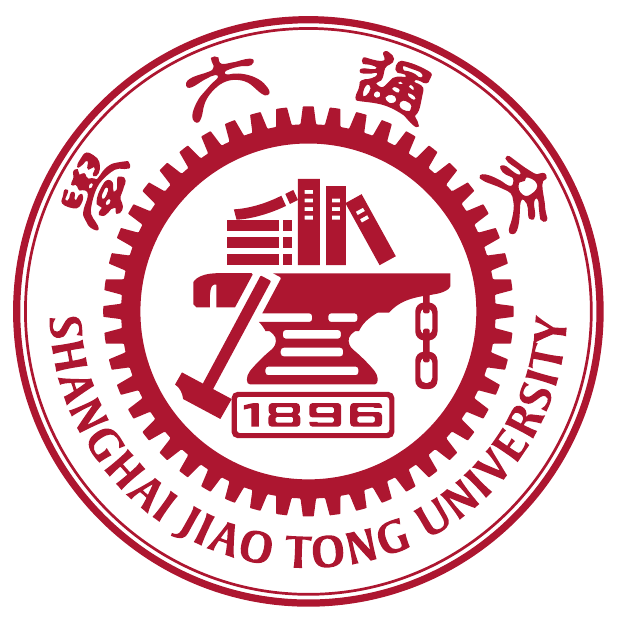}%
    }%
  }%
}

\AddToShipoutPictureFG*{%
  \AtPageUpperLeft{%
    \raisebox{-2.2cm}{%
      \hspace*{\dimexpr 1in+\oddsidemargin\relax}%
      \makebox[\textwidth][l]{\rule{\textwidth}{0.5pt}}
    }%
  }%
}

\title{\textbf{How does feature learning reshape the function space?}}

\author
{
     João Lobo \thanks{Department of Computer Science, University of Warwick, 
        United Kingdom; Email: {\tt joao.lobo-pevidor@warwick.ac.uk} (J.L.), 
        {\tt long.tran-thanh@warwick.ac.uk} (L.T-T.)}
     \and
     Bruno Loureiro\thanks{Departement d'Informatique, \'Ecole Normale Sup\'erieure, PSL \& CNRS; Email: {\tt bruno.loureiro@di.ens.fr}}
     \and
     Long Tran-Than \footnotemark[1]
     \and
     Fanghui Liu\thanks{School of Mathematical Sciences, Institute of Natural Sciences and MOE-LSC, Shanghai Jiao Tong University, China. Part of work was done at Department of Computer Science, and Centre for Discrete Mathematics and its Applications (DIMAP), University of Warwick, United Kingdom; Email: {\tt fanghui.liu@\{sjtu.edu.cn,warwick.ac.uk\}} (Corresponding author)}
}
\date{}

\begin{document}
\disableaddcontentsline

\definecolor{fhcolor}{rgb}{0.523, 0.235, 0.625}
\newcommand{\fh}[1]{\textcolor{fhcolor}{(Fanghui: #1)}}

\newcommand{\jcomment}[1]{\textcolor{olive}{(João: #1)}}

\maketitle

\begin{abstract}
Feature learning is widely regarded as the key mechanism distinguishing neural networks from fixed-kernel methods, yet its impact on the induced function space remains poorly understood. In this work, we precisely characterize how the function space spanned by the features of a two-layer neural network evolves during gradient descent training.
We prove that, in the high-dimensional proportional regime, after a large gradient step the post-update feature distribution is well approximated by a target-dependent spiked Gaussian covariance. This induces a data-adaptive kernel that reshapes the function space and modifies its spectral structure.
Our analysis reveals that feature learning can be interpreted as a distributional transformation in either parameter space or input space, equivalently as the introduction of a target-dependent kernel. In particular, it selectively amplifies eigenvalues aligned with the target direction and mixes leading eigenfunctions, coupling the top radial mode with a target-aligned quadratic harmonic.
%
%
Overall, our results provide a precise function-space perspective on early-stage feature learning: rather than just rescaling a fixed kernel, gradient descent induces a data-adaptive deformation that preferentially enhances directions aligned with the signal in the data.
\end{abstract}

\input{arxiv/intro}

\input{arxiv/setting}
\input{arxiv/main-results}
\input{arxiv/experiments}
\input{arxiv/conclusion}

\bibliography{references}
\bibliographystyle{ims}

\newpage
\appendix
\enableaddcontentsline
\tableofcontents
\newpage

\section{Properties of the Spiked Covariance Matrix}
\input{appendixes/appendix-2nd-moment-matrix}

\section{Technical Lemmas and their proofs}
\input{appendixes/appendix-lemmas}

\section{Proofs of the main results}
\input{appendixes/appendix-main-proofs}

\section{Experimental details}
\input{appendixes/appendix-experiments}

\end{document}

%% file: arxiv/intro.tex
\section{Introduction}
The success of modern neural networks (NNs) is often attributed to feature learning, namely the ability of models to adapt to the structure of the data during training \citep{bach2017breaking,suzuki2019adaptivity,damian2022neural}. This stands in contrast to non-adaptive approaches, such as kernel methods and neural networks operating in the so-called lazy training regime \citep{jacot2018neural,chizat2019lazy}, where features remain effectively frozen at their random, data-independent initialization.

This contrast becomes explicit in random feature models (RFMs) \citep{rahimi2007random} and two-layer NNs, both of which can be written in the form the form $f(\bm x) = \frac{1}{m} \sum_{i=1}^m a_i \phi(\bm x, \bm w_i)$ with a feature map $\phi: \mathcal{X} \times \mathcal{W} \rightarrow \mathbb{R}$. The key difference lies in whether the features are learned: in RFMs, the weights $\{ \bm w_i\}_{i=1}^m \in \mathcal{W}$ are sampled i.i.d. from a probability measure $\mu$ and kept fixed, and only the second layer weights $\bm a :=\{ a_i \}_{i=1}^m$ are optimized, whereas in NNs both layers are jointly trained.

From a function-space perspective, RFMs with $\ell_2$-regularization on $\bm a$ are equivalent to kernel methods via the representer theorem \citep{scholkopf2001generalized}, with empirical kernel $\hat{k}(\bm x, \bm x') = \frac{1}{m} \sum_{i=1}^m \phi(\bm x, \bm w_i)\phi(\bm x', \bm w_i)$,
which approximates, as $m \to \infty$, the population kernel
\[
k_0(\bm x, \bm x') = \int_{\mathcal{W}} \phi(\bm x, \bm w)\phi(\bm x', \bm w)\, \mathrm{d}\mu(\bm w)\,.
\]
The associated reproducing kernel Hilbert space (RKHS), denoted $\mathcal{H}_0(\mu)$, is the closure of functions expressible as weighted averages of the features $\phi(\cdot, \bm w)$.

When the optimization of $\bm a$ is not $\ell_2$-regularized, the induced function space generally differs from an RKHS. For instance, under $\ell_p$-regularization with $1 \leq p < 2$ \citep{celentano2021minimum,chen2023duality}, the function space strictly enlarges relative to $\mathcal{H}_0(\mu)$ by Hölder’s inequality. In the limiting case of $\ell_1$-regularization, one recovers an $\mathcal{F}_1$-type space \citep{bach2017breaking}, closely related to the Barron space $\mathcal{B}$ \citep{barron1993universal,weinan2021barron}. The Barron space can be viewed as the largest function class that two-layer neural networks can learn efficiently in a statistical sense. The relationship between these spaces is clarified by \cite{chen2023duality}, which shows that $\mathcal{B} = \bigcup_{\mu \in \mathcal{P}(\mathcal{W})} \mathcal{H}_0(\mu)$, where $\mathcal{P}(\mathcal{W})$ denotes the set of probability measures on $\mathcal{W}$, establishing a natural connection between RFMs and two-layer neural networks. Notably, under the mean-field regime \citep{chizat2021convergence}, the function space induced by two-layer neural networks forms a subset of the Barron space \citep{wojtowytsch2020can}.

Despite existing characterizations of kernels and neural networks in function space, these perspectives are largely static: they do not explain how training algorithms (e.g., gradient descent) reshape the function space $\mathcal{H}_0(\mu)$ during feature learning. This gap motivates the following question:
\begin{center}
    \emph{How does the kernel (or function space) evolve under gradient updates, and what information is progressively learned?}
\end{center}
In this work, we address this question by providing a precise characterization of the function-space evolution induced by a single large gradient descent step in a two-layer network trained on a Gaussian single-index model. This setting, which has been extensively studied in recent theoretical work on feature learning, has been shown to capture the most general class of functions learnable in the proportional high-dimensional regime \citep{damian2022neural,ba2022high,dandi2023two}. Our analysis reveals how feature learning reshapes the underlying function space and improves its alignment with the target signal. Our contributions are as follows:

\paragraph{1. Approximation by a target-dependent Gaussian distribution:} In \cref{sec:three}, we prove that the post-update feature distribution is well approximated by a target-dependent spiked Gaussian, leading to a data-adaptive kernel $k_1$ (\emph{c.f.} \cref{thm:approx-true-kernel}). It demonstrates that feature learning can be expressed as target-dependent distribution transformation either in the parameter space or in the input-data space on such kernel.

\paragraph{2. Expansion of the data-adaptive kernel around an isotropic kernel:} In \cref{sec:four}, we study the spectrum of the kernel $k_1$. We conduct a Taylor expansion of $k_1$ around an isotropic kernel that isolates the contribution of the spike, and prove that higher-order remainder terms vanish as the dimension grows (\emph{c.f.} \cref{thm:cov-expansion}). In particular, these higher order terms take the form of isotropic kernels coupled with linear and non-linear projections of the input onto the target vector $\bm w^*$. This demonstrates that the role of feature learning is to impose an additional target-dependent kernels, thereby reshaping the function space. 

\paragraph{3. Feature learning in the top eigenspaces for the ReLU activation:} In \cref{sec:relu}, we consider the case of the ReLU activation function, to give an explicit characterization of the kernel $k_1$, its spectrum and the spanned function space. Our results theoretically prove that the spike transforms primarily the top and linear eigenspaces of the operator (\emph{c.f.} \cref{thm:lin-eigenfn},~\ref{thm:approx-top-eigenfn}). Our numerical validations support our theoretical findings and also illustrate the connection between data-adaptive kernels and neural networks.\\


These results provide a precise characterization of how the function space evolves during early training. In particular, they show that this evolution is modulated by the choice of step size, with larger step sizes inducing stronger transformations and increasing the contribution of both linear and higher-order non-linear features. This also bridges between data-adaptive kernels and neural networks, indicating a possibility of exploring proper initialization for feature learning.

\subsection{Related works}

The study of function spaces in deep learning theory stems from RKHS \citep{jacot2018neural_mini} and \citep{lee2017deep_mini}. However, these naturally operate in a lazy training regime \citep{chizat2019lazy_mini}, hence their ability for active feature learning is limited \citep{ghorbani2019limitations_mini}.

Recent works on feature learning in two-layer NNs are exemplified by how parameters change at early stages of training \citep{ba2022high,dandi2023two}. Collectively, these show that following a single step of gradient descent, the updated weight matrix $\bm W_1$ can be approximately decomposed into a deterministic rank-one signal component and a vanishing noise term. Moreover, \cite{moniri2023theory,cui2024asymptotics,dandi2024random} show this low-rank deformation injects informative spikes into the spectrum of the feature matrix $\sigma(\bm W_1^\top \bm x)$, contributing to the alignment towards the target function.
In parallel, \cite{Xu2024NeuralFL} studies feature learning via the \emph{feature geometry} perspective, unifying statistical dependence and feature representations in an inner-product space. However, their focus is on learning optimal features with standard networks rather than understanding how features evolve during training. Furthermore, \cite{dou2021training} studied the conjugate kernel RKHS $\mathcal{H}_t$ and the neural tangent kernel RKHS $\mathcal{K}_t$ through a signed measure. Nevertheless, their framework does not explicitly describe how the function space evolves after certain steps of gradient descent, nor which features are learned through optimization.

%% file: arxiv/setting.tex
\section{Problem Setting and Preliminaries}
\label{sec:setting}

Consider a supervised regression problem with training data $\mathcal{D}=(\bm x_{i}, y_{i})_{i=1}^{n+N}$, which we will assume is drawn from a Gaussian single-index model:
\begin{equation}\label{eq:sim}
    y_i = f^*(\bm x_i)  + \varepsilon_i := g(\langle \bm w^*, \bm x_i \rangle) + \varepsilon_i, \quad \bm x_i \sim \mathcal{N}(\bm 0, \bm I_d)\,,
\end{equation}
where $\bm w^*\in \mathbb S^{d-1}$, $g: \mathbb{R} \rightarrow \mathbb{R}$ is a link function and $\varepsilon_i$ are i.i.d. sub-Gaussian noise random variables with zero mean and variance $\sigma_\varepsilon^2$. This synthetic data model has been the subject of several different works in the theoretical literature, where it was studied as a testbed for separation between lazy and feature learning regimes in the high-dimensional limit \citep{arous2021online, damian2022neural, ba2020generalization, dandi2023two}. In particular, it was shown that despite being efficiently learnable with $n=\Theta(d)$ samples for generic $g$ \citep{li1991sliced,babichev2018,barbier2019optimal, damian2024computational,troiani25a}, non-adaptive kernel methods require infinite data to learn it with arbitrary precision.  
Given the training data, we will consider the problem of learning it with a two-layer neural network defined by $f( \bm x; \bm W, \bm a) = \frac{1}{\sqrt{m}} \sum_{j=1}^m a_j \sigma(\bm w_j^{\top} \bm x)$, where $\bm W \in \mathbb R^{d \times m}, \bm a \in \mathbb R^m$ are the first and second layer weights, respectively, and $\sigma:\mathbb R \to \mathbb R$ is an element-wise activation function. 
We make the following standard assumptions:
\begin{assumption}[Main assumptions]
\label{ass:main-assumptions}~
    \begin{enumerate}[leftmargin=*,topsep=0.5mm, itemsep=0.mm]
        \item(\textbf{Initialization}) At initialization, the first layer weights are distributed as $\bm w^{0}_{j}\sim \mathcal N(0, \frac{1}{d}\bm I_d)$ and the second-layer weights are distributed as $\bm a^{0} \sim \mathcal N(0, \frac{1}{m}\bm I_m)$. 
        \item (\textbf{High-dimensional proportional regime}) We work under the proportional regime scaling, defined as the limit where $n,m,\eta, d\to\infty$ at fixed ratios: \begin{align}\label{eq:joint_limit}
            \alpha \coloneqq \frac{n}{d}, \qquad \beta \coloneqq \frac{m}{d}, \qquad \tilde{\eta}\coloneqq \frac{\eta}{d^\zeta}
        \end{align} where the step size parameter admits $\eta = \Theta(d^{\zeta})$ for some $\zeta \in [1/2, 1)$.
        \item (\textbf{Activation and target functions}) The activation function $\sigma$ is uniformly $L_{\sigma}$-Lipschitz  and $g$ is uniformly bounded and $L_g$-Lipschitz with $\mathbb{E}_{z \sim \mathcal{N}(0,1)}[g(z)]=0$ and $\mu_1 \coloneqq \mathbb{E}_{z\sim \mathcal{N}(0,1)}[g'(z)] \neq 0$. Equivalently, in the language of \cite{arous2021online} we assume $g$ has information exponent $1$.
    \end{enumerate}
\end{assumption}
We employ a two-stage training scheme under the empirical risk minimization via the squared loss as in \cite{ba2022high, dandi2023two, moniri2023theory, cui2024asymptotics, dandi2024random}. Splitting the training data into two parts we have the following framework:
\begin{itemize}[leftmargin=*,topsep=0.5mm, itemsep=0.mm]
    \item Assume we use $n$ i.i.d samples from Eq.~\eqref{eq:sim} to train the first-layer by one single step, while keeping the second-layer fixed:
    \begin{align}
        \label{eq:training_step}
         \bm w_{j}^1 & = \bm w_{j}^0 - \eta \bm g_j^0, \quad \forall j \in [m] \\
         \bm g_j^0 & = \frac{1}{n \sqrt m} \!\sum_{i=1}^n  \left(f(\bm x_i; \bm W^{0},\bm a^{0}) - y_i \right) a_j^0 \bm x_i \sigma^\prime ({\bm w_j^0}^{\top} \bm x_i)\,. \notag
    \end{align}
    \item Given the updated weights $\bm W^{1}$, we update the second-layer weights via ridge regression using another $N$ samples i.i.d from Eq.~\eqref{eq:sim} 
    \begin{align}
        \label{eq:def:main:erm}
        \hat{\bm a}_{\lambda}&=\underset{\bm a\in\mathbb{R}^{m}}{\rm argmin} \sum_{i = 1}^N \left(y_i -f(\bm x_i ;\bm a,\bm W^{1})\right)^{2}+\lambda \| \bm a \|_{2}^{2} =\left(\nicefrac{\bm \Phi^{\top} \bm \Phi}{m}+\lambda I_{m}\right)^{-1}\nicefrac{\bm \Phi^{\top} \bm y}{\sqrt{m}}\,,
    \end{align}
    where the feature matrix $\bm \Phi \in\mathbb{R}^{N\times m}$ with elements $\phi_{i j}=\sigma(\bm x_i^{\top}{\bm w^{1}_{j}})$ and the label vector $\bm y = [y_1, y_2, \cdots, y_N]^{\!\top}$.
\end{itemize}
The effect of training on the first-layer weights is known to depend on the scaling of $\eta$. Considering the data generating process from \cref{eq:sim}, denote by $\bm X \in \mathbb R^{n \times d}$ and $\bm y \in \mathbb R^{n}$ the data matrix and label vector seen during the training step in \cref{eq:training_step}. Defining $\bar{\bm A} := \frac{\mu_1\eta}{\sqrt{m}}\frac{\bm X^\top \bm y \bm a^\top}{n}$, the first-layer weight matrix admits the following description:
\begin{equation}\label{eq:weightapp}
    \bm W_t = \bm W_0 + \bar{\bm A} + \bm E_t\,,    
\end{equation}
where $\bm E_t$ collects asymptotically negligible fluctuations, satisfying $\|\bm E_t\|_{\text{op}}=\tilde{\mathcal O}(1 / \sqrt{d})$.
It holds in the regime $\eta=\Theta(\sqrt{d})$ with any fixed training step $t\in\mathbb{N}$ \citep{ba2022high,wang2024nonlinear}.
Besides, at the first step ($t=1$), this formulation still holds for an intermediate step-size $\eta = \Theta(d^{\zeta})$ with $\zeta \in [1/2, 1)$ in \citep{moniri2023theory} as well as a large step-size $\eta = \Theta(d)$ \citep{cui2024asymptotics,dandi2024random}.
That means, the gradient matrix in \cref{eq:weightapp} can be well approximated by the initial gradient and a rank-one matrix for i) constant gradient steps using small step-size and ii) the first gradient step using an intermediate or large step-size.
Though \cref{eq:weightapp} implies that the gradient matrix has only one spike, the learned feature matrix $\sigma(\bm W_1 \bm X)$ can include more spikes for nonlinear function learning. We will discuss this from the perspective of kernel methods in this paper. 

{\bf Notation:}
We denote vectors in high dimensional spaces by bold lowercase letters ($\bm v$) and matrices/operators by bold uppercase letters ($\bm A$). Functions ($f$) and functional operators ($T$) are represented in standard typeface. We let $\rho_{\mathcal{X}}$ be the measure induced by the input distribution $\mathcal{N}(0, \bm I_d)$, and $L^2(\rho_{\mathcal{X}})$ denote the associated Hilbert function space. The notation $\|\cdot\|$ refers specifically to the standard Euclidean norm in $\mathbb{R}^d$. Whenever an alternative norm is employed, it will be explicitly denoted with the appropriate subscript (e.g., $\|\cdot\|_{L^2(\rho_{\mathcal X})}$ or $\|\cdot\|_\text{op}$). Using standard asymptotic notation $\mathcal{O}, o, \Omega, \Theta$, we track dependencies on $d$ and on the spike strength $B$. We further use the shorthand notation $o_d(f) \coloneqq o(f \cdot d^{-c})$ for some $c \in (0, 1)$ to identify terms that vanish slower than $o(f/d)$.

%% file: arxiv/main-results.tex
\section{Feature learning as a distribution transformation in the kernels}\label{sec:three}

We may now discuss how feature learning can be regarded as a target-dependent distribution transformation.
Under Gaussian initialization, the network implicitly induces the baseline kernel $k_0(\bm{x}, \bm{x}') = \mathbb{E}_{\bm{w} \sim \mathcal{N}(0, \frac{1}{d}\bm{I}_d)}[\sigma(\langle\bm{w}, \bm{x}\rangle)\sigma(\langle\bm{w}, \bm{x}'\rangle)]$ as mentioned before. To capture the network's capacity for feature learning, our analysis isolates the structural deformation of this kernel driven solely by a rank-one update to the weights.

As discussed in \cref{sec:setting}, after one gradient step under the step-size $\eta = \Theta(d^{\zeta})$ with $\zeta \in [1/2, 1)$, we know that $(\bm W_1^\top \bm x)_j = \langle\bm w^0_j, \bm x\rangle +  \langle \bm v_j, \bm x\rangle$ for vectors $\bm v_j \in \mathbb R^d$ deriving from the deterministic update $\bar{\bm A} = \frac{\mu_1\eta}{\sqrt{m}}\frac{\bm X^\top \bm y \bm a^\top}{n}$. Hence the new feature map $\phi \left( \bm x, \bm w + \bm v\right)$ induces the following spiked conjugate kernel, where $\mu$ is the Gaussian measure:
\begin{equation}\label{kernelt}
    k_{1}^*(\bm x, \bm x') = \int_{\mathcal{W}} \phi \left( \bm x, \bm w + \bm v\right)\phi \left( \bm x', \bm w + \bm v\right) \mathrm{d}\mu(\bm w)\,.  
\end{equation}
Note that, $ k_{t}^*$ for constant steps $t$ under the small step-size $\eta=\Theta(\sqrt{d})$ also admits this formulation.
But for a unifying analysis framework, we study the impact of the learning rate coming from the regime \cref{eq:joint_limit} under the first step. 
In the following, we will investigate how the spike $\bm v$ impact the kernel as well as the associated function spaces.

As two positive semi-definite kernels, $k_0$ and $k^*_1$ uniquely determine two different RKHS: $\mathcal H_0, \mathcal H_1^* \subset L^2(\rho_{\mathcal X})$. We let $\rho$ be an unknown probability distribution on $\mathcal{X} \times \mathcal{Y}$ satisfying $\int_{\mathcal{X} \times \mathcal{Y}} y^2 \mathrm{d}\rho(\bm{x}, y) < \infty$, and denote its corresponding marginal distribution over the inputs as $\rho_{\mathcal{X}}$. Following \cite{bach2017equivalence}, we can associate each kernel with a self-adjoint, positive semi-definite, trace-class integral operator. For the two conjugate kernels $k_{0},k_{1}^{\star}$, the operators $T_0,T_1 : L^2(\rho_{\mathcal{X}}) \to L^2(\rho_{\mathcal{X}})$ are given by:
\begin{equation}\label{eq:int-ops}
    (T_0 f)(\bm{x}) = \int_{\mathcal{X}} k_0(\bm{x}, \bm{x}') f(\bm{x}') \mathrm{d}\rho_{\mathcal{X}}(\bm{x}') \quad \text{and} \quad (T_1^* f)(\bm{x}) = \int_{\mathcal{X}} k_1^*(\bm{x}, \bm{x}') f(\bm{x}') \mathrm{d}\rho_{\mathcal{X}}(\bm{x}')\, .
\end{equation}
By the spectral theorem, both operators admit different spectral decompositions, and Mercer's theorem states that we can represent each kernel by its respective spectral decomposition. Therefore, we will study the distribution of the learned features, the formulation of the kernel $k_1^{*}$, as well as the spectrum of function spaces from $\mathcal H_{0}$ to $\mathcal H_{1}^{*}$.

\subsection{Distribution and moments of the learned features}

Let $\bm \varsigma := \sum_{i=1}^n y_i \bm x_i$. The learned feature is given by $\bm z := \bm w + \frac{\mu_1\eta}{n \sqrt{m}}a\bm \varsigma$. Since $a$ is a scalar Gaussian random variable and $\bm w$ is a Gaussian random vector, both independent, $\bm z$ has zero mean. Moreover, conditioned on the dataset (or equivalently conditioned on $\bm \varsigma$) we have:
\begin{equation*}
    \bm z \mid \bm \varsigma \sim \mathcal{N}\!\left(0,\ \frac{1}{d} \bm I_d + \frac{\mu_1^2 \eta^2}{n^2 m^2} \bm \varsigma \bm \varsigma^\top \right)\,.
\end{equation*}
Therefore, conditionally on $\bm \varsigma$, $\bm z$ is a Gaussian vector with a spiked covariance along the random direction $\bm \varsigma$. Without conditioning, the distribution of $\bm z$ is complex. In particular, its covariance matrix is given by 
\begin{equation}\label{eq:covz}
\begin{split}
   \frac{1}{d}\bm \Gamma := \mathbb E[\bm z\bm z^\top] & = \underbrace{\left( 1 + \frac{\eta^2\lambda_{\alpha, \beta}}{d^2}\mathbb E[f^*(\bm x)^2] + \frac{\eta^2\lambda_{\alpha, \beta}}{d^2}\sigma_\varepsilon^2\right)}_{:= A} \frac{\bm I_d}{d} + \underbrace{\left[\frac{\eta^2}{d}\frac{\mu_1^4}{\alpha \beta^2}\Big(\alpha - \frac{1}{d} + \frac{2s}{\mu_1^2d}\Big)\right]}_{:=B} \frac{\bm w^* (\bm w^*)^\top}{d} \,,
\end{split}
\end{equation}
where $\lambda_{\alpha, \beta} := \frac{\mu_1^2}{\alpha \beta^2}$, $\alpha, \beta$ are constants as defined in \cref{eq:joint_limit}, $\mu_1$ is the first Hermite coefficient of $g$ and,  by denoting~$h(t):=[g(t)]^2$, we have
\[
    s \coloneqq \mathbb E[g'(\langle \bm w^*,\bm x\rangle)^2]+ \mathbb E[g(\langle \bm w^*,\bm x\rangle)g''(\langle \bm w^*,\bm x\rangle)] = \frac{1}{2} \mathbb{E}[h''(\langle \bm w^*,\bm x\rangle)]\,.
\]
For short, we write $\bm \Gamma := A\bm I_d + B \bm w^* (\bm w^*)^\top$, and defer the full derivation of these quantities to \cref{app:deriv-2nd-moment}. See the discussion on when $s \gtreqqless 0$ and some examples for particular choices of $g$ in \cref{app:spike-coef}. Lastly, note that $B$ has three terms and $\alpha$ will dominate in the high-dimensional asymptotic regime.
Accordingly, we can ensure $\bm \Gamma$ to be positive definite.

As previously discussed, the (unconditional) distribution of $\bm z$ is complex, but in the next subsection we will show that it can be well approximated by a Gaussian distribution with matching covariance $\bm \Gamma$. This will allow us to derive an approximation for the associated kernel.

\subsection{The kernel formulation}

To tractably analyze this we present our main theorem to approximate the true kernel $k_1^*$ in \cref{kernelt} by a Gaussian kernel $k_1$ governed by $\bm \Gamma$, with the proof deferred to \cref{app:thm-approx-true-kernel}.
\begin{theorem}\label{thm:approx-true-kernel}
    Denote $\bm \varsigma := \sum_{i=1}^n y_i \bm x_i$ and $\bm z := \bm w + \frac{\mu_1\eta}{n \sqrt{m}}a\bm \varsigma$ and let $\mu^*$ be the distribution of $\bm z$. Then, under \cref{ass:main-assumptions}, given the matrix $\bm \Gamma$ defined by \cref{eq:covz}, $k_1^*$ under $\mu^*$ can be approximated by the following kernel 
    \begin{equation}\label{eq:kernel1}
    k_1(\bm x, \bm x') = \mathbb E_{\bm w \sim \mathcal N(0, \frac{1}{d}\bm \Gamma)}[\sigma(\langle \bm w, \bm x\rangle)\sigma(\langle \bm w, \bm x'\rangle)] = k_0(\bm \Gamma^{1/2} \bm x, \bm \Gamma^{1/2}\bm x')\,,
    \end{equation}
    such that their respective integral operators $T_1^*$ and $T_1$ admit
    \[
        \|T_1^* - T_1\|_{\text{HS}} = \|k_1^* - k_1\|_{L^2(\rho_{\mathcal X})\times L^2(\rho_{\mathcal X})} = \mathcal O\left(\frac{\eta^4 \ln^3 d}{d^5}\right)\, .
    \]
    Moreover, defining the pushforward measure of the standard Gaussian measure $\nu = (\rho_{\mathcal X} \circ \bm \Gamma^{-1/2})$ such that $k_0(\bm z, \bm z') = \sum_{i=0}^\infty \omega_i \bm e_i(\bm z)\bm e_i(\bm z')$ with respect to $\nu$, the spectral decomposition of $k_1$ with respect to the original input distribution $\rho_{\mathcal X}$ is given by $k_1(\bm x, \bm x') = \sum_{i=0}^\infty \omega_i \bm e_i(\bm \Gamma^{1/2}\bm x)\bm e_i(\bm \Gamma^{1/2}\bm x')$.
\end{theorem}
\begin{remark}\label{rmk:approx-true-kernel}
    Our result builds upon macroscopic results of order $\mathcal O\left(\frac{|B|}{d}\right)$, by repeatedly showing the trailing terms of our approximations decay with rate $o_d\left(\frac{|B|}{d}\right)$. This bound shows that the approximation of $k_1^*$ via $k_1$ respects the decay needed in order for the result to hold. Indeed, since $|B| = \Theta\left(\frac{\eta^2}{d}\right)$, we have $\mathcal O\left(\frac{\eta^4 \ln^3 d}{d^5}\right) = \mathcal O\left(\frac{|B|^2\ln^3 d}{d^3}\right) \subset o_d\left(\frac{|B|}{d}\right)$. Thanks to \cref{thm:approx-true-kernel}, we can work with $k_1$ throughout this paper. Lastly, note that this result also guarantees an asymptotic approximation in the most aggressive learning rate regime where $\eta = \Theta(d)$.
\end{remark}
\noindent {\bf Data-adaptive kernel:} The relationship between $k_0$ and $k_1$ in \cref{eq:kernel1} can be characterized by a distribution transformation in either parameters or data. From the {\bf parameter-space} perspective, feature learning transforms the parameter distribution from its initial standard Gaussian distribution into a target-dependent distribution shaped by the training objective.
From the {\bf data-space} perspective, feature learning induces a shift in the representation of the input data, moving it toward a structure that is more aligned with the target function or labels.
This can be a data-adaptive kernel, similar in spirit to \cite{follain2024enhanced,Huang2025} with a fixed base kernel with a low-dimensional linear map. 
In contrast, our transformation acts in the full ambient space, leading to anisotropic amplification of signal directions rather than dimensionality reduction.

\section{Feature learning as an additional target-dependent kernel}\label{sec:four}
In this section, we investigate the spectrum of $k_1$ and its dynamics. As a preliminary step, we have the following theorem

\begin{theorem}\label{thm:eigenv-decay}
    Consider the integral operators $T_1$ and $T_0$ defined in \cref{eq:int-ops}. Under \cref{ass:main-assumptions}, for any $\varepsilon > 0$ there exists a truncation radius $R>0$, a constant $C_R > 0$ that depends on $R$ and an absolute constant $c > 0$ such that the $j$-th eigenvalues of the operators satisfy
    \[
        c \lambda_j(T_0) \leq \lambda_j(T_1) \leq C_R \lambda_j(T_0) + \varepsilon\,,
    \]
    for all $k \geq 0$. Consequently, the spectra of both operators exhibit the same asymptotic decay rate.
\end{theorem}

\cref{thm:eigenv-decay} is an important tool because it establishes that $T_1$ does not prematurely deactivate features by forcing eigenvalues to zero. Its formal proof are provided in \cref{app:thm-eigenv-decay}.

\subsection{The spiked covariance expansion framework}
To understand the connection between $k_1$ and $k_0$, we present a general expansion that aims to isolate the contributions of the spike. By performing a Taylor-like series expansion, the influence of the spike is expressed by higher order terms that form the remainder of the expansion. We first give the following expansion, with the proof deferred to \cref{app:thm-cov-expansion}.
\begin{theorem}\label{thm:cov-expansion}
    Let $\bm \Sigma := \gamma_1\bm I + \gamma_2 \bm u \bm u^\top$ be a covariance matrix with $\gamma_1, \gamma_2 > 0$ and $\|\bm u\|= 1$, and \(G:\mathbb R^d \to \mathbb R\) be a measurable function of at most polynomial growth. Denote $D_{\bm u}^{(j)} G$ as the $j$-th order directional derivative of $G$ along $\bm u$ in the sense of tempered distributions, defined by its action on any test function $\varphi \in \mathcal{S}(\mathbb{R}^d)$ as
    \[
        \langle D_{\bm u}^{(j)} G, \varphi \rangle := (-1)^j \langle G, D_{\bm{u}}^{(j)} \varphi \rangle\, , \text{where} \  D_{\bm{u}}\varphi(\bm {w}) = \langle \nabla \varphi(\bm{w}), \bm {u} \rangle\, .
    \]
    Then, we have that
    \[
        \mathbb E_{\bm w \sim \mathcal N(0, \bm \Sigma)}[G(\bm w)] = \mathbb{E}_{\bm w \sim \mathcal N(0, \gamma_1\bm I_d)}[G(\bm{w})] + \sum_{j=1}^\infty \frac{1}{j!}\left(\frac{\gamma_2}{2}\right)^j \mathbb{E}_{\bm w \sim \mathcal N(0, \gamma_1\bm I_d)}[ D_{\bm u}^{(2j)} G(\bm{w})]\,.
    \]    
\end{theorem}

\begin{remark}\label{rmk:well-defined-derivs}
Since we are only interested in using this for $\bm \Gamma$, we note that the conditions on the covariance matrix naturally translate to our case since $A, B > 0$ in the asymptotic regime.
Also note that because the growth of $G$ is polynomially bounded and the Gaussian density belongs to $\mathcal{S}(\mathbb{R}^d)$, the terms $\mathbb{E}_{\bm w \sim \mathcal N(0, \gamma_1\bm I_d)}[ D_{\bm u}^{(2n)} G(\bm{w})]$ are well-defined for all $n \geq 0$.
\end{remark}

\begin{remark}
    This expansion can be intuitively understood as a Taylor series of the scalar function $f(t) = \mathbb{E}_{\bm w \sim \mathcal{N}(0, \bm \Sigma(t))}[G(\bm w)]$ where $\bm \Sigma(t) = \gamma_1\bm{I}_d + t\gamma_2\bm u\bm u^\top$. Iteratively applying Price's Theorem \citep{Price1958,McMahonPrice1964} to obtain the derivatives of $f$ with respect to $t$ reveals that this infinite sum is exactly the formal Taylor series of $f(t)$ around $t = 0$. 
\end{remark}

\subsection{Expansion of the updated kernel $k_1$}
\label{sec:expansion}
To bring \cref{thm:cov-expansion} into our context, for a fixed pair $(\bm x, \bm x')$, we define the function $G(\bm w) := G_{\bm x, \bm x'}(\bm w) = \sigma(\langle \bm w, \bm x \rangle) \sigma(\langle \bm w, \bm x' \rangle)$. Given \cref{rmk:well-defined-derivs} and since $\bm \Gamma$ is positive definite, setting $\gamma_1 := \frac{A}{d}$ and $\gamma_2 := \frac{B}{d}$, we apply the theorem to obtain the expansion:
\[
    k_1(\bm x, \bm x') = \mathbb{E}_{\bm w \sim \mathcal N(0, \frac{1}{d}\bm \Gamma)}[ G(\bm{w})] = \mathbb{E}_{\bm w \sim \mathcal N(0, \frac{A}{d}\bm I_d)}[ G(\bm{w})]+ \sum_{j=1}^\infty \frac{1}{j!}\left(\frac{B}{2d}\right)^j \mathbb{E}_{\bm w \sim \mathcal N(0, \frac{A}{d}\bm I_d)}[ D_{\bm w^*}^{(2j)} G(\bm{w})]\, .
\]
In this series, the $j=0$ term corresponds to the scaled isotropic kernel 
\begin{equation*}
  k_0^{(A)} (\bm x, \bm x') \coloneqq \mathbb{E}_{\bm w \sim \mathcal N(0, \frac{A}{d}\bm I_d)}[ \sigma(\langle \bm w, \bm x \rangle) \sigma(\langle \bm w, \bm x' \rangle)] \,,  
\end{equation*}
and the $j=1$ term admits an exact expression by applying Leibniz rule
\begin{align}\label{eq:1st-order-term}
    S(\bm x, \bm x') := \mathbb{E}_{\bm w \sim \mathcal N(0, \frac{A}{d}\bm I_d)}[ D_{\bm w^*}^{(2)} G(\bm{w})] &= \langle \bm w^*, \bm{x}' \rangle^{2}\mathbb E[\sigma(\langle\bm w, \bm x\rangle)\sigma^{(2)}(\langle\bm w, \bm x'\rangle)] \nonumber \\
    &\quad + 2\langle \bm w^*, \bm{x} \rangle\langle \bm w^*, \bm{x}' \rangle\mathbb E[\sigma'(\langle\bm w, \bm x\rangle)\sigma'(\langle\bm w, \bm x'\rangle)] \\
    &\quad + \langle \bm w^*, \bm{x} \rangle^{2}\mathbb E[\sigma^{(2)}(\langle\bm w, \bm x\rangle)\sigma(\langle\bm w, \bm x'\rangle)] \nonumber\, .
\end{align}
Hence we can isolate the effect of the spike by taking $k_1$ as a first-order perturbation of $k_0^{(A)}$ with
\begin{equation}\label{eq:first-order-k1-exp}
    k_1(\bm x, \bm x') = k_0^{(A)}(\bm x, \bm x')+ \frac{B}{2d}S(\bm x, \bm x')+ R(\bm x, \bm x')\, ,    
\end{equation}
where $S(\bm x, \bm x')$ is defined by \cref{eq:1st-order-term}, and $R(\bm x, \bm x')$ are the residual terms formally defined as the tail of the expansion for $j \geq 2$:
\begin{equation}\label{eq:exp-remainder}
    R(\bm x, \bm x') = \sum_{j=2}^{\infty} \frac{1}{n!} \left(\frac{B}{2d}\right)^j \mathbb{E}_{\bm w \sim \mathcal{N}(0, \frac{A}{d}\bm I_d)}[D_{\bm w^*}^{(2j)}G(\bm w)]\, .
\end{equation}
\cref{ass:main-assumptions} ensures $k_1$, $k_0^{(A)}$, and $S$ grow at most polynomially, thus $R$ defines a bounded integral operator on the Gaussian space, and is dominated by the first terms of the expansion by the following lemma, with the full proof available in \cref{app:lemma-T_R-op-bound}.
\begin{lemma}\label{lemma:T_R-op-bound}
    Under \cref{ass:main-assumptions} for the function $R$ defined by \cref{eq:exp-remainder} with bounded integral operator $T_R f(\bm x) = \int R(\bm x, \bm x') f(\bm x') \mathrm d \rho_{\mathcal X}(\bm x')$, we have that $\|T_R\|_{\text{op}} = \mathcal O\left(\frac{B^2}{d^2}\right)$.
\end{lemma}
\begin{remark}
    As verified by \cite{moniri2023theory}, the feature matrix receives an increasing number of non-linear ``spikes" as $\zeta \to 1$. In our case, every higher order derivative term introduces non-linear projections onto $\bm w^*$, so we expect a similar effect on the kernel as a function of $\zeta$ since this will make $\frac{|B|}{d} \to \Theta(1)$. Even though the importance of these terms grows, $\zeta < 1$ implies $\left(\frac{|B|}{d}\right)^n = o_d\left(\frac{|B|}{d}\right)$ for all $n > 1$, making $S$ the dominant term in every scenario.
\end{remark}

\section{Example under the ReLU activations}
\label{sec:relu}

To make the previous discussion on $k_1$ concrete, we now look closer to the particular case of a ReLU network, giving an explicit characterization of its eigenfunctions as well as numerical illustration. Defining the warped cosine similarity $\gamma_{\bm \Gamma} \coloneqq \frac{\bm x^\top \bm \Gamma \bm x'}{\sqrt{(\bm x^\top \bm \Gamma \bm x)(\bm x'^\top \bm \Gamma \bm x')}}$, we can write $k_{1}$ for the ReLU activation explicitly:
\begin{equation}\label{eq:relu-k1}
    k_1(\bm x, \bm x') = \frac{\sqrt{(\bm x^\top \bm \Gamma \bm x)(\bm x'^\top \bm \Gamma \bm x')}}{2\pi d}\left[\gamma_{\bm \Gamma}(\pi -\arccos (\gamma_{\bm \Gamma})) +\sqrt{1 - \gamma_{\bm \Gamma}^2}\right]\,.
\end{equation}
The calculation leverages the coordinate transformation trick \citep[Appendix A]{liao2018spectrum} and we omit the details here. In the next segment, we discuss the results from \cref{sec:four} applied to the ReLU case.

\subsection{Specialization of the expansion for the ReLU kernel}
First, we consider the expansion of the kernel for the ReLU activation, and determine the first two terms from \cref{eq:first-order-k1-exp}.
When $\sigma(t) = \max(0, t)$, the scaling effect is captured by the identity $k_0^{(A)}(\bm x, \bm x') = A k_0(\bm x, \bm x')$. To determine $S$, if we let $\theta_{\bm x, \bm x'}$ be the angle between $\bm x$ and $\bm x'$, we know that $\sigma'' = \delta$ in the distributional sense, thus
\[
    \mathbb E_{\bm w \sim \mathcal N(0, \frac{A}{d}\bm I_d)}[\sigma''(\langle \bm w, \bm x\rangle)\sigma(\langle \bm w, \bm x'\rangle)] = \frac{\|\bm x'\|}{\|\bm x\|}\frac{\sin \theta_{\bm x, \bm x'}}{2\pi} \quad \text{and} \quad \mathbb E[\sigma(\langle \bm w, \bm x\rangle)\sigma''(\langle \bm w, \bm x'\rangle)] = \frac{\|\bm x\|}{\|\bm x'\|}\frac{\sin \theta_{\bm x, \bm x'}}{2\pi}\, .
\]
Also, $\sigma'(t) = \bm 1_{\{t\geq 0\}}$ almost everywhere so 
\[
    \mathbb E_{\bm w \sim \mathcal N(0, \frac{A}{d}\bm I_d)}[\sigma'(\langle \bm w, \bm x\rangle)\sigma'(\langle \bm w, \bm x'\rangle)] = \frac{\pi - \theta_{\bm x, \bm x'}}{2\pi}\, .
\]
Note that the expectation is always taken under $\mathcal N(0, \frac{A}{d}\bm I_d)$, however that does not affect the result since for every $c > 0$ we have $\sigma(ct) = c \sigma(t)$, $ \sigma'(ct) = \bm 1_{\{ct \geq 0\}} = \sigma'(t) $ and $\delta(ct) = \frac{\delta(t)}{|c|}$.

Therefore, the first-order term in the ReLU case is given by
\begin{equation}\label{eq:S-relu}
    S(\bm x, \bm x') = \frac{\pi - \theta_{\bm x, \bm x'}}{\pi }[\langle \bm x, \bm w^*\rangle\langle \bm x', \bm w^*\rangle] + \frac{\sin \theta_{\bm x, \bm x'}}{2\pi}\left(\frac{\|\bm x'\|}{\|\bm x\|}\langle \bm x, \bm w^*\rangle^2 + \frac{\|\bm x\|}{\|\bm x'\|}\langle \bm x', \bm w^*\rangle^2\right)\,,
\end{equation}
and the ReLU $k_1$ kernel is
\begin{align*}
    k_1(\bm x, \bm x') =& Ak_0(\bm x, \bm x') + \frac{B}{2d}\frac{(\pi - \theta_{\bm x, \bm x'})}{\pi }[\langle \bm x, \bm w^*\rangle\langle \bm x', \bm w^*\rangle] \\
    &+ {\frac{B}{2d}} \frac{\sin \theta_{\bm x, \bm x'}}{2\pi}\left(\frac{\|\bm x'\|}{\|\bm x\|}\langle \bm x, \bm w^*\rangle^2 + \frac{\|\bm x\|}{\|\bm x'\|}\langle \bm x', \bm w^*\rangle^2\right) + R(\bm x, \bm x')\, .
\end{align*}
We can see that this expansion is dominated by the original kernel $k_0$. As a result, the original isotropic eigenfunctions continue to play a fundamental role in shaping the geometry of the new function space. With that in mind, the following lemma details the spectral basis of $T_0$, which we will use to approximate the eigenfunctions of the new operator. 
\begin{lemma}\label{lemma:iso-eigenfn}
    The normalized eigenfunctions of the integral operator $T_0$ when $\sigma(t) = \max(0, t)$ are strictly of the form $\psi_{k,m}(\bm x) = \frac{\|\bm x\|}{\sqrt d}Y_{k,m}(\bm \omega)$, where $ \bm \omega = \frac{\bm x}{\|\bm x\|}$ and $Y_{k,m}$ are the orthonormal spherical harmonics on $\mathbb S^{d-1}$.
\end{lemma}
\paragraph{Approximating the action of $S$:} \cref{lemma:T_R-op-bound} establishes that the effect of the spike onto the new kernel is driven by $S$. However, the terms within $S$ have complex interactions governed by the angles of both input vectors. To circumvent this, we leverage the concentration properties of the Gaussian measure to approximate the action of its integral operator, with proof deferred to \cref{app:lemma-first-order-op}. 
\begin{lemma}\label{lemma:approx-first-order-op}
    Consider the function from \cref{eq:S-relu} with integral operator $T_S: L^2(\rho_{\mathcal X}) \to L^2(\rho_{\mathcal X})$ defined by $T_S f(\bm x) = \int S(\bm x, \bm x') f(\bm x') \mathrm d \rho_{\mathcal X}(\bm x')$. Let $\{\psi_{k, m}\}$ be the eigenbasis of $T_0$ and $f$ be a normalized function that can expressed in that basis. If we define the operators $ T_{S}^{(1*)}, T_{S}^{(2*)} : L^2(\rho_{\mathcal X}) \to L^2(\rho_{\mathcal X})$ by
    \[
        T_{S}^{(1*)} f(\bm x) = \frac{\langle \bm x, \bm w^*\rangle}{2}\langle\langle \cdot, \bm w^*\rangle,f\rangle \quad \text{and} \quad
        T_{S}^{(2*)} f(\bm x) = \frac{1}{2\pi}\left(\frac{\langle \bm x, \bm w^*\rangle^2}{\|\bm x\|} \langle \|\cdot\|, f\rangle + \|\bm x \| \left\langle \frac{\langle \cdot, \bm w^*\rangle^2}{\|\cdot\|}, f \right\rangle \right)\, ,
    \]
    then, under \cref{ass:main-assumptions}, we have 
    \[
        T_S f(\bm x) = T_{S}^{(1*)} f(\bm x) + T_{S}^{(2*)} f(\bm x) + E(\bm x)\, ,
    \] such that $\|E\|_{L^2(\rho_{\mathcal X})} = o_d(1)$.
\end{lemma}
This result characterizes the action of $S$ in terms of explicit actions onto the function space, given by the projections present in $T_S^{(1*)}$ and $T_S^{(2*)}$. Most notably, these projections can be exclusively defined in terms of the basis from \cref{lemma:iso-eigenfn}, allowing us to track how they combine to shape the new functional geometry.

\paragraph{Emergence of feature learning in the top eigenspaces:} Through the explicit form of the ReLU kernel, we can compute the action of the operator $T_1$ on linear functions. Solving the eigenvalue problem analytically for linear functions reveals a clean orthogonal splitting, as shown by the next theorem. We see the linear function aligned with $\bm w^*$ receiving a selective boost to its eigenvalue, while the linear functions that are orthogonal to $\bm w^*$ remain unaffected by the spike. The proof for this is available in \cref{app:thm-lin-eigenfn}.
\begin{theorem}[Feature learning in the linear eigenspace]\label{thm:lin-eigenfn}
    Under \cref{ass:main-assumptions}, the function $\psi_{*}(\bm x) = \langle \bm x, \bm w^*\rangle$ is an eigenfunction of $T_1$ with eigenvalue $\lambda_{\psi_*}(T_1) = A\lambda_{\psi_*}(T_0) + \frac{B}{4d}$. Furthermore, for any function $\psi_{\perp}(\bm x) = \langle \bm x, \bm v\rangle$ such that $\bm v \perp \bm w^*$, we have that $\psi_{\perp}$ is an eigenfunction of $T_1$ with eigenvalue $\lambda_{\psi_\perp}(T_1) = A \lambda_{\psi_\perp}(T_0)$.
\end{theorem}

\cref{thm:lin-eigenfn} implies that the other eigenfunctions of $T_1$ must be orthogonal to the linear function. Furthermore, while the top eigenfunction of $T_0$ is directly related to the constant harmonic $Y_0$, the terms in $S$ strongly interact with the degree-2 zonal harmonic $Y_2$. This results in a superposition of $Y_0$ and $Y_2$ within the new top eigenspace, as detailed in the following theorem (see the proof in \cref{app:thm-approx-top-eigenfn}). 

\begin{theorem}[Feature learning in the top eigenspace]\label{thm:approx-top-eigenfn}
    Let $\Psi$ be the top eigenfunction of the integral operator $T_1$ with associated eigenvalue $\lambda_{\max}(T_1)$. Also, define $\bm \omega = \frac{\bm x}{\|\bm x\|}$ and consider the functions given by $Y_0(\bm \omega) = 1$ and $\hat{Y}_2(\bm \omega, \bm w^*) = \langle \bm \omega, \bm w^*\rangle^2 - \frac{1}{d}$ such that $T_0 \Big[\|\bm x\| Y_0(\bm \omega)\Big] = \lambda_{\max}(T_0)\Big[\|\bm x\| Y_0(\bm \omega)\Big]$ and $T_0 \Big[\|\bm x\| \hat{Y}_2(\bm \omega, \bm w^*)\Big] = \lambda_{2}(T_0)\Big[\|\bm x\| \hat{Y}_2(\bm \omega, \bm w^*)\Big]$. We define the approximate eigenvalue and eigenfunction as
    \begin{equation}\label{eq:approx-top-eigenfn}
        \tilde{\lambda} = A\lambda_{\max}(T_0) + \frac{B}{2\pi d} + o_d\left(\frac{|B|}{d}\right) \,, \quad     
        \tilde{\Psi}(\bm x) = \frac{1}{\sqrt{N}}\frac{\|\bm x\|}{\sqrt d}\Big[ Y_0(\bm \omega)+ \tau Y_2(\bm \omega, \bm w^*) \Big] \, ,
    \end{equation}
    where $N$ is a normalization constant, $\tau := \tau(B) = \frac{1}{d}\sqrt{\frac{2d - 2}{d+2}}\left[\frac{B}{4 \pi (\tilde{\lambda} -A\lambda_2(T_0))}\right]$ and $Y_2 = \frac{\hat{Y}_2}{\| \hat{Y}_2 \|}$ is the normalized quadratic zonal harmonic.
    Then, under \cref{ass:main-assumptions}, we have 
    \[
        T_1\tilde{\Psi} = \tilde{\lambda}\tilde{\Psi} + e
    \] where $\|e\|_{L^2(\rho_{\mathcal X})} = o_d\left(\frac{|B|}{d}\right)$.
    Also, $\|\Psi - \tilde{\Psi}\|_{L^2(\rho_{\mathcal X})}  = o_d\left(\frac{|B|}{d}\right)$ and $|\lambda_{\max}(T_1) - \tilde{\lambda}| = o_d\left(\frac{|B|}{d}\right)$.
\end{theorem}
\begin{remark}
    {Theorems \ref{thm:lin-eigenfn} and \ref{thm:approx-top-eigenfn} show the boost to the linear eigenfunction is of the same order of the boost to the top eigenvalue and of the coefficient of $Y_2$ in the top eigenfunction. This conforms with the empirical findings demonstrated in \cite{moniri2023theory}, where the effects of the quadratic and the linear terms are introduced with similar strength on the spectrum of the feature matrix.}
\end{remark}

%% file: arxiv/experiments.tex
\subsection{Experiments}
\begin{figure}[htb]
    \centering
    \subfigure[Feature alignment]{\label{fig-f}
        \centering
        \includegraphics[width=0.35\textwidth, keepaspectratio]{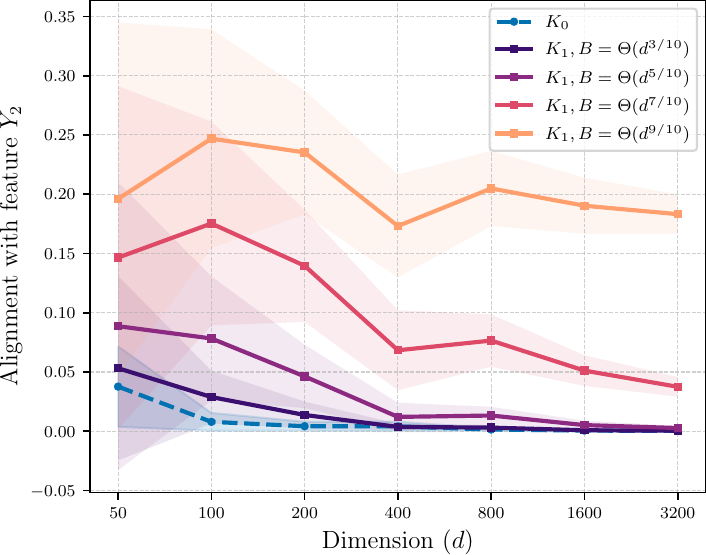}
    }
    \hspace{0.5cm}
    \subfigure[Test MSE]{\label{fig-t}
        \centering
        \includegraphics[width=0.35\textwidth,keepaspectratio]{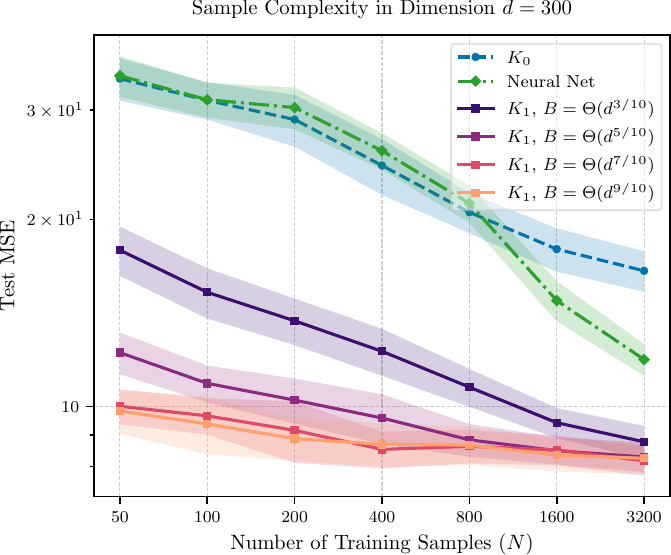}
    }
    \caption{(a) Alignment with the directional feature $Y_2$. (b) Generalization performance at $d=300$.}
    \label{fig:exp}
\end{figure}

Here we provide numerical results to validate and understand our formal results. Over $10$ trials, with variance shown as shaded areas around the curves, we sample $\bm{Z} \in \mathbb{R}^{N \times d}$ with $\bm{z}_i \sim \mathcal{N}(0, \bm{I}_d)$, and compare the ReLU kernel $k_0$, the ReLU kernel $k_1$ from \cref{eq:relu-k1} and a two-layer ReLU MLP (width 400). In \cref{fig-f}, we compute the alignment $\langle \bm{v}_i^{\text{top}}, Y_2(\bm{Z}) \rangle$, where $\bm{v}_i^{\text{top}}$ is the lead eigenvector of the kernel matrices for $i \in \{0, 1\}$. As expected, the alignment for $k_1$ grows with $B$, while $k_0$ always remains near-zero. In \cref{fig-t} we track the Mean Squared Error ($d = 300$) on a test set of $600$ samples when learning $g(t) = 2t^2 + 3t + 4\sin(2t)$, comparing Kernel Ridge Regression and the network across varying training sample sizes $(N)$. We observe the network converging toward the spiked kernel’s performance; noting that while the kernels have privileged access to $\bm w^*$, the network must recover it from the data. Refer to \cref{app:experimental_details} for full details of the experimental setup.

%% file: arxiv/conclusion.tex
\section{Conclusion}
In this work, we studied how deterministic updates of first-layer weights in a two-layer network shape the induced function space. With a data-dependent kernel we showed how the update can be expressed as a distribution shift of the original kernel. We expanded the shifted kernel to reveal mixing between eigenfunctions and selective changes to eigenvalues, favoring functions aligned with the target vector, and determined that this transformation is governed by the scaling of step size.  Our results suggest that working under a distribution shift could be a way to ``skip" the early phases of training of these networks. Future research includes analyzing the higher-order terms in our expansion, particularly in the regime $\eta = \Theta(d)$, where they appear to remain relevant.

%% file: appendixes/appendix-2nd-moment-matrix.tex
\allowdisplaybreaks

\subsection{Derivation of the second moment matrix}\label{app:deriv-2nd-moment}

In the random design, we have $y_i = f^*(\bm x_i) +\varepsilon_i$ and $\bm x_i \sim N(0, \bm I_d)$ where $\varepsilon_i$ are i.i.d. sub-Gaussian noise random variables with zero mean and variance $\sigma_\varepsilon^2$. This way $\bm \varsigma = \sum_{i=1}^n y_i\bm x_i$ is random, the random vector $\bm z = \bm w + \frac{\mu_1\eta}{n\sqrt m}a\bm \varsigma$ has mean $0$ and covariance 
\[
    \mathbb E[\bm z \bm z^\top] = \mathbb E[\bm w \bm w^\top] + \frac{\mu_1^2 \eta^2}{n^2m}\mathbb E[a^2 \bm \varsigma \bm \varsigma^\top]\, .
\]
Since $a \sim \mathcal N(0, \frac{1}{m})$ independently, we have 
\[
    \mathbb E[a^2 \bm \varsigma \bm \varsigma^\top] = \frac{1}{m}\mathbb E[\bm \varsigma \bm \varsigma^\top] = \frac{1}{m}\Big\{ n\mathbb E[y^2 \bm x \bm x^\top] + n(n-1)\mathbb E[y\bm x] \mathbb E[y \bm x^\top] \Big\}\, .
\]
Using that $y = f^*(\bm x) + \varepsilon$ we have
\[
    \mathbb E[y^2 \bm x \bm x^\top] = \mathbb E[f^*(\bm x)^2 \bm x \bm x^\top] + \sigma_{\varepsilon}^2 \bm I_d\, .
\]

Thus,
\[
    \mathbb E[a^2 \bm \varsigma \bm \varsigma^\top] = \frac{1}{m}\Big(n \Big \{\mathbb E[f^*(\bm x)^2 \bm x \bm x^\top]  + \sigma_{\varepsilon}^2\bm I_d\Big\}+ n(n-1)\mathbb E[y\bm x] \mathbb E[y \bm x^\top]\Big )\, .
\]
Hence, the covariance of $\bm z$ is given by the second moment matrix
\begin{align*}
    \mathbb E[\bm z \bm z^\top] &= \mathbb E_{(\bm w, \bm x, a, \varepsilon)}\left[\left(\bm w + \frac{\mu_1\eta}{n \sqrt{m}}a \bm \varsigma\right) \left(\bm w + \frac{\mu_1\eta}{n \sqrt{m}}a \bm \varsigma\right)^\top\right]\\
    &= \frac{1}{d} \bm I_d + \frac{\mu_1^2\eta^2}{n^2 m^2}\Big(n \Big \{\mathbb E[f^*(\bm x)^2 \bm x \bm x^\top]  + \sigma_{\varepsilon}^2\bm I_d\Big\}+ n(n-1)\mathbb E[y\bm x] \mathbb E[y \bm x^\top]\Big )\\
    &= \frac{1}{d} \bm I_d + \frac{\mu_1^2\eta^2}{n m^2}\Big(\Big \{\mathbb E[f^*(\bm x)^2 \bm x \bm x^\top]  + \sigma_{\varepsilon}^2\bm I_d\Big\}+ (n-1)\mathbb E[y\bm x] \mathbb E[y \bm x^\top]\Big )\, .
\end{align*}
Finally, using $n = \alpha d$ and $m = \beta d$, we can rewrite this as
\begin{align*}
    &= \frac{1}{d} \bm I_d + \frac{\mu_1^2\eta^2}{\alpha \beta^2 d^3}\Big(\Big \{\mathbb E[f^*(\bm x)^2 \bm x \bm x^\top]  + \sigma_{\varepsilon}^2\bm I_d\Big\}+ (\alpha d-1)\mathbb E[y\bm x] \mathbb E[y \bm x^\top]\Big )\\
    &= \frac{1}{d} \left(\bm I_d + \frac{\mu_1^2\eta^2}{\alpha \beta^2 d^2}\mathbb E[f^*(\bm x)^2 \bm x \bm x^\top]  + \frac{\mu_1^2\eta^2}{\alpha \beta^2 d^2}\sigma_{\varepsilon}^2\bm I_d + \frac{\mu_1^2\eta^2}{\alpha \beta^2 d^2}(\alpha d-1)\mathbb E[y\bm x] \mathbb E[y \bm x^\top] \right)\\
    &= \frac{1}{d} \left(\bm I_d + \frac{\mu_1^2\eta^2}{\alpha \beta^2 d^2}\mathbb E[f^*(\bm x)^2 \bm x \bm x^\top]  + \frac{\mu_1^2\eta^2}{\alpha \beta^2 d^2}\sigma_{\varepsilon}^2\bm I_d + \frac{\mu_1^2\eta^2}{\alpha \beta^2 d}\left(\alpha -\frac{1}{d}\right)\mathbb E[y\bm x] \mathbb E[y \bm x^\top] \right)\\
    &= \frac{1}{d} \left(\bm I_d + \frac{\eta^2\lambda_{\alpha, \beta}}{d^2}\mathbb E[f^*(\bm x)^2 \bm x\bm x^\top] + \frac{\eta^2\lambda_{\alpha, \beta}}{d^2}\sigma_\varepsilon^2 \bm I_d  + \frac{\eta^2\lambda_{\alpha, \beta}}{d}\Big(\alpha - \frac{1}{d}\Big)\mathbb E[f^*(\bm x)\bm x]\mathbb E[f^*(\bm x)\bm x^\top]\right)\, ,
\end{align*}
where $\lambda_{\alpha, \beta} := \frac{\mu_1^2}{\alpha \beta^2}$ and $\alpha, \beta$ are the same constants as defined in \cref{eq:joint_limit}. By using Stein's Lemma the last term simplifies to
\[
    \mathbb E[f^*(\bm x)\bm x]\mathbb E[f^*(\bm x)\bm x^\top] = \mathbb E[g'(\langle \bm w^*, \bm x\rangle)]^2\bm w^* (\bm w^*)^\top = \mu_1^2\bm w^* (\bm w^*)^\top\, .
\]

We can further rewrite the second term by letting $h(\boldsymbol{x}) = f^*(\boldsymbol{x})^2\boldsymbol{x} \in \mathbb{R}^d$ such that $J_h \in \mathbb{R}^{d \times d}$, and applying Stein's Lemma:
\begin{align*}
    \mathbb{E}[f^*(\boldsymbol{x})^2\boldsymbol{x}\boldsymbol{x}^\top] &= \mathbb{E}[h(\boldsymbol{x})\boldsymbol{x}^\top] \\
        &= \mathbb{E}[J_h(\boldsymbol{x})] = \mathbb{E}[J_{\{f^*(\boldsymbol{x})^2\boldsymbol{x}\}}] \\
        &= \mathbb{E}[\boldsymbol{x}(\nabla f^*(\boldsymbol{x})^2)^\top] + \mathbb{E}[f^*(\boldsymbol{x})^2]\boldsymbol{I}_d \\
        &= 2\mathbb{E}[g(\langle\boldsymbol{w}^*,\boldsymbol{x}\rangle)g'(\langle\boldsymbol{w}^*,\boldsymbol{x}\rangle)\boldsymbol{x}](\boldsymbol{w}^*)^\top + \mathbb{E}[f^*(\boldsymbol{x})^2]\boldsymbol{I}_d.
\end{align*}
In particular, if $g$ is smooth enough, we can apply the same idea to obtain 
\begin{align*}
    \mathbb E[f^*(\bm x)^2 \bm x\bm x^\top] &= 2\mathbb E[\nabla f^*(\bm x) \nabla f^*(\bm x)^\top] + 2\mathbb E[f^*(\bm x) \bm H_{f^*}(\bm x)] + \mathbb E[f^*(\bm x)^2]\bm I_d\\
    &= 2\{\mathbb E[g'(\langle \bm w^*,\bm x\rangle)^2]+ \mathbb E[g(\langle \bm w^*,\bm x\rangle)g''(\langle \bm w^*,\bm x\rangle)]\}\bm w^*(\bm w^*)^\top + \mathbb E[f^*(\bm x)^2]\bm I_d\\
    &\coloneqq 2 s \bm w^*(\bm w^*)^\top + \mathbb E[f^*(\bm x)^2]\bm I_d\,,
\end{align*}
where $\bm H_{f^*}$ is the Hessian matrix of $f^*$. Remarkably in this case we have an isotropic term dependent on $\mathbb E[f(\bm x)^2]$, but also an explicit projection onto the one dimensional space created by the target weights $\bm w^*$. Note that even when $g$ is not twice differentiable, the fact that $g$ is Lipschitz and bounded implies $g$ has weak derivatives in the sense of distribution; thus the quantity above is well-defined regardless of the existence of the classical derivatives of $g$.

As a conclusion, the second moment matrix can be written as
\begin{align*}
    \mathbb E[\bm z\bm z^\top] &= \frac{1}{d} \left(\bm I_d + \frac{\eta^2\lambda_{\alpha, \beta}}{d^2} \Big\{2 s \bm w^*(\bm w^*)^\top + \mathbb E[f^*(\bm x)^2]\bm I_d\Big\} + \frac{\eta^2\lambda_{\alpha, \beta}}{d^2}\sigma_\varepsilon^2 \bm I_d  + \frac{\eta^2\lambda_{\alpha, \beta}}{d}\Big(\alpha - \frac{1}{d}\Big) \mu_1^2\bm w^* (\bm w^*)^\top\right)\\
    &= \frac{1}{d} \left(\bm I_d + \frac{\eta^2\lambda_{\alpha, \beta}}{d^2} \mathbb E[f^*(\bm x)^2]\bm I_d + \frac{\eta^2\lambda_{\alpha, \beta}}{d^2}\sigma_\varepsilon^2 \bm I_d  + \frac{\eta^2\lambda_{\alpha, \beta}}{d}\Big(\alpha - \frac{1}{d} + \frac{2s}{\mu_1^2d}\Big) \mu_1^2\bm w^* (\bm w^*)^\top\right)\\
    &=\frac{1}{d}\left(\bm I_d + \frac{\eta^2\lambda_{\alpha, \beta}}{d^2}\mathbb E[f^*(\bm x)^2]\bm I_d  +\frac{\eta^2\lambda_{\alpha, \beta}}{d^2}\sigma_\varepsilon^2 \bm I_d + \frac{\eta^2}{d}\frac{\mu_1^4}{\alpha \beta^2}\Big(\alpha - \frac{1}{d} + \frac{2s}{\mu_1^2d}\Big) \bm w^* (\bm w^*)^\top\right)\, . 
\end{align*}

\subsection{Characterization of the coefficient $s$ in the new covariance matrix}\label{app:spike-coef}

Remember that, for $h(t) := [g(t)]^2$, we have
\[
    s = \mathbb E[g'(\langle \bm w^*,\bm x\rangle)^2]+ \mathbb E[g(\langle \bm w^*,\bm x\rangle)g''(\langle \bm w^*,\bm x\rangle)] = \frac{1}{2}\mathbb E[h''(\langle \bm w^*, \bm x\rangle)]\, .
\]
To characterize the behavior of $s$, considering that $\|\bm w^*\| = 1$ without loss of generality, we write $ Z:= \langle \bm w^*, \bm x\rangle \sim \mathcal N(0, 1)$ and note that for the standard Gaussian measure the following identity holds
\[
    \mathbb E_Z[h''(Z)] = \mathbb E_Z[(Z^2 - 1) h(Z)] = \mathbb E_Z[(Z^2 - 1) g(Z)^2]\, .
\]
We have
\[
    \mathbb E_Z[(Z^2 - 1) g(Z)^2] = \mathbb E_Z[Z^2g(Z)^2] - \mathbb E_Z[g(Z)^2] = \text{Cov}(Z^2, g(Z)^2)\, .
\]
Also note that $g$ is bounded, therefore
\[
    \mathbb E_Z[(Z^2 - 1) g(Z)^2] \leq M^2_g\mathbb E[|Z^2 - 1|] \leq C_g\, ,
\]
where $C_g$ depends exclusively on $M_g$ and not on the dimension $d$.

Since $g(Z)^2$ is always non-negative, the sign of $s$ depends entirely on the $(Z^2 - 1)$ weighting factor: if $Z \in (-1, 1)$, then $(Z^2 - 1)$ is negative; if $|Z| > 1$, then $(Z^2 - 1)$ is positive. Therefore, $s$ will be negative if  $g(Z)^2$ concentrates most of its mass inside the interval $(-1, 1)$ and decays to zero before the positive regions ($|Z| > 1$) can outweigh it. In other words, $s < 0$ when the link function $g$ is highly localized or ``bump-like" near the origin.

To illustrate the behavior of $s$ across different link functions, we discuss some concrete examples.
\paragraph{Cases when $s < 0$:} To guarantee a negative $s$, we need functions that heavily prioritize the interval $(-1, 1)$ and ignore the tails.

{\bf Indicator function:} Let $g(t) = 1$ if $|t| < 1$, and $0$ elsewhere. This choice makes $g(t)^2$ be exactly $1$ entirely inside the region where $(Z^2 - 1)$ is negative, and it evaluates to $0$ everywhere $(Z^2 - 1)$ is positive. We have $s = \frac{1}{2}\mathbb{E}[(Z^2 - 1)g(Z)^2] = \frac{1}{2}\int_{-1}^{1} (t^2 - 1) \frac{1}{\sqrt{2\pi}} e^{-t^2/2} \mathrm d t$. Because the integrand is strictly negative everywhere in this domain, $s$ must be negative.

{\bf Gaussian bump:} Let $g(t) = \exp(-t^2/2)$. This is a very a localized function. It smoothly peaks at the origin and decays rapidly, heavily weighting the "negative zone" and suppressing the "positive zone. For this case we have $g(t)^2 = \exp(-t^2)$ and evaluating the quantity gives $s = \frac{1}{2}\mathbb{E}[(Z^2 - 1)\exp(-Z^2)]$. This is equivalent to taking the integral of $(t^2 - 1)$ against a tighter Gaussian density $\mathcal{N}(0, 1/3)$. Since the variance is $\frac{1}{3}$, the integral evaluates to something proportional to $(\frac{1}{3} - 1) = -\frac{2}{3}$, yielding a strictly negative $s$.

\paragraph{Cases when $s > 0$:}
To guarantee a positive $s$, we need monotonic functions, or functions that deliberately target the extreme tails of the distribution.

{\bf Identity function:} Let $g(t) = t$. This is the simplest possible case. We have $h(t) = t^2$ and taking the second derivative $h''(t) = 2$. Using the definition $s = \frac{1}{2}\mathbb{E}[h''(Z)]$, we get $s = \frac{1}{2} \mathbb{E}[2] = 1 > 0$.

{\bf Indicator function of the complement:} Let $g(t) = 1$ if $|t| > 1$, and $0$ elsewhere. This function completely zeroes out the region where $(Z^2 - 1)$ is negative. Computing $s = \frac{1}{2}\mathbb{E}[(Z^2 - 1)g(Z)^2] = \int_{1}^{\infty} (t^2 - 1) \frac{1}{\sqrt{2\pi}} e^{-t^2/2} \mathrm d t$. Because $(t^2 - 1)$ is strictly positive for all $t > 1$, the integral is definitively positive.

{\bf Quadratic function:} Let $g(t) = t^2$. Polynomials naturally place massive weight on the tails where numbers grow large, easily overpowering the center. If we compute the quantities necessary $h(t) = t^4$, so $h''(t) = 12 t^2$. Evaluating $s = \frac{1}{2}\mathbb{E}[12Z^2] = 6\mathbb{E}[Z^2] = 6 > 0$.


{\bf ReLU function:} Here we consider the case that $g$ is the ReLU activation, i.e. $g(t) = \max(0,t)$.
Let
\[
    Z = \langle \bm w^*, \bm x\rangle \sim \mathcal N(0, \|\bm w^*\|^2)
\]
then 
\[
    \mathbb E[g(\langle \bm w^*, \bm x\rangle) g''(\langle \bm w^*, \bm x\rangle)] = \mathbb E[g(Z) g''(Z)]
\]
and the second derivative must be interpreted distributionally. Since $g$ is continuous and $g(0) = 0$, the multiplication of the Dirac distribution by $g$ is well-defined and
\[
    gg'' = g \delta_0 = g(0) \delta_0 = 0.
\]

Therefore,
\[
    \mathbb E[g(Z) g''(Z)] = 0.
\]
Moreover $g'(t) = \bm 1_{\{t \geq 0\}}$, so by the symmetry of the centered Gaussian 
\[
    \mathbb E[g'(Z)^2] = \mathbb P(Z \geq 0) = \frac{1}{2}\, .
\]
In the end, combining both results into the definition of $s$ we get
\[
    s = \mathbb E[g'(Z)^2] + \mathbb E[g(Z)g''(Z)] = \frac{1}{2}\,.
\]

%% file: appendixes/appendix-lemmas.tex
\allowdisplaybreaks
\subsection{Characteristic-function comparison for activation products}\label{lemma:activ-prod}
\begin{lemma}
    Let $\sigma: \mathbb{R} \to \mathbb{R}$ be an $L$-Lipschitz function, and define $F(u, v) = \sigma(u)\sigma(v)$. Let $\mu$ and $\nu$ be probability measures on $\mathbb{R}^2$ with characteristic functions $\phi_\mu$ and $\phi_\nu$ such that $|\phi_\mu(\xi) - \phi_\nu(\xi)| \le \delta |\xi|^2 e^{-\xi^\top \Sigma\xi}$ for a positive semidefinite matrix $\Sigma$ such that $\Trarg{\Sigma} \leq C$. Furthermore, assume that $\mu$ and $\nu$ satisfy $P(|x| > R) \le C_0 e^{-\alpha R}$ for some $\alpha > 0$. Then
    \[
        \left| \int F(x) d\mu(x) - \int F(x) d\nu(x) \right| = \mathcal{O}\left(\delta \log^3(1/\delta)\right)
    \]
\end{lemma}

\begin{proof}
    First, we note that because $\sigma$ is $L$-Lipschitz, its growth is at most linear: $|\sigma(x)| \le |\sigma(0)| + L|x|$. Consequently, the product $F(u,v)$ has at most quadratic growth
    \[
        |F(u,v)| \le (|\sigma(0)| + L|u|)(|\sigma(0)| + L|v|) \le C_1(1 + |x|^2)
    \]
    Let $\eta: \mathbb{R}^2 \to \mathbb{R}$ be a smooth ($C^\infty$) function such that $0 \le \eta(x) \le 1$ for all $x$, $\eta(x) = 1$ for all $|x| \le 1$ and $\eta(x) = 0$ for all $|x| \ge 2$.
    
    Because $\eta(x)$ is constant (either $1$ or $0$) outside the region where $1 < |x| < 2$, its gradient $\nabla \eta(x)$ is exactly zero everywhere except inside that closed, compact set. As a continuous function inside a compact set, and given this property of the gradient, there exists some absolute, finite constant $C$ such that
    \[
        |\nabla \eta(x)| \le C \quad \text{for all } x \in \mathbb{R}^2
    \]
    Now, we define the specific cutoff function $\chi_R(x)$ by
    \[
        \chi_R(x) = \eta\left(\frac{x}{R}\right)\, .
    \]
    That makes $\chi_R$ smooth and $\chi_R(x) = 1$ if $|x|\leq R$, $\chi_R(x) = 0$ if $|x| > 2R$, otherwise $1 \leq \chi_R(x) \leq 0$ when $R < |x| < 2R$.
    
    By the Chain Rule, we have
    \[
        \nabla \chi_R(x) = \nabla \left[ \eta\left(\frac{x}{R}\right) \right] = \frac{1}{R} (\nabla \eta)\left(\frac{x}{R}\right)
    \]
    and therefore 
    \[
        |\nabla \chi_R(x)| = \left| \frac{1}{R} (\nabla \eta)\left(\frac{x}{R}\right) \right| = \frac{1}{R} \left| (\nabla \eta)\left(\frac{x}{R}\right) \right| \leq \frac{C}{R}.
    \]
    
    Using that $\chi_R(x)$ supported on $B_{2R}$ with $|\nabla \chi_R| \le \frac{C}{R}$, we define $F_R(x) = F(x)\chi_R(x)$. We split the integral into the truncated core and the tail
    \[
        \int F(x) d\mu(x) = \int F(x) \chi_R(x) d\mu(x) + \int_{|x| > R} F(x)(1 - \chi_R(x)) d\mu(x)
    \]
    Using the quadratic growth bound $|F(x)| \le C_1(1 + |x|^2)$ and the sub-exponential tail of $\mu$, the tail error evaluates to
    \[
        \int_{|x| > R} F(x)(1 - \chi_R(x)) d\mu(x)\le \int_{|x|>R} C_1(1 + |x|^2) d\mu(x) \le C_2 R^2 e^{-\alpha R}
    \]
    The total error is bounded by the core difference plus the tail errors
    \[
        \left| \int F(x) d\mu(x) - \int F(x) d\nu(x) \right| \le \left| \int F(x)\chi_R(x) d\mu(x) - \int F(x)\chi_R(x) d\nu(x) \right| + 2C_2 R^2 e^{-\alpha R}\, .
    \]
    
    On the core term, if we let $F_R(x) \coloneqq F (x) \chi_R(x)$, we use Parseval's identity to obtain
    \[
        \int F(x)\chi_R(x) d\mu(x) - \int F(x)\chi_R(x) d\nu(x)  = \frac{1}{(2\pi)^2} \int_{\mathbb{R}^2} \widehat{F_R}(\xi) (\phi_\mu(-\xi) - \phi_\nu(-\xi)) d\xi\, .
    \]
    Next, we look at the distributional gradient $\nabla F_R$ which can be written as
    \[
        \nabla F_R = \chi_R \nabla F + F \nabla \chi_R\, ,
    \]
    and we will bound the $L^\infty$ norm of both terms on the support ball $B_{2R}$. 
    
    The gradient is given by $\nabla F = (\sigma'(u)\sigma(v), \sigma(u)\sigma'(v))$. Since $\sigma$ is $L$-Lipschitz, $|\sigma'| \le L$ almost everywhere. Also, on the ball of radius $2R$, the activation function is bounded by $|\sigma(u)| \leq C_{\sigma}R$ for some absolute constant $C_{\sigma} > 0$. Therefore, $|\nabla F| = \mathcal{O}(R)$. Furthermore, $|F|$ grows quadratically, and since $|\nabla \chi_R| \leq \frac{C}{R}$, their product is $\mathcal{O}(R)$. Thus, the maximum value of the gradient is bounded by $\|\nabla F_R\|_{L^\infty} \le C_3 R$.
    
    If the symbol $\mathcal{F}$ denotes the Fourier transform operator, by the properties of the Fourier transform, we have that
    \[
        \mathcal{F}\{\nabla F_R\} = i\xi \widehat{F_R}(\xi)\,.
    \]
    Next, the maximum absolute value of any Fourier transform is always bounded by 
    \[
        |\mathcal{F}\{\nabla F_R\}| \leq \|\nabla F_R\|_{L^1}
    \]
    and the $L^1$ norm is bounded by the $L^\infty$ norm times the area of the support ball, which is $\pi(2R)^2$, thus
    \[
        \|\nabla F_R\|_{L^1} \leq \|\nabla F_R\|_{L^\infty}\pi (2R)^2 \leq (C_3 R) \pi (2R)^2 = C_4 R^3\, .
    \]
    Combining these facts together we get
    \[
        |\mathcal{F}\{\nabla F_R\}| =|\xi| |\widehat{F_R}(\xi)| \le \|\nabla F_R\|_{L^1} \le C_4 R^3\,
    \]
    and finally arrive at
    \[
        |\widehat{F_R}(\xi)| \le \frac{C_4 R^3}{|\xi|}\, .
    \]

    Substituting this new bound and the anisotropic characteristic function estimate into the Parseval integral gives
    \[
        \left| \int F(x)\chi_R(x)d\mu(x) - \int F(x)\chi_R(x)d\nu(x) \right| \le \frac1{(2\pi)^2} \int_{\mathbb R^2} \left( \frac{C_4R^3}{|\xi|} \right) \left( \delta |\xi|^2 e^{-\frac12\xi^\top\Sigma\xi} \right)d\xi .
    \]
    
    Canceling one power of $|\xi|$ yields
    \[
        \left| \int F(x)\chi_R(x)d\mu(x) - \int F(x)\chi_R(x)d\nu(x) \right| \le C_5\delta R^3 \int_{\mathbb R^2} |\xi| e^{-\frac12\xi^\top\Sigma\xi}d\xi .
    \]
    
    Now diagonalize the covariance matrix:
    \[
        \Sigma = Q^\top \begin{pmatrix} \lambda_1 &0\\ 0&\lambda_2 \end{pmatrix} Q,
    \]
    with orthogonal $Q$, and perform the rotation $v=Q\xi$. Since orthogonal transformations preserve Lebesgue measure and Euclidean norm,
    \[
        \int_{\mathbb R^2}|\xi| e^{-\frac12\xi^\top\Sigma\xi} d\xi = \int_{\mathbb R^2} |v| e^{-\frac12(\lambda_1v_1^2+\lambda_2v_2^2)} dv .
    \]
    
    Because $\Sigma$ is positive semidefinite and
    \[
        \operatorname{Tr}(\Sigma) \le C,
    \]
    the Gaussian factor provides exponential decay in every nondegenerate direction. Even when one eigenvalue degenerates, the integral remains effectively one-dimensional in the degenerate direction and therefore finite. Consequently,
    \[
        \int_{\mathbb R^2} |v| e^{-\frac12(\lambda_1v_1^2+\lambda_2v_2^2)} dv \le C' ,
    \]
    for some absolute constant $C' > 0$
    
    Therefore,
    \[
        \left| \int F(x)\chi_R(x)d\mu(x) - \int F(x)\chi_R(x)d\nu(x) \right| \le C_6\delta R^3 .
    \]

    Thus, the total bound as a function of the truncation radius $R$ is
    \[
        \left| \int F(x) d\mu(x) - \int F(x) d\nu(x)\right| \le C_6 \delta R^3 + 2C_2 R^2 e^{-\alpha R}\, .
    \]
    We balance the terms by setting the tail decay equal to the $\delta$ parameter
    \[
        e^{-\alpha R} = \delta \implies R = \frac{1}{\alpha} \log(1/\delta)\, .
    \]
    
    Plugging this $R$ back into the total bound we obtain the result
    \[
        \left| \int F d\mu - \int F d\nu \right| \le C_6 \delta \left( \frac{1}{\alpha} \log(1/\delta) \right)^3 + 2C_2 \delta \left( \frac{1}{\alpha} \log(1/\delta) \right)^2 = \mathcal{O}\left( \delta \log^3(1/\delta) \right)\, .
    \]
\end{proof}

\subsection{Closed form of the infinite sums}
\begin{lemma}\label{lemma:ith-inf-series}
    Consider a point $y \in (-\infty, 1/2)$, then for a fixed $i \geq 0$, we have the following identity
    \[
        \sum_{k=i}^\infty \binom{2k}{2i}\frac{(2k-2i-1)!!}{k!} y^{k} = \frac{y^i}{i!}(1-2y)^{-(i+1/2)}.
    \]
\end{lemma}
\begin{proof}

    First, we check
    \[
        C_{i,k} := \binom{2k}{2i}\frac{(2k-2i-1)!!}{k!}.
    \]
    We note that $2k-2i-1$ is necessarily odd, thus we can write
    \[
        (2k-2i-1)!! = \frac{(2k-2i)!}{2^{k-i}(k-i)!},
    \]
    and if we expand the binomial coefficient we have
    \[
        \binom{2k}{2i} = \frac{2k!}{2i!(2k-2i)!}.
    \]
    Therefore
    \[
        C_{i,k} = \frac{(2k-2i)!}{2^{k-i}(k-i)!}\frac{2k!}{2i!(2k-2i)!}\frac{1}{k!} = \frac{1}{2^{k-i}(k-i)!}\frac{2k!}{2i!}\frac{1}{k!}.
    \]

    Now, we use the identity $2n! = 2^n(2n-1)!!$ to obtain
    \[
        C_{i,k} = \frac{2^kk!(2k-1)!!}{2^ii!(2i-1)!!}\frac{1}{2^{k-i}(k-i)!}\frac{1}{k!} = \frac{(2k-1)!!}{i!(k-i)!(2i-1)!!},
    \]
    so our goal will be to prove that
    \begin{equation}\label{eq:ith-inf-sum}
        \frac{y^i}{i!}(1-2y)^{-(i+1/2)} = \sum_{k=i}^\infty \frac{(2k-1)!!}{i!(k-i)!(2i-1)!!} y^{k}         
    \end{equation}

    Next, we study the function
    \[
        (1-2y)^{-(i+1/2)}.
    \]

    Consider the Maclaurin series of the function
    \[
        (1+x)^r = \sum \binom{r}{n} x^r,
    \]
    which is defined for all $|x|<1$ and real number $r.$
    
    If we let $x = -2y$ and $r = -(i+1/2)$ we have
    \begin{align*}
        \binom{-(i+1/2)}{n} &= \frac{-(i+1/2)-(i+3/2)\dots-(i+ n - 1/2)}{n!} \\
        &= (-1)^k\frac{[(2i+1)(2i+3)\dots (2i+2n-1)]}{2^n n!} \\
        &= (-1)^k\frac{(2i+2n-1)!!}{2^n n!(2i-1)!!}.
    \end{align*}
    And substituting everything back
    \begin{align*}
        (1-2y)^{-(i+1/2)} &= \sum_{n = 0}^\infty (-1)^n\frac{(2i+2n-1)!!}{2^n n!(2i-1)!!} (-2y)^n\\
        &= \sum_{n=0}^\infty (-1)^{2n}\frac{(2i+2n-1)!!}{2^n n!(2i-1)!!} 2^n y^n\\
        &=\sum_{n=0}^\infty \frac{(2i+2n-1)!!}{n!(2i-1)!!}y^n.
    \end{align*}

    Multiplying the expansion by $\frac{y^i}{i!}$ we get
    \[
       \frac{y^i}{i!}(1-2y)^{-(i+1/2)}  = \sum_{n=0}^\infty \frac{(2i+2n-1)!!}{i!n!(2i-1)!!}y^{n+i},
    \]
    and if we let $k = n+i$, rearranging the indexes leads to
    \[
       \frac{y^i}{i!}(1-2y)^{-(i+1/2)}  = \sum_{k=i}^\infty \frac{(2k-1)!!}{i!(k-i)!(2i-1)!!}y^k,
    \]
    which is exactly the expression from Equation \ref{eq:ith-inf-sum}.
\end{proof}

\subsection{Concentration for Gaussian integrals}
\begin{lemma}\label{lemma:integral-concentration-1}
    Consider a fixed $\bm x \in \mathbb R^d$. Let $\bm x' \sim \mathcal N(0, \bm I_d)$ and define the set
    \[
        \mathcal A_{\epsilon} = \left\{\bm x' \in \mathbb R^d: \left|\theta_{\bm x, \bm x'} - \frac{\pi}{2}\right|< \epsilon \right\}, \epsilon \in \left(0, \frac{\pi}{2}\right).
    \]

    Then 
    \[
        \rho_{\mathcal X}(\mathcal A^c_{\epsilon}) < 3e^{-c_0d\epsilon^2},
    \]
    and consequently, for any function $g \in L^2(\rho_{\mathcal X})$, we have
    \[
        \left|\int_{\mathcal A^c_{\epsilon}} \frac{(\pi - \theta_{\bm x, \bm x'})}{\pi}g(\bm x')\mathrm d\rho_{\mathcal X}(\bm x')\right| \leq \sqrt {3} e^{-c_1d\epsilon^2/2}\|g\|_{L^2(\rho_{\mathcal X})},
    \]
    and
    \[
        \left|\int_{\mathcal A^c_{\epsilon}} \frac{\sin \theta_{\bm x, \bm x'}}{2\pi}g(\bm x')\mathrm d\rho_{\mathcal X}(\bm x')\right| \leq \sqrt {3} e^{-c_2d\epsilon^2/2}\|g\|_{L^2(\rho_{\mathcal X})}.
    \]
    where $c_0, c_1, c_2 >0$ are absolute constants.
\end{lemma}
\begin{proof}
    If we consider the set
    \[
        \mathcal{A}_{\epsilon} = \left\{\left|\theta_{\bm x, \bm x'}-\frac{\pi}{2}\right| < \epsilon\right\},
    \] 
    for a fixed $\bm x$ we define
    \[
        \cos \theta_{\bm x, \bm x'} = \frac{\langle \bm x, \bm x'\rangle}{\|\bm x\| \|\bm x'\|} \coloneqq \frac{Z}{\sqrt{Z^2 +\|\bm y_{\perp}\|^2}}
    \]
    where $\bm y_{\perp}$ is a $(d-1)$ dimensional vector orthogonal to $\bm x$. We define the bad event
    \[
        \mathcal A^c_{\epsilon} = \left \{ \left| \theta_{\bm x, \bm x'} - \frac{\pi}{2} \right| > \epsilon \right \} = \Big\{|\cos \theta_{\bm x, \bm x'}| > \sin \epsilon\Big\}.
    \]

    Next we note that
    \[
        |\cos \theta_{\bm x, \bm x'}| = \frac{|Z|}{\sqrt{Z^2 +\|\bm y_{\perp}\|^2}} \leq \frac{|Z|}{\|\bm y_{\perp}\|}
    \]
    thus
    \[
        \mathbb P(\mathcal A_{\epsilon}^c) \leq \mathbb P\left( \frac{|Z|}{\|\bm y_{\perp}\|} \geq \sin \epsilon\right)\, .
    \]
    Splitting the event according to $\|\bm y_{\perp}\|^2 \geq \frac{d-1}{2}$ we have
    \[
        \mathbb P(\mathcal A_{\epsilon}^c) \leq \mathbb P\left( |Z| \geq \sin \epsilon \sqrt{\frac{d-1}{2}}\right) + \mathbb P\left( \|\bm y_{\perp}\|^2 < \frac{d-1}{2}\right).
    \]

    Using the standard Gaussian tail bound and the lower-tail concentration bound for the chi-squared distribution, there exists an absolute constant $c > 0$ such that
    \[
        \mathbb P(|Z| > t) \leq 2e^{-t^2/2}\, , \quad  \mathbb P\left( \|\bm y_{\perp}\|^2 < \frac{d-1}{2}\right) \leq e^{-c(d-1)}
    \]
    and we have
    \[
        \mathbb P(\mathcal A_{\epsilon}^c) \leq 2e^{-\sin(\epsilon)^2(d-1)/4} + e^{-c(d-1)}\, .
    \]
    Using the bound 
    \[
        \sin(\epsilon) \ge \frac{2}{\pi}\epsilon, \quad \forall \epsilon \in [0, \pi/2],
    \]
    and assuming without loss of generality that $\epsilon > 0$ is chosen such that $c > \frac{\epsilon^2}{\pi}$ we can write
    \[
        \mathbb P(\mathcal A_{\epsilon}^c) \leq 3e^{-c_0d\epsilon^2}
    \]
    for some absolute constant $c_0 > 0$.

    Given this, we study the integral
    \[
        \int_{\mathcal{A}^c_{\epsilon}} \frac{(\pi - \theta_{\bm x, \bm x'})}{\pi} g(\bm x') \mathrm d\rho_{\mathcal X}(\bm x')\, .
    \]

    Noting that $\theta_{\bm x, \bm x'} \in [0, \pi]$, we have $\left| \frac{\pi - \theta_{\boldsymbol{x},\boldsymbol{x}'}}{\pi} \right| \le 1$ and by the Cauchy-Schwarz inequality we can bound the integral over $\mathcal{A}^c$:
    \[
        \left| \int_{\mathcal{A}^c_{\epsilon}} \frac{(\pi - \theta_{\bm x, \bm x'})}{\pi} g(\bm x') \mathrm d\rho_{\mathcal X}(\bm x') \right| \leq \sqrt{\rho_{\mathcal X}(\mathcal{A}^c_{\epsilon})}. \sqrt{\int |g(\bm x')|^2 \mathrm d \rho_{\mathcal X}} \leq \sqrt 3 e^{-c_1d\epsilon^2/2}\|g\|_{L^2(\rho_{\mathcal X})}.
    \]

    Lastly, the result for 
    \[
        \left|\int_{\mathcal A^c_{\epsilon}} \frac{\sin \theta_{\bm x, \bm x'}}{2\pi}g(\bm x')\mathrm d\rho_{\mathcal X}(\bm x')\right|
    \]
    follows from the fact that $|\sin \theta| \leq 1$ and the same argument under the concentration of the measure.
\end{proof}

%% file: appendixes/appendix-main-proofs.tex
\subsection{Proof of Theorem \ref{thm:approx-true-kernel}}\label{app:thm-approx-true-kernel}
\begin{proof}
    We split the proofs into several parts. First, we show that the characteristic function of the true distribution is sufficiently close to the one from the Gaussian distribution governed by $\bm \Gamma$. Then, we use that to bound $\|k_1^* - k_1\|_{L^2(\rho_{\mathcal X})\times L^2(\rho_{\mathcal X})}$, which immediately implies the operators are also close in norm. Since these objects live in the $L^2(\rho_{\mathcal X})$ space, the inputs are inherently unbounded, thus we further separate this case to deal with a bounded set and use the exponential decay of the measure to control the tails. Lastly, we show the distribution shift identity by introducing a pushforward measure dictated by $\bm \Gamma$ to translate the weights and inputs to the context of isotropic weights.
    \item \paragraph{Convergence of the true distribution to the spiked Gaussian:}
    Unconditionally, $\bm z$ is not Gaussian in general (it is a mixture of sub-exponential random variables), but conditional on $\bm \varsigma$ it is Gaussian.
    To be specific, if we let $c \coloneqq \frac{\mu_1 \eta}{n \sqrt{m}}$, with $a \sim \mathcal N(0, \frac{1}{m})$ and $\bm w \sim \mathcal N(0, \frac{1}{d}\bm I_d)$, condition on $\bm \varsigma = \sum y_i \bm x_i$, we have
    \[
        \bm z \ | \ \bm \varsigma \sim \mathcal N(0, \bm \Sigma(\bm \varsigma)), \quad \bm \Sigma(\bm \varsigma) = \frac{1}{d}\bm I_d + \frac{c^2}{m}\bm \varsigma \bm \varsigma^\top\, .
    \]
    So conditioning on $\bm \varsigma$, $\bm z$ is Gaussian with an anisotropic rank-one spike along the direction $\bm \varsigma$. Therefore, if we fix a test vector $\bm t \in \mathbb R^d$ and consider the projection $\langle \bm z, \bm t\rangle$, conditionally on $\bm \varsigma$ we have
    \[
        \langle \bm z, \bm t\rangle \ | \ \bm \varsigma \sim \mathcal N\left(0, \frac{\|\bm t\|^2}{d} + \frac{c^2}{m}\langle \bm \varsigma, \bm t\rangle^2 \right)\, .
    \]
    And consequently, the unconditional characteristic function is given by
    \[
        \phi_{\langle \bm z, \bm t\rangle}(u) = \exp\left(- \frac{u^2}{2}\frac{\|\bm t\|^2}{d}\right) \mathbb E_{\bm \varsigma}\left[ \exp\left(- \frac{u^2}{2}\frac{c^2\langle \bm \varsigma, \bm t\rangle^2}{m}\right) \right]\, .
    \]
    
    Now, to investigate the concentration of this variable, we look at the difference
    \[
        \left| \mathbb E_{\bm \varsigma}\left[ \exp\left(- \frac{u^2}{2}\frac{c^2\langle \bm \varsigma, \bm t\rangle^2}{m}\right) \right] -  \exp\left(- \frac{u^2}{2}\frac{c^2\mathbb E_{\bm \varsigma}[\langle \bm \varsigma, \bm t\rangle^2]}{m}\right)\right|\, .
    \]
    
    If we define $X = \frac{c^2}{m}\langle \bm \varsigma, \bm t\rangle^2$, $\mu = \frac{c^2}{m}\mathbb E_{\bm \varsigma}[\langle \bm \varsigma, \bm t\rangle^2]$ and $\alpha = \frac{u^2}{2}$, we perform a second-order Taylor expansion of the function $f(x) = e^{-\alpha x}$ around $\mu$. By Taylor's theorem, there exists some $\xi$ between $X$ and $\mu$ such that
    \[
        e^{-\alpha X} = e^{-\alpha \mu} - \alpha e^{-\alpha \mu}(X - \mu) + \frac{\alpha^2 e^{-\alpha \xi}}{2}(X - \mu)^2
    \]

    Taking the expectation of both sides yields:$$\mathbb{E}[e^{-\alpha X}] = e^{-\alpha \mu} - \alpha e^{-\alpha \mu}\mathbb{E}[X - \mu] + \frac{\alpha^2}{2}\mathbb{E}[e^{-\alpha \xi}(X - \mu)^2]$$Because $\mu = \mathbb{E}[X]$, the first-order term cancels out exactly ($\mathbb{E}[X - \mu] = 0$), leaving:$$\mathbb{E}[e^{-\alpha X}] - e^{-\alpha \mu} = \frac{\alpha^2}{2}\mathbb{E}[e^{-\alpha \xi}(X - \mu)^2]$$

    Since $\alpha = u^2/2 > 0$ and $X \geq 0$ (as it is a scaled squared projection), it follows that $\xi \geq 0$ and therefore $e^{-\alpha \xi} \leq 1$. Taking the absolute value, we can bound the difference directly by the variance of $X$
    \[
        |\mathbb{E}[e^{-\alpha X}] - e^{-\alpha \mu}| \leq \frac{\alpha^2}{2}\mathbb{E}[(X - \mu)^2] = \frac{\alpha^2}{2}\text{Var}(X)
    \]
    Substituting our definitions back in, we have $\frac{\alpha^2}{2}\text{Var}(X) = \frac{u^4}{8}\text{Var}(X)$. Given that $\frac{c^2}{m} = \mathcal{O}(\frac{\eta^2}{d^4})$, the variance of the sub-exponential variable yields the following bound
    \begin{equation}\label{eq:char-fn-bound}
        |\mathbb{E}[e^{-\alpha X}] - e^{-\alpha \mu}| \leq \mathcal{O}\left(\frac{\eta^4}{d^5}\right)
    \end{equation}
    
    \item \paragraph{Bounding the kernel difference:}
    We let $p_{\bm \Gamma}$ be the density function of the Gaussian distribution $\mathcal N(0, \frac{1}{d}\bm \Gamma)$ and $p_1^* = \mathbb E_{\bm \varsigma}[p_{\bm I_d}(\bm z \ | \ \bm \varsigma)]$ be the true density function from the non-Gaussian distribution followed by $z$.
    
    Denote $G(\bm w) := \sigma(\langle \bm w, \bm x\rangle) \sigma(\langle \bm w, \bm x'\rangle)$, the true kernel after the deterministic update is given by
    \[
        k_1^*(\bm x, \bm x') = \int G(\bm w) p_1^*(\bm w) \mathrm d \bm v\, ,
    \]
    while the Gaussian kernel $k_1$ is
    \[
        k_1(\bm x, \bm x') = \int G(\bm w) p_{\bm \Gamma}(\bm w) \mathrm d \bm v\, .
    \]
    We define a 2-dimensional vector $\bm{v}$ representing the two projections
    \[
        \bm{v} = \begin{bmatrix} v_1 \\ v_2 \end{bmatrix} = \begin{bmatrix} \langle \bm{w}, \bm{x} \rangle \\ \langle \bm{w}, \bm{x}' \rangle \end{bmatrix} \in \mathbb{R}^2\, ,
    \]
    then we can write $G$ in terms of this 2D variable: $G(\bm{v}) = \sigma(v_1)\sigma(v_2)$. The error between both kernels is the difference in expectations over this 2D plane
    \[
        |k_1^*(\bm x, \bm x') - k_1(\bm x, \bm x')| = \left| \int_{\mathbb{R}^2} G(\bm{v}) p_1^*(\bm{v}) \mathrm d\bm{v} - \int_{\mathbb{R}^2} G(\bm{v}) p_{\bm \Gamma}(\bm{v}) \mathrm d\bm{v} \right|
    \]
    
    Taking the Fourier transforms, we map this into the 2D frequency domain. In this context, the frequency variable is $\bm{u} = (u_1, u_2)$ and the error is given by
    \[
        |k_1^*(\bm x, \bm x') - k_1(\bm x, \bm x')| = \frac{1}{(2\pi)^2} \left| \int_{\mathbb{R}^2} \hat{G}(\bm{u}) [\Phi_1^*(\bm{u}) - \Phi_{\bm \Gamma}(\bm{u})] \mathrm d\bm{u} \right|\, ,
    \]
    where $\Phi$ are the respective characteristic functions of the distributions. We can see that the inner product between $\bm u$ and $\bm v$ gives
    \[
        \langle \bm u, \bm v\rangle = u_1 \langle \bm w, \bm x\rangle + u_2 \langle \bm w, \bm x'\rangle = \langle \bm w, u_1 \bm x + u_2 \bm x' \rangle\, .
    \]
    Therefore, if we define our test vector as $\bm t_u = u_1 \bm x + u_2 \bm x' \in \mathbb R^d$, $\langle \bm u, \bm v\rangle = \langle \bm w, \bm t_u\rangle$, the characteristic function for $\bm u$ is given by
    \[
        \Phi(\bm u) = \mathbb E[\exp(i \langle \bm u, \bm v\rangle)] = \mathbb E[\exp(i \langle \bm w, \bm t_u \rangle)]\, .
    \]
    This implies the characteristic function acting on $\bm u$ is precisely the characteristic function of a 1D projection of $\bm w$, as considered in the bound from \cref{eq:char-fn-bound}. Thus, we can plug in the bound and factor the $d$-dependent estimate out of the integral
    \[
        |\Phi_1^*(\bm{u}) - \Phi_{\bm \Gamma}(\bm{u})| \leq \mathcal{O}\left(\frac{\eta^4\ln^3 d}{d^5}\right) \|\bm{u}\|^2 \exp\left(-\frac{1}{2}\bm{u}^T \bm \Sigma \bm{u}\right) \, ,
    \]
    where we write the covariance matrix $\bm \Sigma$ as
    \[
        \bm \Sigma = \frac{1}{d}\begin{bmatrix}
            \|\bm x\|^2 & \langle \bm x, \bm x' \rangle\\
            \langle \bm x, \bm x' \rangle & \|\bm x'\|^2
        \end{bmatrix}\, .
    \]
    
    We will prove the kernels are close in norm in the product space $L^2(\rho_{\mathcal X}) \times L^2(\rho_{\mathcal X})$, which leads to the same conclusion to the Hilbert-Schmidt norm and the operator norm of the difference of integral operators. For this we split this proof with a truncation argument considering a bounded domain and then use the Lipschitz property of $\sigma$ to ensure the tails decays exponentially. 
    Ultimately, we want analyze the integral
    \[
        \iint|k_1^*(\bm x, \bm x') - k_1(\bm x, \bm x')|^2 \mathrm d \rho_{\mathcal X}(\bm x)\mathrm d \rho_{\mathcal X}(\bm x')
    \]
    so we start controlling the difference inside a bounded set.
    \paragraph{Bound of the pointwise difference over a bounded set: }
    To avoid singularities near the origin, we redefine our bounded domain as the set
    \[
        B_{R,c} = \{\bm x \in \mathbb{R}^d : c\sqrt{d} \le \|\bm x\| < R\sqrt{d}\}
    \]
    and we consider the bounded product set $B_{R,c} \times B_{R,c}$.
    Then, since $\bm \Sigma$ is positive semidefinite and inside this set it is never degenerate, we have 
    \[
        \Trarg{\bm \Sigma} = \frac{\|\bm x\|^2 + \|\bm x'\|^2}{d} \leq 2R^2\, .
    \]
    Now, because both weight distributions are sub-exponential, using Lemma \ref{lemma:activ-prod}, we get that
    \[
        \sup_{\bm x, \bm x' \in B_{R, c} \times B_{R, c}}|k_1^*(\bm x, \bm x') - k_1(\bm x, \bm x')| = \mathcal O\left(\frac{\eta^4\ln^3 d}{d^5}\right)\, .
    \]

    \paragraph{Bounding the integral over the tails:}
    Since the integration over the bounded set is handled, we analyze the integral over $( B_{R, c} \times B_{R, c} )^c$. Because the activation function $\sigma$ is $L_{\sigma}$-Lipschitz, we have the inequality $|\sigma(t)| \le |\sigma(0)| + L_{\sigma}|t|$. Consequently, for both kernels, there exist absolute constants $C_1, C_2 > 0$ such that the diagonal grows at most quadratically with the input norm
    \[
        k(\bm{x}, \bm{x}) \le C_1 + C_2\mathbb E_{\bm w}[\langle \bm w, \bm x\rangle ^2].
    \]
    Because the weights in both kernels share the same covariance matrix $\frac{1}{d}\bm \Gamma$, we have that
    \[
        \mathbb E_{\bm w}[\langle \bm w, \bm x\rangle ^2] = \frac{1}{d}\bm x^\top \bm \Gamma \bm x\leq \frac{\lambda_{\max}(\bm \Gamma)}{d}\|\bm x\|^2\, ,
    \]
    in both cases. Since $\lambda_{\max}(\bm \Gamma) = A+B$ and $B = \mathcal O\left(\frac{\eta^2}{d}\right)$ we have
    \[
         \mathbb E_{\bm w}[\langle \bm w, \bm x\rangle ^2]\leq C'\frac{\eta^2\|\bm x\|^2}{d^2}
    \]
    for some absolute constant $C' > 0$. Hence, we can find $C_2' > 0$ such that
    \[
        k(\bm{x}, \bm{x}) \le C_1 + C_2'\frac{\eta^2\|\bm x\|^2}{d^2}.
    \]
    By the Cauchy-Schwarz inequality, absorbing all constants into $C > 0$, the off-diagonal terms are bounded by
    \[
        |k(\bm x, \bm x')| \le \sqrt{k(\bm x, \bm x) k(\bm x', \bm x')} \le \sqrt{\left(1 + C\frac{\eta^2\|\bm x\|^2}{d^2}\right)\left(1 + C\frac{\eta^2\|\bm x'\|^2}{d^2}\right)}
    \]
    and using that $\sqrt {ab} \leq (a+b)/2$, absorbing necessary constants into $C$ again, we have
    \[
        |k(\bm x, \bm x')| \leq C \left(1 + \frac{\eta^2\|\bm x\|^2}{d^2} + \frac{\eta^2\|\bm x'\|^2}{d^2}\right)
    \]
    and using the algebraic identity $(a-b)^2 \le 2a^2 + 2b^2$, we get
    \[
        |k_1^*(\bm x, \bm x') - k_1(\bm x, \bm x')|^2 \leq 2|k_1^*(\bm x, \bm x')|^2 + 2|k_1(\bm x, \bm x')|^2 
    \]
    and applying our bound gives
    \[
        |k_1^*(\bm x, \bm x') - k_1(\bm x, \bm x')|^2 \leq 4C \left(1 + \frac{\eta^2\|\bm x\|^2}{d^2} + \frac{\eta^2\|\bm x'\|^2}{d^2}\right)^2\, .
    \]
    Finally, using the identity $(a+b+c)^2 \leq 3(a^2 + b^2 + c^2)$, we have
    \[
        |k_1^*(\bm x, \bm x') - k_1(\bm x, \bm x')|^2 \leq M \left(1 + \frac{\eta^4\|\bm x\|^4}{d^4} + \frac{\eta^4\|\bm x'\|^4}{d^4}\right) \, ,
    \]
    where $M$ is a constant depending on $L_{\sigma}$, $\sigma(0)$ and $C$. Therefore, the integral over the unbounded set is bounded by
    \[
        \iint_{(B_{R, c}\times B_{R, c})^c}|k_1^* - k_1|^2 \mathrm d \rho_{\mathcal X}(\bm x)\mathrm d \rho_{\mathcal X}(\bm x') \leq \iint_{(B_{R, c}\times B_{R, c})^c} M \left(1 +  \frac{\eta^4\|\bm x\|^4}{d^4} +  \frac{\eta^4\|\bm x'\|^4}{d^4}\right)  \mathrm d \rho_{\mathcal X}(\bm x)\mathrm d \rho_{\mathcal X}(\bm x') \, .
    \]
    Noting that
    \begin{align*}
        \Big( B_{R,c} \times B_{R,c}\Big)^c &\subset \left\{\Big\{\|\bm x\| > R \sqrt d\Big\} \times \mathbb R^d\right\} \cup \left\{\Big\{\|\bm x\| < c \sqrt d\Big\} \times \mathbb R^d\right\} \\
        & \quad \cup \left\{\mathbb R^d \times \Big\{\|\bm x'\| > R \sqrt d\Big\}\right\} \cup \left\{\mathbb R^d \times \Big\{\|\bm x'\| < c \sqrt d\Big\}\right\}\, ,
    \end{align*}
    since $\rho_{\mathcal X}$ is a probability measure, by symmetry, using the union bound we have that
    \[
        \iint_{(B_{R, c}\times B_{R, c})^c} 1\mathrm d \rho_{\mathcal X}(\bm x)\mathrm d \rho_{\mathcal X}(\bm x') \leq 2\mathbb P\left(\|\bm x\| > R\sqrt {d}\right) + 2\mathbb P\left(\|\bm x\| < c\sqrt {d}\right)\, .
    \]
    If $\bm x \sim \mathcal{N}(0, \bm I_d)$, by standard Gaussian concentration we have the following probability bound
    \[
        \mathbb{P}\Big(\|\bm x\| > R\sqrt{d}\Big) = \mathbb{P}\Big(\|\bm x\| > \sqrt{d} + (R-1)\sqrt{d}\Big) \leq \exp\left(-\frac{(R-1)^2 d}{2}\right)\,.
    \]
    Furthermore, the probability mass of the excluded inner ball is bounded by
    \[
        \mathbb{P}(\|\bm x\| < c\sqrt{d}) = \mathbb{P}(\|\bm x\|^2 < c^2 d) \le \left[c^2 e^{1-c^2}\right]^{d/2}\, .
    \]
    For every fixed $0<c<1$, we have
    \[
        c^2 e^{1-c^2}<1\, .
    \]
    Defining $y := c^2 e^{1-c^2}$, the bound becomes
    \[
        \mathbb P(\|\bm x\|<c\sqrt d) \leq  y^{d/2}.\, 
    \]
    Since $y <1$, we have $\ln y<0$, and therefore
    \[
        y^{d/2} = \exp\!\left(\frac {d}{2} \ln y\right) = \exp(-c'd)\, ,
    \]
    where
    \[
        c' \coloneqq -\frac12\ln y = \frac12\bigl(c^2-1-2\ln c \bigr) >0\, .
    \]

    Give this, it suffices to estimate the term
    \[
        \iint_{\left\{\{\|\bm x\| > R \sqrt d\} \times \mathbb R^d\right\}}\frac{\eta^4}{d^4}\|\bm x\|^4\mathrm d\rho_{\mathcal X}(\bm x) \mathrm d\rho_{\mathcal X}(\bm x')\, ,
    \]
    and by symmetry the other will follow exactly the same.
    Since the integral does not depend on $\bm x'$, we have
    \[
        \iint_{\left\{\{\|\bm x\| > R \sqrt d\} \times \mathbb R^d\right\}}\frac{\eta^4}{d^4}\|\bm x\|^4\mathrm d\rho_{\mathcal X}(\bm x) \mathrm d\rho_{\mathcal X}(\bm x') = \int_{\{\|\bm x\| > R \sqrt d\}}\frac{\eta^4}{d^4}\|\bm x\|^4\mathrm d\rho_{\mathcal X}(\bm x) \, .
    \]
    By Cauchy Schwarz we have that
    \[
        \int_{\{\|\bm x\| > R \sqrt d\}}\frac{\eta^4}{d^4}\|\bm x\|^4\mathrm d\rho_{\mathcal X}(\bm x)  \leq \frac{\eta^4}{d^4} \sqrt{\mathbb P \Big(\|\bm x\| > R\sqrt d\Big)} \sqrt{\mathbb E_{\bm x}[\|\bm x\|^8]}
    \]
    and since $\|\bm{x}\|^2$ follows a $\chi^2_d$ distribution, we have $\mathbb{E}[\|\bm {x}\|^8] = \mathcal O(d^4)$. Substituting this back and using the concentration of the measure again, we get
    \[
        \int_{\{\|\bm x\| > R \sqrt d\}}\frac{\eta^4}{d^4}\|\bm x\|^4\mathrm d\rho_{\mathcal X}(\bm x) = \mathcal O\left[\frac{\eta^4}{d^2}\exp\left(-\frac{(R-1)^2 d}{4}\right)\right] \, .
    \]
    Analogously, we have
    \[
        \int_{\{\|\bm x\| < c \sqrt d\}}\frac{\eta^4}{d^4}\|\bm x\|^4\mathrm d\rho_{\mathcal X}(\bm x) \leq \frac{\eta^4}{d^4}\sqrt{\mathbb P \Big(\|\bm x\| < c\sqrt d\Big)}\sqrt{\mathbb E_{\bm x}[\|\bm x\|^8]} = \mathcal O\left[\frac{\eta^4}{d^2}\exp(-c'd/2)\right] \, .
    \]
    
    Hence, collecting everything, the integral over $\Big( B_{R, c} \times B_{R, c} \Big)^c$ is bounded by
    \[
        \iint_{(B_{R, c}\times B_{R, c})^c}|k_1^* - k_1|^2 \mathrm d \rho_{\mathcal X}(\bm x)\mathrm d \rho_{\mathcal X}(\bm x') =  \mathcal O\left[\frac{\eta^4}{d^2} \exp \left( -\frac{(R-1)^2d}{4}\right) \right] + \mathcal O\left[\frac{\eta^4}{d^2}\exp(-c'd/2)\right] \, .
    \]
    Since $\eta = \Theta(d^\zeta)$, with $\zeta \in [1/2, 1)$, this gives the bound
    \[
        \iint_{(B_{R, c}\times B_{R, c})^c}|k_1^* - k_1|^2 \mathrm d \rho_{\mathcal X}(\bm x)\mathrm d \rho_{\mathcal X}(\bm x') = \mathcal O\left(d^{4\zeta-2}e^{-K d} \right)\, .
    \]
    for the absolute constant $K = \min \left\{\frac{(R-1)^2}{4}, \frac{c'}{2}\right\} > 0$, which does not depend on $d$. Because of the exponential decay, this is strictly smaller than the bound inside $B_{R, c}\times B_{R, c}$.
    \paragraph{Final bound on the norm:}
    From the previous discussions we know
    \[  
        \iint_{B_{R, c}\times B_{R, c}}|k_1^* - k_1|^2 \mathrm d \rho_{\mathcal X}(\bm x)\mathrm d \rho_{\mathcal X}(\bm x') =  \mathcal O\left(\frac{\eta^4\ln^3 d}{d^5}\right)^2
    \]
    and 
    \[
        \iint_{(B_{R, c}\times B_{R, c})^c}|k_1^* - k_1|^2 \mathrm d \rho_{\mathcal X}(\bm x)\mathrm d \rho_{\mathcal X}(\bm x') = \mathcal O\left(d^{4\zeta-2}e^{-K d} \right)\, .
    \]
    Writing $\|k_1^* - k_1\|_{L^2(\rho_{\mathcal X})\times L^2(\rho_{\mathcal X})}$ as the integral of interest and separating into the integration of both sets
    \[
        \iint|k_1^* - k_1|^2 \mathrm d \rho_{\mathcal X}(\bm x)\mathrm d \rho_{\mathcal X}(\bm x') = \iint_{B_{R, c}\times B_{R, c}}|k_1^* - k_1|^2 \mathrm d \rho_{\mathcal X}(\bm x)\mathrm d \rho_{\mathcal X}(\bm x') + \iint_{(B_{R, c}\times B_{R, c})^c}|k_1^* - k_1|^2 \mathrm d \rho_{\mathcal X}(\bm x)\mathrm d \rho_{\mathcal X}(\bm x')
    \]
    we can combine the bounds over both sets to obtain
    \[
        \|k_1^* - k_1\|_{L^2(\rho_{\mathcal X})\times L^2(\rho_{\mathcal X})} = \mathcal O\left(\frac{\eta^4\ln^3 d}{d^5}\right) \,,
    \]
    and since
    \[
        \|T_1^* - T_1 \|_{\text{op}} \leq \|T_1^* - T_1 \|_{\text{HS}} = \|k_1^* - k_1\|_{L^2(\rho_{\mathcal X})\times L^2(\rho_{\mathcal X})}\, ,
    \]
    the same bound shows the operators are close in norm as well.

    \item \paragraph{Proof of the distribution shift identity}
    We consider the integral operators $T_0: L^2(\rho_{\mathcal X}) \to L^2(\rho_{\mathcal X})$ and $T_1: L^2(\rho_{\mathcal X}) \to L^2(\rho_{\mathcal X})$ 
    \[
        (T_0f)(\bm x) = \int_{\mathcal X} k_0(\bm x, \bm x')f(\bm x') \mathrm d\rho_{\mathcal X}(\bm x') \qquad (T_1f)(\bm x) = \int_{\mathcal X} k_1(\bm x, \bm x')f(\bm x') \mathrm d\rho_{\mathcal X}(\bm x'),
    \]
    and because $T_0$ and $T_1$ are trace-class and compact, Mercer's theorem guarantees that both kernels can be diagonalized. For $k_0$ we have
    \[
        k_0(\bm x, \bm x') = \sum_{i=1}^\infty \xi_i \varphi_i(\bm x)\varphi_i(\bm x')
    \]
    where $\{\varphi_i\}_{i=1}^\infty$ is an orthonormal system and $\{\xi_i\}_{i=1}^\infty$ is a family of eigenvalues associated with these basis such that $\xi_1 \geq \xi_2 \geq ... > 0$. For $k_1$ we have a similar result and
    \[
        k_1(\bm x, \bm x') = \sum_{i=1}^\infty \chi_i \psi_i(\bm x)\psi_i(\bm x')\, ,
    \]
    where $\{\psi_i\}_{i=1}^\infty$ is an orthonormal system and $\{\chi_i\}_{i=1}^\infty$ are the associated non-increasing eigenvalues. In particular, if the kernels are universal, we have that $\overline{\text{span}\{\varphi_i: i \geq 1\}} = \overline{\text{span}\{\psi_i: i \geq 1\}} = L_2(\mathcal X, \rho_{\mathcal X})$.
    
    Consider the matrix from \cref{eq:covz} (without the $\frac{1}{d}$ scaling), written in short form
    \[
        \bm \Gamma  = A\bm I_d + B\bm w^* (\bm w^*)^\top
    \]
    with constants $A$ and $B$ from the original definition satisfying $A + B > |B|$. In particular, given $\bm w \sim \mathcal N\left(0, \frac{1}{d}\bm I\right)$ and $\langle \bm w, \bm x\rangle$, we know that for $\bm z \sim \mathcal N\left(0, \frac{1}{d}\bm \Gamma\right)$ we can write
    \[
        \langle \bm z, \bm x\rangle = \langle \bm w, \bm \Gamma^{1/2} \bm x \rangle
    \]
    and the kernels are related by
    \begin{align*}
        k_1(\bm x, \bm x') &= \mathbb E_{\bm z \sim \mathcal N\left(0, \frac{1}{d}\bm \Gamma\right)}[\sigma(\langle \bm z, \bm x\rangle)\sigma(\langle\bm z, \bm x'\rangle)] \\
        &= \mathbb E_{\bm w \sim \mathcal N\left(0, \frac{1}{d}\bm I_d\right)}[\sigma(\langle \bm w, \bm \Gamma^{1/2} \bm x \rangle)\sigma(\langle \bm w, \bm \Gamma^{1/2} \bm x \rangle)]\\
        &= k_0(\bm \Gamma^{1/2}\bm x, \bm \Gamma^{1/2}\bm x').
    \end{align*}
    
    Let us define the measure $\nu$ as the pushforward measure $\nu = (\bm \Gamma^{\frac{1}{2}})_{\#} V$
    \[
        \nu(V) = (\rho_{\mathcal X} \circ \bm \Gamma ^{-1/2})(V)= \rho_{\mathcal X} (\bm \Gamma^{-1/2}V)\, ,
    \]
    for every measurable set $V$ of $(\bm \Gamma ^{1/2} \mathcal X)$. From now on, we denote the input space transformed by the $\bm \Gamma^{1/2}$ matrix as $(\bm \Gamma^{1/2}\mathcal X) := \mathcal Z \subset \mathbb R^d.$
    
    Consider now, the spaces $L_2(\mathcal X, \rho_{\mathcal X})$ and $L_2(\mathcal Z, \nu)$, and $T_0^{\nu}:\mathcal H_0 \to \mathcal H_0$ the integral operator of $k_0$ w.r.t $\nu$
    \[
        (T_0^{\nu} f)(\bm z) = \int_{\mathcal Z } k_0(\bm z, \bm z')f(\bm z') \mathrm d\nu(\bm z'),\ \forall \bm z \in \mathcal Z
    \]
    
    Due to Mercer's decomposition we have that $k_0$ can be diagonalized into a family of eigenvalues $\{\omega_i\}_{i\in N}$ and eigenfunctions $\{\bm e_i\}_{i \in N}$ which are orthonormal in $L_2(\mathcal Z, \nu)$ and 
    \[
        k_0(\bm z, \bm z') = \sum_{i \in \mathbb N} \omega_i \bm e_i(\bm z) \bm e_i(\bm z') \ \text{ w.r.t. $\nu$},
    \]
    where $\{\omega_i\}$ is not necessarily the same family as $\{\xi_i\}$, nor $\{\bm e_i\}$ are the same eigenfunctions as $\{\bm \varphi_i\}$.
    
    We define $\bm h_i(\cdot) = \bm e_i \circ \bm \Gamma^{1/2}(\cdot)$, the integral operator $T_1$ can be written as
    \begin{align*}
        (T_1 f)(\bm x) &= \int_{\mathcal X} k_1(\bm x, \bm x')f(\bm x') \mathrm d\rho_{\mathcal X}(\bm x') \\
        & = \int_\mathcal X k_0(\bm \Gamma^{1/2}\bm x, \bm \Gamma ^{1/2}\bm x') f(\bm x') \mathrm d\rho_{\mathcal X}(\bm x')\\
        & = \int_\mathcal Z k_0(\bm z, \bm z') (f\circ \bm \Gamma^{-1/2}) ( \bm z') \mathrm d \nu(\bm z'),
    \end{align*}
    and if we plug in $f = \bm h_i$ we get
    \begin{align*}
        (T_1 f)(\bm x) &= \int_\mathcal Z k_0(\bm z, \bm z') (\bm h_i\circ \bm \Gamma^{-1/2}) ( \bm z') \mathrm d \nu(\bm z')\\
        &= \int_\mathcal Z k_0(\bm z, \bm z') (\bm e_i \circ \bm \Gamma ^{1/2}\circ \bm \Gamma^{-1/2}) ( \bm z') \mathrm d \nu(\bm z')\\
        &= \int_\mathcal Z k_0(\bm z, \bm z') \bm e_i( \bm z') \mathrm d \nu(\bm z')\\
        & = \omega_i \bm e_i(\bm z) = \omega_i \bm e_i(\bm \Gamma^{1/2}\bm x),
    \end{align*}
    so $\{\bm e_i\circ \bm \Gamma ^{1/2}\}$ are eigenfunctions of $T_1$ with the same eigenvalues from $k_0$ w.r.t. $\nu$. Also, 
    \[
        \int_\mathcal X \bm e_i(\bm \Gamma^{1/2}\bm x) \bm e_j(\bm \Gamma^{1/2}\bm x) \mathrm d\rho_{\mathcal X}(\bm x) = \int_{\mathcal Z}  \bm e_i(\bm z) \bm e_j(z) \mathrm d\nu(\bm z) = \delta_{i,j},
    \]
    which implies $\{\bm e_i\circ \bm \Gamma ^{1/2}\}$ are orthonormal in $L_2(\mathcal X, \rho_{\mathcal X})$. Therefore, we can also diagonalize $k_1$ such that 
    \[
        k_1(\bm x, \bm x') = \sum_{i\in \mathbb N} \omega_i \bm e_i(\bm \Gamma^{1/2}\bm x)\bm e_i(\bm \Gamma^{1/2}\bm x') \ \text{ w.r.t $\rho_{\mathcal X}$}.
    \]    
\end{proof}

\subsection{Proof of \cref{thm:eigenv-decay}}\label{app:thm-eigenv-decay}

\begin{proof}
    Since we work with the Gaussian measures, the domain is naturally unbounded, thus we split the proof in cases.
    \item \paragraph{Difference under bounded and unbounded domain}
    Supposing we consider an unbounded integration space $\mathcal X$, introduce the bounded set 
    \[
        B_R = \{\bm y \in \mathbb R^d: \|\bm y\| < R\},
    \]
    and we want to show that outside this set we able to control the fluctuations of the eigenvalues well enough.
    
    Since the Gaussian measures are regular, for a given $\varepsilon > 0$, we can choose $R = R(\varepsilon) > 0$ such that
    \[
        \rho_{\bm \Gamma}(B_R^c) < \varepsilon \quad \rho_{\bm I_d}(B_R^c) < \varepsilon
    \]
    and using this we write
    \[
        T_1 f(\bm x) = T_1^{R}f(\bm x) + E_1f(\bm x) = \int_{B_R} k_1(\bm x, \bm x') f(\bm x') \bm{1_{B_R}}(\bm x)\mathrm d\rho_{\bm I_d}(\bm x') + \int_{B_R^c} k_1(\bm x, \bm x') f(\bm x')\mathrm d\rho_{\bm I_d}(\bm x')
    \]
    and similarly $T_0 = T_0^R + E_0$, noting that since $T_1,T_0$ are compact and self adjoint, $T_0^R, T_1^R, E_0, E_1$ must also be.

    For the bounded domain case, we assume the following result holds: for every $k \geq 0$,
    \[
        c\lambda_k(T_0^R) \leq \lambda_k(T_1^R) \leq C\lambda_k(T_0^R).
    \]
    
    In fact, since the lower bound of the function $r$ does not depend on the norm, its value is the same for the bounded and unbounded case, thus the lower bound on the eigenvalues comes for free
    \[
        c\lambda_k(T_0) \leq \lambda_k(T_1).
    \]
    
    Furthermore, since $\sigma$ is $L_\sigma$-Lipschitz, we have $|\sigma(z) - \sigma(0)| \leq L|z|$ and
    \[
        k_1(\bm{x}, \bm{x}) \leq 2L^2 \mathbb{E}_{\bm{w}}[(\bm{w}^T \bm{x})^2] + 2\sigma(0)^2.
    \]
    Since for $k_1$, $\bm w \sim \mathcal N(0, \frac{1}{d}\bm \Gamma)$, we have
    \[
         k_1(\bm{x}, \bm{x}) \leq C\left(1+\frac{\|\bm x\|^2}{d}\right).
    \]
    for some absolute constant $C > 0$ depending on $\lambda_{\max}(\bm \Gamma), L$ and $\sigma(0)$.
    \[
        \|E_1\|_\op \leq \Trarg{E_1} \leq  \int_{B_R^c} k_1(\bm x, \bm x)\mathrm d\rho_{\bm I_d}(
        \bm x)\leq  \sqrt{\mathbb E_{\bm x}\left[ C\left(1+\frac{\|\bm x\|^2}{d}\right)^2 \right]}\sqrt{\rho_{\bm I_d}(B_R^c)}\, .
    \]
    First, we have $\sqrt{\rho_{\bm I_d}(B_R^c)} < \sqrt \varepsilon$. Also, we can calculate
    \[
        \mathbb E_{\bm x}\left[ C\left(1+\frac{\|\bm x\|^2}{d}\right)^2 \right] =C\mathbb E_{\bm x}\left[1 + \frac{2\|\bm x\|^2}{d}+ \frac{\|\bm x\|^4}{d^2} \right] \leq 6C
    \] 
    hence
    \[
        \|E_1\|_\op \leq \sqrt {6C \varepsilon} \coloneqq C_1' \sqrt{\varepsilon}\,. 
    \]
    Similarly, we have that $\|E_1\|_\op \leq C_0'\sqrt{\varepsilon}$ for some absolute constant $C_0'$.

    Using Weyl's monotonicity theorem we get, for all $k \geq 0$,
    \[
        |\lambda_k(T_1) - \lambda_k(T_1^R)| \leq \|E_1\|_\op < C_1'\sqrt{\varepsilon}
    \]
    and 
    \[
        |\lambda_k(T_0) - \lambda_k(T_0^R)| \leq \|E_0\|_\op < C_0'\sqrt{\varepsilon}
    \]
    which shows the real eigenvalues and the truncated ones are close whenever you fix a radius $R$.     
    
    \item \paragraph{Result for bounded domain}
    
    Based on the above discussion we set out to prove that the eigenvalue decay is maintained in the bounded domain. 
    
    We assume that the integration space $\mathcal X$ is bounded, i.e. $\|\bm x\| < \infty$ for all $\bm x \in \mathcal X$. 
    
    Consider the operator $\mathcal T_\#:  L^2(\rho_{\bm I_d}) \to L^2(\rho_{\bm \Gamma})$ acting on $f$ by the following rule
    \[
        \mathcal T_\# f = f \circ \bm \Gamma^{-1/2}
    \]
    which is designed to push $f$ forward to the transformed space where the new kernel is acting. Since all measures in play are Gaussian, and because $\bm \Gamma$ is full rank, we can perform a change of variables $\bm y = \bm \Gamma ^{1/2} \bm x$ to get
    \[
        \|\mathcal T_\# f\|^2_{L^2(\rho_{\bm \Gamma})} = \int |f\circ \bm \Gamma^{-1/2}(\bm y)|^2 \mathrm d\rho_{\bm \Gamma}(\bm y) = \int |f(\bm x)|^2 \mathrm d\rho_{\bm I_d}(\bm x) = \|f\|^2_{L^2(\rho_{\bm I_d})},
    \]
    thus $\mathcal T_\#$ is an isometry between both spaces and in particular an isomorphism. Similarly, we introduce its conjugate operator $(\mathcal T_\#)^* = \mathcal T_\#^{-1} := \mathcal{T^\#} : L^2(\rho_{\bm \Gamma}) \to L^2(\rho_{\bm I_d})$ which corresponds to the pullback operator
    \[
        \mathcal T^\# g = g\circ \bm \Gamma^{1/2},
    \]
    and is also an isometry between the spaces.
    
    Courant-Fischer's theorem characterizes the $k$-th eigenvalue of any operator by the identity
    \[
        \lambda_k(T_1) = \sup_{V:\dim(V) = k}\inf_{f\in V}\frac{\langle f, T_1 f\rangle}{\|f\|^2_{L^2(\rho_{\bm I_d})}}
    \]
    and we can expand quadratic form into 
    \begin{align*}
        \langle f, T_1 f\rangle_{L^2(\rho_{\bm I_d})} &= \iint f(\bm x)k_1(\bm x, \bm x')f(\bm x')\mathrm d\rho_{\bm I_d}(\bm x)\mathrm d\rho_{\bm I_d}(\bm x')\\
        &=\iint f(\bm \Gamma^{-1/2}\bm y)k_0(\bm y, \bm y')f(\bm \Gamma^{-1/2}\bm y')\mathrm d\rho_{\bm \Gamma}(\bm y)\mathrm d\rho_{\bm \Gamma}(\bm y')\\
        &=\iint f(\bm \Gamma^{-1/2}\bm y)k_0(\bm y, \bm y')f(\bm \Gamma^{-1/2}\bm y')\left[\frac{\mathrm d\rho_{\bm \Gamma}}{\mathrm d\rho_{\bm I_d}}(\bm y)\right]\mathrm d\rho_{\bm I_d}(\bm y)\left[\frac{\mathrm d\rho_{\bm \Gamma}}{\mathrm d\rho_{\bm I_d}}(\bm y')\right]\mathrm d\rho_{\bm I_d}(\bm y')\\
        &:=\iint r(\bm y)f(\bm \Gamma^{-1/2}\bm y)k_0(\bm y, \bm y')r(\bm y')f(\bm \Gamma^{-1/2}\bm y')\mathrm d\rho_{\bm I_d}(\bm y)\mathrm d\rho_{\bm I_d}(\bm y')\\
        &= \langle r\mathcal T_\# f, T_0 (r\mathcal T_\# f)\rangle_{L^2(\rho_{\bm I_d})}
    \end{align*}
    where we use the equality $k_1(\bm x, \bm x') = k_0(\bm \Gamma^{1/2}\bm x, \bm \Gamma^{1/2}\bm x')$ and $r$ is the Radon-Nikodym derivative induced by the two Gaussian measures:
    \[
        r(\bm y) := \frac{\mathrm d\rho_{\bm \Gamma}}{\mathrm d\rho_{\bm I_d}}(\bm y)= \frac{1}{\sqrt{\det \bm \Gamma}}\exp\left(-\frac{1}{2}\bm y^\top (\bm \Gamma^{-1} - \bm I_d)\bm y\right).
    \]
    Given all these definitions, we introduce the operator $\tilde T_1 := \mathcal T_\# T_1 \mathcal T^\#$, leading to the identity
    \[
        \tilde T_1 h(\bm y) = \int k_0(\bm y, \bm y') h(\bm y') \mathrm d\rho_{\bm \Gamma}(\bm y'),
    \]
    furthermore since $T_1$ and $\tilde T_1$ are related by isometries, they are unitarily equivalent and share the same eigenvalues.
    
    To continue we define the operator $M: L^2(\rho_{\bm \Gamma}) \to L^2(\rho_{\bm I_d})$ that maps $g \to r g$ and note that it is bounded and invertible since $r$ is bounded below and above under the assumption that the integration space is bounded, i.e. it is an isomorphism.
    
    If we let $h = M(g)= r g$ we can see
    \begin{align*}
        \langle g, \tilde T_1 g\rangle_{L^2(\rho_{\bm \Gamma})} &= \iint k_0(\bm y, \bm y')g(\bm y)g(\bm y') \mathrm d\rho_{\bm \Gamma}(\bm y)\mathrm d\rho_{\bm \Gamma}(\bm y')\\
        &= \iint k_0(\bm y, \bm y')[r(\bm y)g(\bm y)][r(\bm y')g(\bm y')] \mathrm d\rho_{\bm I_d}(\bm y)\mathrm d\rho_{\bm I_d}(\bm y')
    \end{align*}
    where we use that $r(\bm y)d\rho_{\bm I_d}(\bm y) = d\rho_{\bm \Gamma}(\bm y)$, and thus
    
    \begin{equation}\label{eq:quad_forms}
        \langle g, \tilde T_1 g\rangle_{L^2(\rho_{\bm \Gamma})}= \langle M(g), T_0 M(g)\rangle_{L^2(\rho_{\bm I_d})} = \langle h, T_0 h\rangle_{L^2(\rho_{\bm I_d})}.
    \end{equation}
    
    Analyzing the norms, and writing $g = \frac{1}{r}h$, we get
    \begin{align*}
        \|f\|^2_{L^2(\rho_{\bm I_d})} = \int |f(\bm x)|^2 \mathrm d\rho_{\bm I_d}(\bm x) &= \int |g(\bm y)|^2 \mathrm d\rho_{\bm \Gamma}(\bm y) \\
        &= \int \left|\frac{1}{r(\bm y)}h(\bm y)\right|^2 \mathrm d\rho_{\bm \Gamma}(\bm y)\\
        &= \int \left|\frac{1}{r(\bm x)}h(\bm y)\right|^2 r(\bm y)\mathrm d\rho_{\bm I_d}(\bm y) \\
        &=  \int \frac{1}{r(\bm y)}|h(\bm y)|^2 \mathrm d\rho_{\bm I_d}(\bm y),
    \end{align*}
    and we have that
    \begin{equation}\label{eq:norms}    
        \|f\|^2_{L^2(\rho_{\bm I_d})} = \left\|\frac{1}{\sqrt r}h\right\|^2_{L^2(\rho_{\bm I_d})}.
    \end{equation}
    Thus we can rewrite the variational form as the following identity
    
    \begin{equation}\label{eq:var_forms}
        \frac{\langle f, T_1 f\rangle_{L^2(\rho_{\bm I_d})}}{\|f\|_{L^2(\rho_{\bm I_d})}} = \frac{\langle g, \tilde T_1 g\rangle_{L^2(\rho_{\bm \Gamma})}}{\|g\|^2_{L^2(\rho_{\bm \Gamma})}} = \frac{\langle h, T_0 h\rangle_{L^2(\rho_{\bm I_d})}}{\left\|\frac{1}{\sqrt r} h\right\|^2_{L^2(\rho_{\bm I_d})}}.
    \end{equation}
    
    Since $\bm \Gamma - \bm I_d \succ 0$, we have that $\bm y^\top (\bm \Gamma^{-1} - \bm I_d)\bm y \leq 0$ for all $\bm y \in \mathbb R^d$ and
    \[
        r(\bm y) =  \frac{1}{\sqrt{\det \bm \Gamma}}\exp\left(-\frac{1}{2}\bm y^\top (\bm \Gamma^{-1} - \bm I_d)\bm y\right) \geq \frac{1}{\sqrt{\det \bm \Gamma}}, \ \forall \bm y \in \mathbb R^d.
    \]
    
    And because we assumed $\mathcal X$  to be bounded, we can find constants $c, C > 0$ such that $c\leq r(\bm y) \leq C$ for all $\bm y \in \mathcal X$, thus the relationship between the norms of $f$ and $h$ in \cref{eq:norms} gives
    \[
        \frac{1}{C}\|h\|^2_{L^2(\rho_{\bm I_d)}} \leq \|f\|^2_{L^2(\rho_{\bm I_d)}} = \|g\|^2_{L^2(\rho_{\bm \Gamma)}}\leq \frac{1}{c}\|h\|^2_{L^2(\rho_{\bm Id)}}.
    \]
    
    Combining these bounds with the identity for the quadratic forms (\ref{eq:quad_forms}), we know that for a function $g$ from a fixed subspace $V'$ with dimension $k$, if $h = M(g)$, we have
    \[
        c\frac{\langle h, T_0 h \rangle}{\|h\|^2_{ L^2(\rho_{\bm I_d})}} \leq \frac{\langle g, \tilde T_1 g \rangle}{\|g\|^2_{ L^2(\rho_{\bm \Gamma})}} \leq C\frac{\langle h, T_0 h \rangle}{\|h\|^2_{ L^2(\rho_{\bm I_d})}}.
    \]
    
    Now, given a subspace $V' \subset L^2(\rho_{\bm \Gamma})$ of dimension $k$, we define
    \[
        W = M(V') := \{rg : g \in V'\} \subset L^2(\rho_{\bm I_d}).
    \]
    Because $M$ is bijective and isomorphic, $W$ has dimension $k$ and the mapping $V' \mapsto W$ is a bijection, allowing us to associate every subspace $V' \subset L^2(\rho_{\bm \Gamma})$ with a respective unique subspace $W \subset  L^2(\rho_{\bm I_d})$. 
    
    Thus, taking the infimum over all functions inside $V'$ is the same as taking the infimum over the associated subspace $W$ and 
    \[
        c\left(\inf_{h \in W}\frac{\langle h, T_0 h \rangle}{\|h\|^2_{ L^2(\rho_{\bm I_d})}}\right) \leq \inf_{g \in V'}\frac{\langle g, \tilde T_1 g \rangle}{\|g\|^2_{ L^2(\rho_{\bm \Gamma})}}  \leq C\left(\inf_{h \in W}\frac{\langle h, T_0 h \rangle}{\|h\|^2_{ L^2(\rho_{\bm I_d})}}\right),
    \]
    and finally taking the supremum over all possible sets $V'$ of dimension $k$ is the same as taking the supremum over the associated sets $W$ and we get
    \[
        c\left(\sup_{\dim W = k}\inf_{h \in W}\frac{\langle h, T_0 h \rangle}{\|h\|^2_{ L^2(\rho_{\bm I_d})}}\right) \leq \sup_{\dim V' = k}\inf_{g \in V'}\frac{\langle g, \tilde T_1 g \rangle}{\|g\|^2_{ L^2(\rho_{\bm \Gamma})}}  \leq C\left(\sup_{\dim W = k}\inf_{h \in W}\frac{\langle h, T_0 h \rangle}{\|h\|^2_{ L^2(\rho_{\bm I_d})}}\right),
    \]
    which are exactly the expressions for the eigenvalues of $\lambda(T_0)$ and $\lambda_k(\tilde T_1) = \lambda_k(T_1)$, giving the desired result
    \[
        c \lambda_k(T_0) \leq \lambda_k(\tilde T_1) = \lambda_k(T_1) \leq C\lambda_k(T_0).
    \]

\end{proof}
\subsection{Proof of Theorem \ref{thm:cov-expansion}}\label{app:thm-cov-expansion}
\begin{proof}
    First, we denote $\bm \Lambda := \bm I_d + \frac{\gamma_2}{\gamma_1}\bm u \bm u^\top$. Then by Sherman–Morrison, for $\bm \Sigma = \gamma_1 \bm \Lambda \bm I_d$, we have:
    \[
        \bm \Sigma^{-1} = \frac{1}{\gamma_1}\bm I_d - \frac{\gamma_2}{\gamma_1(\gamma_1+\gamma_2)}\bm u \bm u^\top,
    \]
    and, in particular, the determinants are given by 
    \[
        \det(\bm \Sigma) = \det (\bm \Lambda) \det (\gamma_1 \bm I_d) = \frac{\gamma_1+\gamma_2}{\gamma_1}\det (\gamma_1 \bm I_d).
    \]
    Therefore, we write the expectation of $G$ using the explicit Gaussian density under covariance $\bm \Sigma$ and use the Taylor expansion of the exponential function to obtain
    \begin{align*}
        \mathbb E_{\bm w \sim \mathcal N(0, \bm \Sigma)}[G(\bm w)] &= \int G(\bm w)\frac{1}{(2\pi)^{d/2}\sqrt{\det{\bm \Sigma}}}\exp\left({-\frac{\bm w^\top \bm \Sigma^{-1}\bm w}{2}}\right)\mathrm d\bm w\\
        &=\int \sqrt{\frac{\gamma_1}{\gamma_1+\gamma_2}}\exp\left(\frac{\gamma_2}{2\gamma_1(\gamma_1+\gamma_2)}\langle \bm u, \bm w\rangle^2 \right)\frac{G(\bm w)}{(2\pi)^{d/2}\sqrt{\det \gamma_1\bm I_d}}e^{-\frac{\|\bm w\|^2}{2\gamma_1}}\mathrm d\bm w\\
        &=\sqrt{\frac{\gamma_1}{\gamma_1+\gamma_2}} \int \sum_{k=0}^\infty \frac{1}{k!}\left(\frac{\gamma_2}{2\gamma_1(\gamma_1+\gamma_2)}\right)^k \langle \bm u, \bm w\rangle^{2k} G(\bm w)\mathrm d\rho_{\gamma_1 \bm I_d}(\bm w)\\
        &=\sqrt{\frac{\gamma_1}{\gamma_1+\gamma_2}}\sum_{k=0}^\infty \frac{1}{k!}\left(\frac{\gamma_2}{2\gamma_1(\gamma_1+\gamma_2)}\right)^k\int \langle \bm u, \bm w\rangle^{2k} G(\bm w)\mathrm d\rho_{\gamma_1 \bm I_d}(\bm w)\\
        &= \sqrt{\frac{\gamma_1}{\gamma_1+\gamma_2}}\sum_{k=0}^\infty \frac{1}{k!}\left(\frac{\gamma_2}{2\gamma_1(\gamma_1+\gamma_2)}\right)^k \mathbb E_{\bm w \sim \mathcal N(0, \gamma_1\bm I_d)}[\langle \bm u, \bm w\rangle^{2k}G(\bm w)],
    \end{align*}
    where we used the relationship between the determinants in the second passage.
    
    Next, we justify the application of Fubini's theorem. If we denote $\phi_{\gamma_1}(\bm w) = e^{-\frac{\|\bm w\|^2}{2\gamma_1}}$ and define
    \[
        f(k, \bm w) = \frac{1}{k!}\left(\frac{\gamma_2}{2\gamma_1(\gamma_1+\gamma_2)}\right)^k \langle \bm u, \bm w\rangle^{2k} G(\bm w)\mathrm \phi_{\gamma_1}(\bm w),
    \]
    and we check its sufficient condition: absolute integrability
    \[
        \int \sum_{k \geq 0} |f(k, \bm w)| \mathrm d \bm w < \infty.
    \]

    Looking at this sum we can see
    \begin{align*}
        \sum_{k \geq 0} |f(k, \bm w)| \mathrm d \bm w  &= \sum_{k\geq 0}\frac{1}{k!}\left|\left(\frac{\gamma_2}{2\gamma_1(\gamma_1+\gamma_2)}\right)^k \langle \bm u, \bm w\rangle^{2k} G(\bm w)\mathrm \phi_{\gamma_1}(\bm w)\right| \\
        &\leq \sum_{k\geq 0} \frac{1}{k!}\left|\left(\frac{\gamma_2}{2\gamma_1(\gamma_1+\gamma_2)}\right) \langle \bm u, \bm w\rangle^{2}\right|^k |G(\bm w)||\mathrm \phi_{\gamma_1}(\bm w)| \\
        &= \exp\left(\left|\frac{\gamma_2}{2\gamma_1(\gamma_1+\gamma_2)}\right| \langle \bm u, \bm w\rangle^{2}\right) |G(\bm w)\mathrm| \phi_{\gamma_1}(\bm w)\\
        &= \exp\left(\left|\frac{\gamma_2}{2\gamma_1(\gamma_1+\gamma_2)}\right| \langle \bm u, \bm w\rangle^{2} - \frac{\|\bm w\|^2}{2\gamma_1}\right) |G(\bm w)|.
    \end{align*}

    If this exponent is negative, the Gaussian density is able to suppress the polynomial growth of $G$ under the integral. We always have $\langle \bm u, \bm w \rangle \leq \| \bm w \|$ because of the unit vector $\bm u$, which leads to
    \[
     \exp\left(\left|\frac{\gamma_2}{2\gamma_1(\gamma_1+\gamma_2)}\right|  - \frac{1}{2\gamma_1}\right) = \exp\left(\frac{|\gamma_2| - (\gamma _1 + \gamma_2)}{2\gamma_1(\gamma_1+\gamma_2)}\right)\,.
    \]
    Due to the assumption that $\gamma_1 + \gamma_2 > |\gamma_2|$, the exponent is always negative and consequently
    \[
        \int \sum_{k \geq 0} |f(k, \bm w)| \mathrm d \bm w < \infty.
    \]

    We now study the terms
    \[
         \mathbb E_{\bm w \sim \mathcal N(0, \gamma_1\bm I_d)}[\langle \bm u, \bm w\rangle^{2k}G(\bm w)].
    \]
    By using Stein's Lemma exhaustively we can obtain
    \[
        \mathbb E_{\bm w \sim \mathcal N(0, \gamma_1\bm I_d)}[\langle \bm u, \bm w\rangle^{2k}G(\bm w)] = \sum_{n=0}^k \gamma_1^{k+n}\binom{2k}{2n}(2k -2n-1)!! \mathbb{E}[ D_{\bm u}^{(2n)} G(\bm{w})],
    \]
    and the original expression is given by
    \[
        \mathbb E_{\bm w \sim \mathcal N(0, \bm \Sigma)}[G(\bm w)] = \sqrt{\frac{\gamma_1}{\gamma_1+\gamma_2}}\sum_{k=0}^\infty\frac{1}{k!}\left(\frac{\gamma_2}{2\gamma_1(\gamma_1+\gamma_2)}\right)^k\left[\sum_{n=0}^k \gamma_1^{k+n}\binom{2k}{2n}(2k -2n-1)!! \mathbb{E}[ D_{\bm u}^{(2n)} G(\bm{w})]\right].
    \]
    Since the exchange of the integral and the outer series was justified previously, and for each fixed $k$ the sum over $n$ is finite, the resulting double series is absolutely summable under the same domination bound; hence we may interchange the order of summation.
    \[
        \mathbb E_{\bm w \sim \mathcal N(0, \bm \Sigma)}[G(\bm w)] = \sqrt{\frac{\gamma_1}{\gamma_1+\gamma_2}}\sum_{n=0}^\infty \mathbb{E}[ D_{\bm u}^{(2n)} G(\bm{w})]\left[\sum_{k=n}^\infty\frac{\gamma_1^{n}}{k!}\left(\frac{\gamma_2}{2(\gamma_1+\gamma_2)}\right)^k\binom{2k}{2n}(2k -2n-1)!! \right].
    \]
    Now, due to Lemma \ref{lemma:ith-inf-series} we have the closed form expression for each infinite sum given a fixed $i$, and if we choose $y = \frac{\gamma_2}{2(\gamma_1+\gamma_2)}$:
    \begin{align*}    
        \sum_{k=n}^\infty\frac{\gamma_1^{n}}{k!}\left(\frac{\gamma_2}{2\gamma_1(\gamma_1+\gamma_2)}\right)^k\binom{2k}{2n}(2k -2n-1)!! &= \frac{\gamma_1^{n}}{n!}\left(\frac{\gamma_2}{2(\gamma_1+\gamma_2)}\right)^n\left(\frac{\gamma_1+\gamma_2}{\gamma_1}\right)^{(n+1/2)}\\
        &=\frac{1}{n!}\left(\frac{\gamma_2}{2}\right)^n\sqrt{\frac{\gamma_1+\gamma_2}{\gamma_1}}.
    \end{align*}
    
    Thus, substituting this back into the expansion gives
    \[
        \mathbb E_{\bm w \sim \mathcal N(0, \bm \Sigma)}[G(\bm w)] = \sum_{i=0}^\infty \frac{1}{n!}\left(\frac{\gamma_2}{2}\right)^n\mathbb{E}_{\bm w \sim \mathcal N(0, \gamma_1\bm I_d)}[ D_{\bm u}^{(2n)} G(\bm{w})]
    \]
    and the proof is complete.
\end{proof}
\subsection{Proof of Lemma \ref{lemma:T_R-op-bound}}\label{app:lemma-T_R-op-bound}
\begin{proof}
    
    For $t \in [0, 1]$, we define the parameterized family of covariance matrices
    \[
        \bm \Gamma(t) = A \bm I_d + t B \bm w^* (\bm w^*)^\top,
    \]
    and, if we let $\phi$ be the Gaussian density function induced by $\mathcal N(0, \frac{1}{d} \bm \Gamma(t))$, we define the scalar function $F:[0,1] \to \mathbb R$ given by
    \[
        F(t) = \int G(\bm w)\phi(\bm w, t)\mathrm d \bm w.
    \]
    
    Since $\phi \in C^\infty$, for every $n \geq 0$, the $n$-th derivative of $F$ is well-defined through the distributional property
    \[
        F^{(n)}(t) = \int G(\bm w)\frac{\mathrm d^n}{\mathrm d t^n}[\phi(\bm w, t)]\mathrm d \bm w,
    \]
    and $F \in C^\infty$. Expanding the Taylor series of $F$ around $t=0$ and evaluating at $1$, Taylor's Theorem guarantees the exact identity
    \[
        F(1) = F(0) + F'(0) + \frac{F''(\xi)}{2!}
    \]
    with $\xi \in [0, 1]$.
        
    By Price's Theorem \citep{Price1958, McMahonPrice1964}, we have the formal identity
    \[
        F^{(n)}(t) = \left(\frac{B}{2d}\right)^n\mathbb E_{\bm w \sim \mathcal N(0, \frac{1}{d}\bm \Gamma(t))}[D_{\bm w^*}^{(2n)} G(\bm w) ],
    \]
    therefore
    \[
        F''(\xi) = \left(\frac{B}{2d}\right)^2 \mathbb E_{\bm w \sim \mathcal N(0, \frac{1}{d}\bm \Gamma(\xi))}[D_{\bm w^*}^{(4)} G(\bm w) ],
    \]
    and if we show the fourth derivative under the expectation is bounded independently of the dimension, the multiplying factor will ensure the asymptotic profile of the result.

    Define the quantities
    \[
        h_1 = \langle\bm w, \bm x\rangle \quad h_2 = \langle\bm w, \bm x'\rangle
    \]
    which are joint Gaussian variables such that $(h_1, h_2) := \bm h \sim \mathcal{N}(0, \bm Q_\xi)$ where
    \[
        \bm Q_\xi = \frac{1}{d}\begin{bmatrix}
            \bm x^\top \bm \Gamma(\xi) \bm x & \bm x^\top \bm \Gamma(\xi) \bm x'\\
            \bm x^\top \bm \Gamma(\xi) \bm x' & (\bm x')^\top \bm \Gamma(\xi) \bm x'
        \end{bmatrix}.
    \]
    
    If $\phi_{\bm \Gamma(\xi)}(\bm w)$ denotes the Gaussian density from the measure $\mathcal N(0, \bm \Gamma(\xi))$, let $\phi(h_1, h_2)$ denote the probability density function of this 2D Gaussian.
    
    Then we can write
    \begin{align*}
        \mathbb E_{\bm w \sim \mathcal N(0, \frac{1}{d}\bm \Gamma(t))}[D_{\bm w^*}^{(4)} G(\bm w) ] &= \int D_{\bm w^*}^{(4)} G(\bm w)\phi_{\bm \Gamma(\xi)}(\bm w)\mathrm d \bm w \\
        &= \int D_{\bm w^*}^{(4)} [\sigma(\langle \bm w, \bm x\rangle)\sigma(\langle \bm w, \bm x'\rangle)] \phi_{\bm \Gamma(\xi)}(\bm w)\mathrm d \bm w \\
        &= \iint D_{\bm w^*}^{(4)} [\sigma(h_1)\sigma(h_2)] \phi(h_1, h_2) \mathrm d h_1 \mathrm dh_2.   
    \end{align*}
    
    Next, by the chain rule, applying the directional derivative $D_{\bm w^*} = \langle \bm w^* , \nabla_{\bm w}\rangle$ yields the two dimensional operator over $h_1$ and $h_2$
    \[
        D_{\bm w^*} = \left( \langle \bm w^*, \bm x\rangle  \frac{\partial}{\partial h_1} + \langle \bm w^*, \bm x'\rangle \frac{\partial}{\partial h_2} \right) := \langle \bm U, \nabla_{\bm h}\rangle,
    \]
    where we define
    \[
        \bm U = [\langle \bm w^*, \bm x\rangle, \langle \bm w^*, \bm x'\rangle]^{\top} \in \mathbb R^2.
    \]

    Since $\sigma$ is only Lipschitz, we have no information about its higher order derivatives. Remember that, for any locally integrable $f$ and any smooth compactly supported test function $\varphi$,
    \[
        \langle D_{w^*} f, \varphi \rangle = - \langle f, D_{w^*} \varphi \rangle .
    \]
    Iterating,
    \[
        \langle D_{w^*}^n f, \varphi \rangle = (-1)^n \langle f, D_{w^*}^n \varphi \rangle .
    \]
    Hence, we compute the expectation as a 2D integral and, integrating by parts, we transfer the derivative operator to the density function. 
    \[
        \iint \left[ D_{\bm w^*}^4 \Big( \sigma(h_1)\sigma(h_2) \Big) \right] \phi(h_1, h_2) \, \mathrm d h_1 \mathrm d h_2 =  \iint \sigma(h_1)\sigma(h_2) \left[ D_{\bm w^*}^4 \phi(h_1, h_2) \right] \, \mathrm d h_1 \mathrm d h_2.
    \]

    Now, the closed form of the density is given by
    \[
        \phi(h_1, h_2) = \phi(\bm h) = \frac{1}{2\pi\sqrt{ \det \bm Q(\xi)}} \exp\left( -\frac{1}{2} \bm h^T \bm Q_\xi^{-1} \bm h \right)
    \]
    and differentiating this with our notation gives
    \[
        \nabla_{\bm h}\phi(h_1, h_2) = \nabla_{\bm h}\phi(\bm h) = \phi(\bm h)\nabla_{\bm h}\left( -\frac{1}{2} \bm h^T Q_\xi^{-1} \bm h \right) = -\phi(\bm h) \bm Q_\xi^{-1}\bm h
    \]
    which implies the directional derivative is given by 
    \[
        D_{\bm w^*}\phi(h_1, h_2)  = -\langle \bm U, \bm Q_\xi^{-1}\bm h\rangle \phi(\bm h).
    \]

    Therefore, iterating through this 4 times we obtain 
    \[
        D_{\bm w^*}^4 \phi(h_1, h_2) = \left(\langle \bm U, \bm Q_\xi^{-1} \bm h\rangle^4 - 6\langle \bm U, \bm Q_\xi^{-1} \bm h\rangle^2 \langle \bm U, \bm Q_\xi^{-1} \bm U\rangle + 3\langle \bm U, \bm Q_\xi^{-1} \bm U\rangle^2  \right)\phi(\bm h),
    \]
    and, if we denote the polynomial multiplying $\phi$ by $\mathrm P_4(\bm U, \bm h, \bm Q_{\xi})$, we can see that
    \[
        |\mathrm P_4(\bm U, \bm h, \bm Q_{\xi})| \leq \|\bm U\|^4 \left(\|\bm Q^{-1}_\xi\|^4_{\text{op}}\|\bm h\|^4 + 6 \|\bm Q^{-1}_\xi\|^3_{\text{op}}\|\bm h\|^2 + 3\|\bm Q^{-1}_\xi\|^2_{\text{op}}\right).
    \]

    Using that $\sigma$ is $L_{\sigma}$-Lipschitz we have,
    \[
        |\sigma(t)| \le M_{L_{\sigma}}(1 + |t|)
    \]
    where $M_{L_{\sigma}} = \max(|\sigma(0)|, L_{\sigma})$, therefore the product is bounded by
    \[
        |\sigma(h_1)\sigma(h_2)| \le M_{L_{\sigma}}^2(1 + |h_1|)(1 + |h_2|) \leq M_{L_{\sigma}}^2(2+ \|\bm h\|^2).
    \]

    Thus, the absolute value of $R$ can be bounded with
    \begin{align*}
        |R(\bm x, \bm x')| &\leq \frac{B^2}{8d^2}\iint |\sigma(h_1)\sigma(h_2)| |\mathrm P_4(\bm U, \bm h, \bm Q_{\xi})| \phi(h_1, h_2) \, \mathrm d h_1 \mathrm d h_2\\
        & \leq \frac{B^2}{8d^2}\iint M_{L_{\sigma}}^2(2 + \|\bm h\|^2)  |\mathrm P_4(\bm U, \bm h, \bm Q_{\xi})| \phi(h_1, h_2) \, \mathrm d h_1 \mathrm d h_2\\
        &\leq \frac{B^2}{8d^2}P^*(\bm U, \bm Q_{\xi}),
    \end{align*}
    where we define
    \begin{align*}
        P^*(\bm U, \bm Q_{\xi}) &= M_L^2 \|\bm U\|^4 \cdot \mathbb{E}_{\bm h} \left[ (2 + \|\bm h\|^2)(\|\bm Q_\xi^{-1}\|_{\text{op}}^4\|\bm h\|^4 + 6\|\bm Q_\xi^{-1}\|_{\text{op}}^3\|\bm h\|^2 + 3\|\bm Q_\xi^{-1}\|_{\text{op}}^2) \right]\\
        &=M_L^2 \|\bm U\|^4 \cdot \mathbb{E}_{\bm h} \Big[ \|\bm Q_\xi^{-1}\|_{\text{op}}^4\|\bm h\|^6 + (2\|\bm Q_\xi^{-1}\|_{\text{op}}^4 + 6\|\bm Q_\xi^{-1}\|_{\text{op}}^3)\|\bm h\|^4 \\
        &\qquad +(3\|\bm Q_\xi^{-1}\|_{\text{op}}^2 + 12\|\bm Q_\xi^{-1}\|_{\text{op}}^3)\|\bm h\|^2 + 6\|\bm Q_\xi^{-1}\|_{\text{op}}^2 \Big].
    \end{align*}

    Next, we note that since $\bm h \sim \mathcal N(0, \bm Q_{\xi})$ we know
    \[
        \mathbb{E}[\|\bm{h}\|^{2p}] \leq (2p-1)!! \cdot (2\|\bm Q_\xi\|_{op})^p,
    \]
    therefore distributing the expectation operator and bounding every moment of $\|\bm h\|$ we have
    \begin{align}\label{eq:K'-bound}
        P^*(\bm U, \bm Q_{\xi}) &\leq M_{L_{\sigma}}^2 \|\bm U\|^4 \cdot \Big[ (120 \|\bm Q_\xi^{-1}\|_{\text{op}}^4\|\bm Q_\xi\|_{\text{op}}^3 + 12(2\|\bm Q_\xi^{-1}\|_{\text{op}}^4\| + 6\|\bm Q_\xi^{-1}\|_{\text{op}}^3)\|\bm Q_\xi\|_{\text{op}}^2 \nonumber \\
        &\qquad + 2(2\|\bm Q_\xi^{-1}\|_{\text{op}}^2 + 6\|\bm Q_\xi^{-1}\|_{\text{op}}^3)\|\bm Q_\xi\|_{\text{op}} +4\|\bm Q_\xi^{-1}\|_{\text{op}}^2 \Big].
    \end{align}

    To understand the operator norm of $T_R$ we bound its HS norm. Since the absolute value of $R$ is bounded by $P^*$, we have
    \[
        \|T_R\|^2_{\text{HS}} \leq \frac{B^4}{64d^4}\mathbb E_{\bm x, \bm x'}[P^*(\bm U, \bm Q_{\xi})^2] \, .
    \]
    To bound this value, we first note that, since $\bm w^*$ is a unit vector,
    \[
        \|\bm U\|^{2} = \langle \bm w^*, \bm x\rangle^2 + \langle \bm w^*, \bm x'\rangle^2 := z_1^2 + z_2^2
    \]
    where $z_1, z_2 \sim \mathcal N(0, 1)$ are under i.i.d. Gaussian input distribution and therefore
    \[
        \mathbb E_{\bm x, \bm x'}[\|\bm U\|^k] = \mathbb E[(z_1^2 + z_2^2)^{k/2}] = \mathcal O(1),
    \]
    because the expectation simplifies to a sum of one dimensional Gaussian moments, which are independent of the dimension $d$.
    
    Next, we bound the moments of operator norm $\bm Q_{\xi}$ and its inverse. If we define $\tilde {\bm X} = [\bm x, \bm x'] \in \mathbb R^{d \times 2}$ and let $\tilde {\bm W} = \tilde {\bm X}^\top \tilde {\bm X}$, we have that
    \[
        \bm Q_\xi = \frac{1}{d}\tilde{\bm X}^\top \bm \Gamma(\xi) \tilde{\bm X}
    \]
    and since $\bm \Gamma(\xi)$ is PSD and $A + \xi B > 0$, the eigenvalues of this matrix must be strictly positive. Let $\gamma_{\min} = A$ denote the minimum eigenvalue of $\bm \Gamma(\xi)$. Then, for the inverse matrix we have
    \[
        \bm Q_{\xi}^{-1} \preceq \frac{d}{\gamma_{\min}} \tilde{\bm W}^{-1}.
    \]
    which implies the following bounds on the moments
    \[
        \mathbb{E}_{\bm x, \bm x'} \left[ \|\bm Q_\xi^{-1}\|_{\text{op}}^k \right] \le \left(\frac{d}{\gamma_{\min}}\right)^k \mathbb{E}_{\bm W} \left[ \|\tilde{\bm W}^{-1}\|_{\text{op}}^k \right].
    \]

    Because $\tilde{\bm X}$ is a $d \times 2$ matrix of independent $\mathcal{N}(0, 1)$ entries, the inner product matrix follows a Wishart distribution: $\tilde{\bm W} \sim \mathcal{W}_2(\bm I_2, d)$ and, consequently, its inverse follows an Inverse-Wishart distribution $\tilde{\bm W}^{-1} \sim \mathcal{W}^{-1}_2(\bm I_2, d)$.

    For an $n \times n$ Inverse-Wishart matrix with $d$ degrees of freedom, the $k$-th moments are strictly finite and integrable as long as the degrees of freedom exceed the matrix size, i.e. $d > n + 2k - 1$. In our case, $\tilde{\bm W}$ is $2 \times 2$, therefore we have
    \[
       \mathbb{E}_{\tilde{\bm W}} \left[ \|\tilde{\bm W}^{-1}\|_{\text{op}}^k \right] = \mathcal O\left(\frac{1}{d^k}\right)
    \]
    as long as $d > 2k + 1.$

    As a consequence, since $A$ does not depend on $d$, for a fixed $k$, as the dimension grows, the operator norm is bounded independently of $d$, i.e.
    \[
        \mathbb{E}_{\bm x, \bm x'} \left[ \|\bm Q_\xi^{-1}\|_{\text{op}}^k \right] = \mathcal O(1)\, .
    \]
    To bound the non-inverse moments, we decompose the quadratic form directly
    \[
        \bm{Q}_\xi = \frac{1}{d} \tilde{\bm{X}}^\top (A\bm{I}_d + B\bm{w^*}(\bm{w^*})^\top) \tilde{\bm{X}} = \frac{A}{d} \tilde{\bm{W}} + \frac{B}{d} \bm{z}\bm{z}^\top
    \]
    where $\bm{z} = \tilde{\bm{X}}^\top \bm{w^*} \in \mathbb{R}^2$. Because the columns of $\tilde{\bm{X}}$ are standard Gaussian and $\bm{w^*}$ is a deterministic unit vector, the projection $\bm{z}$ is distributed as a standard 2-dimensional Gaussian, $\bm{z} \sim \mathcal{N}(0, \bm{I}_2)$.

    Applying the triangle inequality to the operator norm we obtain
    \[
        \|\bm{Q}_\xi\|_{\text{op}} \le \frac{A}{d} \|\widetilde{\bm{W}}\|_{\text{op}} + \frac{B}{d} \|\bm{z}\bm{z}^\top\|_{\text{op}} = \frac{A}{d} \|\widetilde{\bm{W}}\|_{\text{op}} + \frac{B}{d} \|\bm{z}\|_2^2\, .
    \]
    Raising the norm to the $k$-th power and applying the inequality $(a+b)^k \le 2^{k-1}(a^k + b^k)$ yields
    \[
        \mathbb{E}_{\bm x, \bm x'}[\|\bm{Q}_\xi\|_{\text{op}}^k] \le 2^{k-1} \left( \frac{A^k}{d^k} \mathbb{E}[\|\widetilde{\bm{W}}\|_{\text{op}}^k] + \left(\frac{B}{d}\right)^k \mathbb{E}[\|\bm{z}\|_2^{2k}] \right)
    \]
    Given $B = \Theta(\eta^2/d)$, with $\eta = \Theta(d^\zeta)$ and $\zeta \in [1/2, 1)$, we have $B = \Theta(d^{2\zeta - 1})$ and the coefficient scales as $\frac{B}{d} = \Theta(d^{2\zeta-2})$. Because $\zeta < 1$ by assumption, it follows that $2\zeta - 2 < 0$ and $\frac{B}{d} = o_d(1)$. Furthermore, because $\|\bm{z}\|_2^2$ follows a chi-squared distribution with 2 degrees of freedom, its moments $\mathbb{E}[\|\bm{z}\|_2^{2k}]$ are constants dependent only on $k$, bounded by $\mathcal{O}(1)$. Hence, we also have
    \[
        \mathbb{E}_{\bm x, \bm x'} \left[ \|\bm Q_\xi\|_{\text{op}}^k \right] = \mathcal O(1)\, .
    \]
    
    Furthermore, for fixed exponents $k_1, k_2 \geq 0$, by Cauchy-Schwarz we are able to obtain
    \[
        \mathbb E_{\bm x, \bm x'}\left[\|\bm Q_{\xi}^{-1}\|_{\text{op}}^{k_1}\|\bm Q_{\xi}\|_{\text{op}}^{k_2}\right] = \mathcal O(1).
    \]    
    
    Lastly, notice that if we square the bound in \ref{eq:K'-bound}, we obtain a combination of powers of $\|\bm U\|, \|\bm Q_{\xi}^{-1}\|_{\text{op}}$ and $\|\bm Q_{\xi}\|_{\text{op}}$. Applying Holder's inequality to decouple the expectations on each of the terms of the inequality and collecting all the bounds we have on the moments of these quantities we get
    \[
        \mathbb E_{\bm x, \bm x'}[P^*(\bm U, \bm Q_\xi)^2] = \mathcal O(1).
    \]
    
    Thus, we have the following result for the HS norm
    \begin{align*}
        \|T_R\|_{\text{HS}}^2 &= \iint |R(\bm x, \bm x')|^2\mathrm d \rho_{\mathcal X}(\bm x)\mathrm d \rho_{\mathcal X}(\bm x')\\
        &\leq  \frac{B^4}{64d^4} \mathbb E_{\bm x, \bm x'}[P^*(\bm U, \bm Q_{\xi})^2]\\
        &=\mathcal O\left(\frac{B^4}{d^4}\right)
    \end{align*}
    which implies
    \[
        \| T_R\|_{\text{op}} \leq \|T_R\|_{\text{HS}} = \mathcal O\left(\frac{B^2}{d^2}\right)
    \]
    and the proof is completed.
\end{proof}
\subsection{Proof of Lemma \ref{lemma:iso-eigenfn}}
\begin{proof}
    We write the ReLU kernel as its closed-form
    \[
        k_0(\bm x, \bm x') = \frac{\|\bm x\|\|\bm x'\|}{2 \pi d}\left[\gamma(\pi -\arccos (\gamma)) +\sqrt{1 - \gamma^2}\right] := \frac{r r'}{2 \pi d}J_1(\gamma)
    \]
    where $r = \|\bm x\|$ and $r' = \|\bm x'\|$ and $J_1$ is the standard first-order arc cosine kernel.
    
    In this context, the function $J_1(\gamma)$ is a dot-product kernel on the unit sphere and, by the Funk-Hencke's theorem, $J_1$ can be expanded into the spherical harmonics basis. Therefore, if we denote $\bm \omega = \frac{\bm x}{\|\bm x\|}$ and $\bm \omega' = \frac{\bm x'}{\|\bm x'\|}$, we have
    \[
        k_0(\bm x, \bm x') = \frac{r r'}{2 \pi d} \sum_{k=0}^{\infty} \lambda_k \sum_{m} Y_{k,m}(\bm \omega) Y_{k,m}(\bm \omega').
    \]
    By moving the $r$ and $r'$ inside the sum, we get the exact Mercer expansion over the Gaussian space
    \[
        k_0(\bm x, \bm x') = \frac{1}{2 \pi}\sum_{k=0}^{\infty} \sum_{m} \lambda_k \left( \frac{r}{\sqrt {d}} Y_{k,m}(\bm \omega) \right) \left( \frac{r'}{\sqrt{d}} Y_{k,m}(\bm \omega') \right).
    \]
    Note that these functions are indeed normalized and pairwise orthogonal since the inner product of two elements of this family can always be separated into an integral acting only on the radii and an integral acting only on the angles. Because the spherical harmonics are pairwise orthogonal, the angular integral will enforce the orthogonality conditions.
    
    Therefore, since this expansion fully constructs the kernel, the only eigenfunctions with non-zero eigenvalues are of the form $\frac{r}{\sqrt d} Y_{k,m}(\bm \omega)$.
\end{proof}
\subsection{Proof of Lemma \ref{lemma:approx-first-order-op}}\label{app:lemma-first-order-op}
\begin{proof}
    For this proof we write
    \begin{align*}
        T_1 f(\bm x) &= AT_0 f(\bm x) + \frac{B}{2d} T_{S} f(\bm x)+  T_R f(\bm x)\\
        &= AT_0 f(\bm x) + \frac{B}{2d} \left(T_{(1*)} f(\bm x) + T_{(2*)} f(\bm x)\right) +  T_R f(\bm x)
    \end{align*}
    and, as we have seen by Lemma \ref{lemma:T_R-op-bound}, $\|T_Rf\|_{L^2(\rho_\mathcal X)} = o_d\left(\frac{|B|}{d}\right)$. Thus, we will study the terms $T_{(1*)}f$ and $T_{(2*)} f$.
    
    {\bf Analysis of the first term:}
    We start by analyzing
    \[
        T_{(1*)}f(\bm x) =\langle \bm x, \bm w^*\rangle\int \frac{(\pi - \theta_{\bm x, \bm x'})}{\pi}\langle \bm x', \bm w^*\rangle f(\bm x')\mathrm d\rho_{\mathcal X}(\bm x'),
    \]    
    and focus on the term under the integral. Note that $f$ and $\langle \cdot, \bm w^*\rangle$ are normalized functions such that $\|f\|_{L^2(\rho_{\mathcal X})} = \|\langle \cdot, \bm w^*\rangle\|_{L^2(\rho_{\mathcal X})} = 1.$
    
    If we consider the set 
    \[
        \mathcal A = \left\{\bm x' \in \mathbb R^d: \left|\theta_{\bm x, \bm x'} - \frac{\pi}{2}\right|< \epsilon \right\}, \epsilon \in \left(0, \frac{\pi}{2}\right),
    \]
    we split the integral over the entire space into the sum of integrals over $\mathcal A$ and $\mathcal A^c$.

    By Lemma \ref{lemma:integral-concentration-1}, if we let $g = \langle \cdot, \bm w^*\rangle f$, for a given $\epsilon > 0$, the concentration of the Gaussian measure implies
    \[
        \left|\int_{\mathcal A^c} \frac{(\pi - \theta_{\bm x, \bm x'})}{\pi} \langle \bm x', \bm w^*\rangle f(\bm x') \mathrm d\rho_{\mathcal X}(\bm x')\right| \leq \sqrt 2e^{-c_1d\epsilon^2/2}\|g\|_{L^2(\rho_{\mathcal X})}
    \]
    for an absolute constant $c_1 > 0$, provided $g \in L^2(\rho_{\mathcal X})$.
    We can see that
    \[
        \int |\langle \bm x', \bm w^*\rangle f(\bm x')|^2\mathrm d \rho_{\mathcal X}(\bm x')  = \int |\|\bm x'\|^2 (\cos\theta_{\bm x', \bm w^*})^2 f(\bm x')|^2\mathrm d \rho_{\mathcal X}(\bm x') \leq \int \|\bm x'\|^2 |f(\bm x')|^2\mathrm d \rho_{\mathcal X}(\bm x').
    \]
    Now, due to Lemma \ref{lemma:iso-eigenfn}, the orthonormal eigenbasis is of the separated form
    \[
        \psi_{k,m}(\bm{x}) = \frac{\|\bm x\|}{\sqrt d}Y_{k,m}(\bm \omega),
    \]
    a purely radial function governing the magnitude and $Y_{k,m}(\bm \omega)$ is a spherical harmonic of degree $k$, governing the direction $\bm \omega = \frac{\bm x}{\|\bm x\|}$.

    Although the ReLU kernel is not universal on the Gaussian function space, by assumption we can write $f$ as a linear combination of elements of the eigenbasis of the kernel
    \[
        f(\bm x) = \sum_{k,m \geq 0} c_{k, m}\psi_{k, m}(\bm x)
    \]
    therefore
    \begin{align*}
        |f(\bm{x})|^2 &= \left( \sum_{k,m} c_{k,m} \frac{\|\bm{x}\|}{\sqrt{d}} Y_{k,m}(\bm \omega) \right) \left( \sum_{k',m'} \overline{c_{k',m'}} \frac{\|\bm{x}\|}{\sqrt{d}} \overline{Y_{k',m'}(\bm \omega)} \right) \\
        &= \frac{\|\bm{x}\|^2}{d} \sum_{k,m} \sum_{k',m'} c_{k,m} \overline{c_{k',m'}} Y_{k,m}(\bm \omega) \overline{Y_{k',m'}(\bm \omega)}
    \end{align*}

    Decomposing $f$ into this form inside of the integral yields the separation
    \begin{align*}
        \int \|\bm x'\|^2 |f(\bm x')|^2\mathrm d \rho_{\mathcal X}(\bm x') &= \sum_{k,m} \sum_{k', m'} c_{k,m} c_{k',m'} \left(\frac{1}{d}\int_{\mathbb R}| r^2|^2 \mathrm \phi(r)dr\right) \left(\int_{\Omega}Y_{k,m}( \omega)\overline{Y_{k',m'}(\omega)}\mathrm d\Omega\right),\\
        &= \sum_{k,m}|c_{k,m}|^2\left(\frac{1}{d}\int_{\mathbb R}| r^2|^2 \mathrm \phi(r)dr\right)\\
        &=\left(\frac{1}{d}\int_{\mathbb R}| r^2|^2 \mathrm \phi(r)dr\right)
    \end{align*}
    where we denote by $\phi(r)$ the Gaussian density associated with the radial decomposition of the norm and we use Parseval's identity to simplify
    \[
        \sum_{k,m}|c_{k,m}|^2 = \|f\|_{L^2(\rho_{\mathcal X})} = 1.
    \]
    
    The second term of the product is simply the norm of the spherical harmonics, which evaluates to $1$ since it is a normalized function. The first term is the integral of a fourth Gaussian moment so we solve analytically 
    \[
        \frac{1}{d}\int r^4\phi(r)\mathrm dr = \frac{1}{d}\mathbb E[\|\bm x\|^4] = \frac{d(d+2)}{d} = d+2.
    \]
    Therefore, if we define
    \[
        E_1^{(1)}(\bm x) := \langle \bm x, \bm w^*\rangle \int_{\mathcal A^c} \frac{(\pi - \theta_{\bm x, \bm x'})}{\pi}\langle \bm x', \bm w^*\rangle f(\bm x')\mathrm d\rho_{\mathcal X}(\bm x').
    \]
    we have
    \[
        \left|E_1^{(1)}(\bm x)\right| = \mathcal O(de^{-d\epsilon^2})|\langle\bm x, \bm w^*\rangle|
    \]
    and we can see the $L^2$ norm respects
    \[
        \|E_1^{(1)}\|_{L^2(\rho_{\mathcal X})} = \mathcal O(de^{-d\epsilon^2}).
    \]
    
    Moreover, inside of $\mathcal{A}$ we write
    \begin{align*}
        &\langle\bm x, \bm w^*\rangle\int_{\mathcal{A}} \frac{(\pi - \theta_{\bm x, \bm x'})}{\pi}\langle \bm x', \bm w^*\rangle f(\bm x')\mathrm d\rho_{\mathcal X}(\bm x') \\
        & \quad=\frac{\langle\bm x, \bm w^*\rangle}{2} \left[ \int_{\mathcal{A}} \langle \bm x', \bm w^*\rangle f(\bm x')\mathrm d\rho_{\mathcal X}(\bm x') - \int_{\mathcal{A}}  \frac{\theta_{\bm x, \bm x'} - \pi/2}{\pi}\langle \bm x', \bm w^*\rangle f(\bm x')\mathrm d\rho_{\mathcal X}(\bm x')\right]\\
        &\quad := \frac{\langle\bm x, \bm w^*\rangle}{2}\int_{\mathcal{A}} \langle \bm x', \bm w^*\rangle f(\bm x')\mathrm d\rho_{\mathcal X}(\bm x') + E_2^{(1)}(\bm x).
    \end{align*}
    
    Since inside of $\mathcal{A}$ the difference between the angles is bounded by $\epsilon$, using Cauchy-Schwarz we can bound $E_2^{(1)}(\bm x)$ by
    \begin{align*}
        |E_2(\bm x)^{(1)}| &= \left|\langle\bm x, \bm w^*\rangle \int_{\mathcal{A}}  \frac{\theta_{\bm x, \bm x'} - \pi/2}{\pi}\langle \bm x', \bm w^*\rangle f(\bm x')\mathrm d\rho_{\mathcal X}(\bm x')\right| \leq \epsilon|\langle\bm x, \bm w^*\rangle| \int_{\mathcal{A}} |\langle \bm x', \bm w^*\rangle f(\bm x')|\mathrm d\rho_{\mathcal X}(\bm x') \\
        &\leq \epsilon |\langle\bm x, \bm w^*\rangle| \| \langle \cdot, \bm w^*\rangle \|_{L^2(\rho_{\mathcal X})}\| f\|_{L^2(\rho_{\mathcal X})} \leq \epsilon|\langle\bm x, \bm w^*\rangle|,        
    \end{align*}
    because $\langle \cdot, \bm w^*\rangle$ and $f$ are normalized in $L^2(\rho_{\mathcal X})$. Again, this implies the $L^2$ norm respects
    \[
        \|E_2^{(1)}\|_{L^2(\rho_{\mathcal X})} = \mathcal O(\epsilon).
    \]
    
    This gives the following expression for the first term
    \[
        \langle \bm x, \bm w^*\rangle\int \frac{(\pi - \theta_{\bm x, \bm x'})}{\pi}\langle \bm x', \bm w^*\rangle f(\bm x')\mathrm d\rho_{\mathcal X}(\bm x') = \frac{\langle \bm x, \bm w^*\rangle}{2}\int_{\mathcal{A}} \langle \bm x', \bm w^*\rangle f(\bm x')\mathrm d\rho_{\mathcal X}(\bm x') + E_1^{(1)}(\bm x) + E_2^{(1)}(\bm x)
    \]
    
    We can calculate the integral over $\mathcal{A}$ by writing
    \[
        \int_{\mathcal{A}} \langle \bm x', \bm w^*\rangle f(\bm x')\mathrm d\rho_{\mathcal X}(\bm x') = \int \langle \bm x', \bm w^*\rangle f(\bm x')\mathrm d\rho_{\mathcal X}(\bm x') - \int_{\mathcal{A}^c} \langle \bm x', \bm w^*\rangle f(\bm x')\mathrm d\rho_{\mathcal X}(\bm x')
    \]
    and by using the same argument as before, the integral over $\mathcal{A}^c$ induces a function $E_3^{(1)}$ and
    \begin{align*}
        \langle \bm x, \bm w^*\rangle\int_{\mathcal{A}} \langle \bm x', \bm w^*\rangle f(\bm x')\mathrm d\rho_{\mathcal X}(\bm x') &= \langle \bm x, \bm w^*\rangle\int \langle \bm x', \bm w^*\rangle f(\bm x')\mathrm d\rho_{\mathcal X}(\bm x') +  E_3^{(1)}(\bm x)\\
        &=\langle \bm x, \bm w^*\rangle\langle\langle \cdot, \bm w^*\rangle, f\rangle +  E_3^{(1)}(\bm x).
    \end{align*}
    such that $\|E_3^{(1)}\|_{L^2(\rho_{\mathcal X})} = \mathcal O(de^{-d\epsilon^2})$.
    
    Finally, the first term can be written as
    \begin{align*}
        T_{(1*)}f(\bm x) &= \langle \bm x, \bm w^*\rangle\int \frac{(\pi - \theta_{\bm x, \bm x'})}{\pi}\langle \bm x', \bm w^*\rangle f(\bm x')\mathrm d\rho_{\mathcal X}(\bm x') \\
        &= \frac{\langle \bm x, \bm w^*\rangle}{2}\langle\langle \cdot, \bm w^*\rangle,f\rangle + E_1^{(1)}(\bm x) + E_2^{(1)}(\bm x) + E_3^{(1)}(\bm x)\\
        &:= T_{S}^{(1*)}(\bm x)+ E_1^{(1)}(\bm x) + E_2^{(1)}(\bm x) + E_3^{(1)}(\bm x)
    \end{align*}
    and choosing $\epsilon = d^{-1/3}$ we have the bound on the $L^2$ difference by noticing
    \begin{align*}
        \|T_{(1*)}f - T_{S}^{(1*)}f\|_{L^2(\rho_{\mathcal X})} &\leq \left(\|E_1^{(1)}\|_{L^2(\rho_{\mathcal X})} + \|E_2^{(1)}\|_{L^2(\rho_{\mathcal X})} +\|E_3^{(1)}\|_{L^2(\rho_{\mathcal X})}\right) \\
        &\leq \left(\mathcal O(de^{-d^{1/3}})+\mathcal O(d^{-1/3}) + \mathcal O(de^{-d^{1/3}})\right)\\
        &= o_d(1).
    \end{align*} 

    {\bf Analysis of the second term:}
    Now we turn to the second term
    \[
        T_{(2*)}f(\bm x) = \int\frac{\sin \theta_{\bm x, \bm x'}}{2\pi}\left(\frac{\|\bm x'\|}{\|\bm x\|}\langle \bm x, \bm w^*\rangle^2 + \frac{\|\bm x\|}{\|\bm x'\|}\langle \bm x', \bm w^*\rangle^2\right)f(\bm x')\mathrm d\rho_{\mathcal X}(\bm x'),
    \]
    and apply the same treatment.

    Considering the same set $\mathcal A$, we split the integral over the entire space into $\mathcal A$ and $\mathcal A^c$, and first we show the following function is negligible
    \[
        E^{(2)}_1 (\bm x) := \frac{\langle \bm x, \bm w^*\rangle^2}{\|\bm x\|}\int_{\mathcal A^c} \frac{\sin \theta_{\bm x, \bm x'}}{2\pi}\|\bm x'\| f(\bm x')\mathrm d\rho_{\mathcal X}(\bm x') + \|\bm x\|\int_{\mathcal A^c} \frac{\sin \theta_{\bm x, \bm x'}}{2\pi}\frac{\langle \bm x', \bm w^*\rangle^2}{\|\bm x'\|} f(\bm x')\mathrm d\rho_{\mathcal X}(\bm x')
    \]

    Looking at the integral in the second term, again by Lemma  \ref{lemma:integral-concentration-1}, if we define $g = \frac{\langle \cdot, \bm w^*\rangle ^2}{\|\bm \cdot\|}f$ we have that
    \[
        \left|\int_{\mathcal A^c} \frac{\sin \theta_{\bm x, \bm x'}}{2\pi} {\frac{\langle \bm x', \bm w^*\rangle ^2}{\|\bm x'\|} f(\bm x') }\mathrm d \rho_{\mathcal X}(\bm x')\right| \leq \sqrt 2e^{-c_2d\epsilon^2/2}\|g\|_{L^2(\rho_{\mathcal X})},
    \]
    for an absolute constant $c_2 > 0$, as long as $g \in L^2(\rho_{\mathcal X}).$ Notice that, outside of a set with zero measure, we have
    \[
        \langle \bm x', \bm w^*\rangle = \|\bm x'\| \cos \theta_{\bm x', \bm w^*} \implies \frac{\langle \bm x', \bm w^*\rangle ^2}{\|\bm x'\|} = \cos \theta_{\bm x', \bm w^*} \langle \bm x', \bm w^*\rangle
    \]
    therefore calculating the $L^2(\rho_\mathcal X)$ norm of $g$ gives:
    \[
        \int \left|\frac{\langle \bm x', \bm w^*\rangle ^2}{\|\bm x'\|} f(\bm x')\right|^2 \mathrm d \rho_{\mathcal X}(\bm x')= \int \left| \cos \theta_{\bm x, \bm w^*} \langle \bm x', \bm w^*\rangle f(\bm x')\right|^2 \mathrm d \rho_{\mathcal X}(\bm x') \leq \int \left| \langle \bm x', \bm w^*\rangle f(\bm x')\right|^2 \mathrm d \rho_{\mathcal X}(\bm x').
    \]
    Recall that in the analysis of the first term we showed that $\|\cdot \|f \in L^2(\rho_{\mathcal X})$, and since this function dominates $g$, we must have $g \in L^2(\rho_{\mathcal X})$ and the concentration bound holds.
    
    Similarly, if we let $g' = \|\cdot \| f$, we already showed $g' \in L^2(\rho_{\mathcal X})$, and using the lemma again implies
    \[
        \left|\int_{\mathcal A^c} \frac{\sin \theta_{\bm x, \bm x'}}{2\pi} \|\bm x'\|f(\bm x')\mathrm d \rho_{\mathcal X}(\bm x')\right| \leq \sqrt 2e^{-c_2d\epsilon^2/2}\|g'\|_{L^2(\rho_{\mathcal X})}.
    \]
    
    Thus, outside of a set of measure zero, we have the following bound
    \[
        |E_1^{(2)}(\bm x)| \leq \mathcal O(de^{-d \epsilon^2})\left(\|\bm x\| + \frac{\langle \bm x, \bm w^*\rangle ^2}{\|\bm x\|}\right)
    \]
    which implies the $L^2(\rho_{\mathcal X})$ norm is bounded by
    \[
        \|E_1^{(2)}\|_{L^2(\rho_{\mathcal X})} \leq  \mathcal O(de^{-d \epsilon^2})\left(\Big\|\|\cdot\|\Big\|_{L^2} +\left\|\frac{\langle \cdot, \bm w^*\rangle ^2}{\|\cdot\|}\right\|_{L^2} \right) = \mathcal O(d^{3/2} e^{-d\epsilon^2}).
    \]
    
    Next, inside of $\mathcal{A}$,  we have
    \begin{align*}
        &\int_{\mathcal A} \frac{\sin \theta_{\bm x, \bm x'}}{2\pi}\left(\frac{\|\bm x'\|}{\|\bm x\|}\langle \bm x, \bm w^*\rangle^2 + \frac{\|\bm x\|}{\|\bm x'\|}\langle \bm x', \bm w^*\rangle^2\right)f(\bm x')\mathrm d\rho_{\mathcal X}(\bm x') \\
        &\quad := \frac{1}{2 \pi}\int_{\mathcal A}\left(\frac{\|\bm x'\|}{\|\bm x\|}\langle \bm x, \bm w^*\rangle^2 + \frac{\|\bm x\|}{\|\bm x'\|}\langle \bm x', \bm w^*\rangle^2\right)f(\bm x')\mathrm d\rho_{\mathcal X}(\bm x') + E_2^{(2)}(\bm x).
    \end{align*}
    where we define
    \[
        E_2^{(2)}(\bm x) = \int_{\mathcal A} \left[\frac{\sin \theta_{\bm x, \bm x'}-1}{2\pi}\right] \left(\frac{\|\bm x'\|}{\|\bm x\|}\langle \bm x, \bm w^*\rangle^2 + \frac{\|\bm x\|}{\|\bm x'\|}\langle \bm x', \bm w^*\rangle^2\right)f(\bm x')\mathrm d\rho_{\mathcal X}(\bm x').
    \]

    Note that
    \[
        |\sin \theta - 1| = 1- \sin \theta  = 1 - \cos(\theta - \pi/2),
    \]
    and if we expand $\cos x$ using Taylor's theorem around $0$, noting that $\cos'' x = -\cos x$, we obtain
    \[
        \cos x = 1  + \frac{\cos''(c)}{2}x^2 \implies 1 - \cos x = \frac{\cos(c)}{2}x^2
    \]
    for some $c \in [0, 2\pi]$. Since $\cos \theta \leq 1$ for all $\theta \in [0, 2\pi]$ we have the identity
    \[
        1 - \cos x \leq \frac{x^2}{2}.
    \]
    Thus, letting $x = \theta - \pi/2$ we get
    \[
        |\sin \theta - 1| \leq \frac{1}{2}\left(\theta - \frac{\pi}{2}\right)^2, \forall \theta \in [0, 2\pi],
    \]
    and, since the difference between the angles is bounded by $\epsilon$ inside of $\mathcal A$, using Cauchy-Schwarz again we can obtain the following bound for $E_2^{(2)}$:
    \begin{align*}
        |E_2^{(2)}(\bm x)| &\leq \frac{\epsilon^2}{4\pi}\sqrt{\int \left(\frac{\|\bm x'\|}{\|\bm x\|}\langle \bm x, \bm w^*\rangle^2 + \frac{\|\bm x\|}{\|\bm x'\|}\langle \bm x', \bm w^*\rangle^2\right)^2\mathrm d\rho_{\mathcal X}(\bm x')} \sqrt{\int|f(\bm x')|^2\mathrm d\rho_{\mathcal X}(\bm x')}\\
        &\leq \frac{\epsilon^2}{4 \pi}\left(\Big\|\|\cdot\|\Big\|_{L^2(\rho_{\mathcal X})}\frac{\langle \bm x, \bm w^*\rangle^2}{\|\bm x\|} + \|\bm x \|\left\|\frac{\langle \cdot, \bm w^*\rangle^2}{\|\cdot\|}\right\|_{L^2(\rho_{\mathcal X})}\right)
    \end{align*}
    almost everywhere.

    Taking the $L^2(\rho_{\mathcal X})$ norm gives
    \[
        \|E_2^{(2)}(\bm x)\|_{L^2(\rho_{\mathcal X})} \leq \frac{2\epsilon^2}{4\pi} \Big\|\|\cdot\|\Big\|_{L^2(\rho_{\mathcal X})}\left\|\frac{\langle \cdot, \bm w^*\rangle^2}{\|\cdot\|}\right\|_{L^2(\rho_{\mathcal X})} \leq \frac{2\epsilon^2 \sqrt{d}}{4\pi}.
    \]

    Lastly, we are left with
    \[
        T_{(2*)}f(\bm x) = \frac{1}{2\pi }\int_{\mathcal A} \left(\frac{\|\bm x'\|}{\|\bm x\|}\langle \bm x, \bm w^*\rangle^2 + \frac{\|\bm x\|}{\|\bm x'\|}\langle \bm x', \bm w^*\rangle^2\right)f(\bm x')\mathrm d\rho_{\mathcal X}(\bm x') + E_1^{(2)}(\bm x) + E_2^{(2)}(\bm x).
    \]
    
    Following the same argument as before we can write the integral over $\mathcal A$ as the integral over the whole space plus a term $E_3^{(2)}(\bm x)$ such that 
    \[
        |E_3^{(2)}(\bm x)| \leq \mathcal O(de^{-d \epsilon^2})\left(\|\bm x\| + \frac{\langle \bm x, \bm w^*\rangle ^2}{\|\bm x\|}\right)
    \]
    almost everywhere, and thus $\|E^{(3)}\|_{L^2(\rho_{\mathcal X})} =  \mathcal O(d^{3/2}e^{-d \epsilon^2})$.

    Therefore, we can write
    \begin{align*}
        T_{(2*)}f(\bm x) &= \frac{1}{2\pi}\int \left(\frac{\|\bm x'\|}{\|\bm x\|}\langle \bm x, \bm w^*\rangle^2 + \frac{\|\bm x\|}{\|\bm x'\|}\langle \bm x', \bm w^*\rangle^2\right)f(\bm x')\mathrm d\rho_{\mathcal X}(\bm x') + E_1^{(2)}(\bm x) + E_2^{(2)}(\bm x) + E_3^{(2)}(\bm x)\\
        & = \frac{1}{2\pi}\frac{\langle \bm x, \bm w^*\rangle^2}{\|\bm x\|} \langle \|\cdot\|, f\rangle + \frac{1}{2\pi}\|\bm x \| \left\langle \frac{\langle \cdot, \bm w^*\rangle^2}{\|\cdot\|}, f \right\rangle + E_1^{(2)}(\bm x) + E_2^{(2)}(\bm x) + E_3^{(2)}(\bm x)\\
        &:= T_{S}^{(2*)}(\bm x) + E_1^{(2)}(\bm x) + E_2^{(2)}(\bm x) + E_3^{(2)}(\bm x).
    \end{align*}

    Finally, since we chose $\epsilon = d^{-1/3}$, we have
    \begin{align*}
        \|T_{(2*)}f - T_{S}^{(2*)}f\|_{L^2(\rho_{\mathcal X})} &\leq \left(\|E_1^{(2)}\|_{L^2(\rho_{\mathcal X})} + \|E_2^{(2)}\|_{L^2(\rho_{\mathcal X})} + \|E_3^{(2)}\|_{L^2(\rho_{\mathcal X})}\right) \\
        &\leq \left(\mathcal O(d^{3/2}e^{-d^{1/3}})+ \mathcal O(d^{-1/6}) + \mathcal O(d^{3/2}e^{-d^{1/3}})\right)\\
        &= o_d(1).
    \end{align*}
\end{proof}

\subsection{Proof of Theorem \ref{thm:lin-eigenfn}}\label{app:thm-lin-eigenfn}

\begin{proof}

    We show the linear eigenfunction $\psi = \langle \cdot ,\bm v\rangle$, $\bm v \in \mathbb R^d$, continues to be an eigenfunction of $T_1$ if $\bm v = \bm w^*$ or $\bm v \perp \bm w^*$. 
    
    We use Fubini's Theorem to write
    \[
        T_1 \psi(\bm x) = \int k_1(\bm x, \bm x')\psi(\bm x')\mathrm d \rho_{\mathcal X}(\bm x') = \mathbb E_{\bm w \sim \mathcal N(0, \frac{1}{d} \bm \Gamma)}[\sigma(\langle \bm w, \bm x\rangle) \mathbb E_{\bm x' \sim \mathcal N(0, \bm I_d)}[\sigma(\langle \bm w, \bm x'\rangle)\psi(\bm x')]]
    \]
    and analytically solve these integrals. Looking at the inner expectation, and applying Stein's Lemma, we have
    \[
        \mathbb E_{\bm x' \sim \mathcal N(0, \bm I_d)}[\sigma(\langle \bm w, \bm x'\rangle)\psi(\bm x')] = \mathbb E_{\bm x' \sim \mathcal N(0, \bm I_d)}[\sigma(\langle \bm w, \bm x'\rangle)\langle \bm x', \bm v\rangle] = \mathbb E_{\bm x' \sim \mathcal N(0, \bm I_d)}[\langle D_{\bm x'}\left\{ \sigma(\langle \bm w, \bm x'\rangle) \right\}, \bm v\rangle)].
    \]

    Now, since $\sigma'(t) = \bm 1_{\{t \geq 0\}}$ almost everywhere, we have
    \[
        \mathbb E_{\bm x' \sim \mathcal N(0, \bm I_d)}[D_{\bm x'} \left\{ \sigma(\langle \bm w, \bm x'\rangle) \right\}] = \mathbb E_{\bm x' \sim \mathcal N(0, \bm I_d)}[\langle \bm w, \bm v\rangle \bm 1_{\{\langle \bm w, \bm x'\rangle \geq 0\}}] = \frac{1}{2}\langle \bm w, \bm v\rangle.
    \]

    Going back to the outer integral
    \[
        T_1\psi(\bm x) = \frac{1}{2}\mathbb E_{\bm w \sim \mathcal N(0, \frac{1}{d}\bm \Gamma)}[\sigma(\langle \bm w, \bm x\rangle) \langle \bm w, \bm v\rangle],
    \]
    and using Stein's Lemma and the same argument again
    \[
        \mathbb E_{\bm w \sim \mathcal N(0, \frac{1}{d}\bm \Gamma)}[\sigma(\langle \bm w, \bm x\rangle) \langle \bm w, \bm v\rangle] =  \frac{1}{d}\mathbb E_{\bm w \sim \mathcal N(0, \frac{1}{d}\bm \Gamma)}[\langle\bm \Gamma \nabla_{\bm w}\sigma(\langle \bm w, \bm x\rangle), \bm v\rangle] = \frac{1}{2{d}}\langle \bm \Gamma \bm x, \bm v\rangle, 
    \]
    therefore
    \[
        T_1\psi(\bm x) = \frac{1}{4d}\langle \bm \Gamma \bm x, \bm v\rangle.
    \]

    Note that the same calculation shows that
    \[
        T_0 \psi(\bm x) = \frac{1}{4d}\langle \bm x, \bm v\rangle.
    \]
    
    As a consequence, if we let $\bm v = \bm w^*$, we have
    \[
        \frac{1}{4d}\langle \bm \Gamma \bm x, \bm v\rangle = \frac{1}{4d}\langle \bm \Gamma \bm x, \bm w^*\rangle = \left(\frac{A}{4d} + \frac{B}{4d}\right)\langle\bm x, \bm w^*\rangle,
    \]
        
    And, on the other hand, if $\bm v \perp \bm w^*$
    \[
        \frac{1}{4d}\langle \bm \Gamma \bm x, \bm v\rangle = \frac{A}{4d}\langle\bm x, \bm v\rangle.        
    \]

\end{proof}

\subsection{Proof of Theorem \ref{thm:approx-top-eigenfn}}\label{app:thm-approx-top-eigenfn}
\begin{proof}

    We have from Lemma \ref{lemma:approx-first-order-op} that, for a given $f$ that can be represented by the eigenfunctions of the kernel, the action of operator $S$ can be approximated by
    \[
        T_Sf = T_S^{{(1*)}}f + T_S^{{(2*)}}f + E.
    \]
    If we recall the closed form of the first term
    \[
        T_S^{{(1*)}}f(\bm x) = \frac{\langle \bm x, \bm w^*\rangle}{2}\langle\langle \cdot, \bm w^*\rangle,f\rangle, 
    \]
    we note the inner product term implies that any function that is orthogonal to the linear function must not interact with this term. 
    
    Also note that Theorem \ref{thm:lin-eigenfn} shows the linear function is a true eigenfunction of $T_1$, thus if a different function happens to be the top eigenfunction of the operator, then they must be orthogonal to each other. 
    
    Therefore, to understand the action
    \[
        T_1 f(\bm x) = AT_0f(\bm x) + \frac{B}{2d}T_S f(\bm x) + T_Rf(\bm x)
    \]
    on any function other than the linear function it suffices to solve the eigenvalue problem for $AT_0 + \frac{B}{2d}T_{S}^{(2*)}$.

    {\bf Solving the eigenvalue problem for $AT_0 + \frac{B}{2d}T_{S}^{(2*)}$:}
    Recall that
    \[
        T_{S}^{(2*)}f (\bm x) = \frac{1}{2\pi}\left( \|\bm x \| \left\langle \frac{\langle \cdot, \bm w^*\rangle^2}{\|\cdot\|}, f \right\rangle + \frac{\langle \bm x, \bm w^*\rangle^2}{\|\bm x\|} \langle \|\cdot\|, f\rangle \right)
    \]
    and if we define $f_1(\bm x) := \|\bm x\|$ and $f_2(\bm x) := \frac{\langle \bm x, \bm w^*\rangle^2}{\|\bm x\|}$ we have
    \[
        T_{S}^{(2*)}f = \frac{1}{2\pi}\Big( f_1\langle f_2, f\rangle + f_2 \langle f_1, f\rangle\Big)\, ,
    \]
    therefore the eigenfunctions must lie in $\text{span}\{f_1, f_2\}$. To bridge the gap between both operators, we recall the eigenbasis of the isotropic operator given by Lemma \ref{lemma:iso-eigenfn}, and we translate our functions $f_1$ and $f_2$ to that basis.
        
    First, we let $r = \|\bm x\|$ and $\bm \omega = \frac{\bm x}{\|\bm x\|}$ and consider the following elements of the spherical harmonics basis
    \[
        Y_0(\bm \omega) = 1, \quad \hat{Y}_2(\bm \omega, \bm w^*) = \langle \bm \omega, \bm w^*\rangle^2 - \frac{1}{d}.
    \]
    Note the notation $\hat{Y}_2$ to indicate this function is not normalized in $L^2(\rho_{\mathcal X})$.
    With these we can write
    \[
        f_1(\bm x) = \|\bm x\| = r Y_0(\bm \omega), \quad f_2(\bm x) = \frac{\langle \bm x, \bm w^*\rangle^2}{\|\bm x\|} = r\langle \bm \omega, \bm w^*\rangle^2,
    \]
    and writing the inner product as a combination of both functions yields
    \[
        \langle \bm \omega, \bm w^*\rangle^2 = \hat{Y}_2(\bm \omega, \bm w^*) + \frac{1}{d}Y_0(\bm \omega),
    \]
    \[
        f_1(\bm x) = r Y_0(\bm \omega), \quad f_2(\bm x) = r \hat{Y}_2(\bm \omega, \bm w^*) + \frac{r}{d} Y_0(\bm \omega).
    \]
    
    Next, we solve the eigenvalue problem of the combined operator $AT_0 +\frac{B}{2d} T_{S}^{(2*)}$. Let $\varphi$ be a non trivial eigenfunction of this operator. Since $T_0$ is fully diagonalized by $r Y_0$ and $r \hat{Y}_2$ and since the action of $T_{S}^{(2*)}$ is restricted to $\text{span}\{r Y_0, r \hat{Y}_2\}$, we know $\varphi$ must take the form
    \[
        \varphi = c_0 [r Y_0(\bm \omega)] + c_2[r \hat{Y}_2(\bm \omega, \bm w^*)]
    \]
    for some coefficients $c_0$ and $c_2$.
    
    If we apply
    \[
        T_{S}^{(2*)}\varphi = \frac{\langle f_2, \varphi \rangle}{2 \pi}f_1 + \frac{ \langle f_1, \varphi \rangle}{2 \pi}f_2,
    \]
    using the identities $f_1 = r Y_0(\omega)$ and $f_2 = r \hat{Y}_2(\omega, \bm w^*) + \frac{r}{d}Y_0(\omega)$, we obtain
    \begin{align*}
        T_{S}^{(2*)}\varphi &= \left\langle r \hat{Y}_2 + \frac{r}{d} Y_0, c_0[r Y_0] + c_2[r \hat{Y}_2]\right\rangle \frac{r Y_0}{2 \pi}  + \left\langle r Y_0,  c_0[r Y_0] + c_2[r \hat{Y}_2]\right\rangle\frac{1}{2 \pi}\left( r \hat{Y}_2 + \frac{r}{d} Y_0 \right)\\
        &= \frac{r Y_0}{2 \pi} \left( \frac{c_0}{d} \langle r Y_0, r Y_0\rangle + c_2 \langle r \hat{Y}_2, r \hat{Y}_2\rangle\right) + \frac{1}{ 2 \pi}\left( r \hat{Y}_2 + \frac{r}{d} Y_0 \right)c_0\langle r Y_0, r Y_0\rangle\\
        &= \frac{r Y_0}{2 \pi}\left( \frac{2c_0}{d} \langle r Y_0, r Y_0\rangle + c_2 \langle r \hat{Y}_2, r \hat{Y}_2\rangle\right) + \frac{r \hat{Y}_2}{2 \pi} \Big(c_0\langle r Y_0, r Y_0\rangle \Big).
    \end{align*}

    This translates to the matrix operator on the vector of coefficients
    \[
        T_{S}^{(2*)}\varphi = \frac{1}{2\pi}\begin{bmatrix}
            \frac{2}{d}\langle r Y_0, r Y_0\rangle & \langle r \hat{Y}_2, r \hat{Y}_2\rangle\\
            \langle r Y_0, r Y_0\rangle & 0
        \end{bmatrix} \begin{pmatrix}
            c_0\\
            c_2
        \end{pmatrix} \cdot \begin{pmatrix}
            r Y_0\\
            r \hat{Y}_2
        \end{pmatrix}
    \]
    and since $\langle r Y_0, r Y_0\rangle = d$ and $\langle r \hat{Y}_2,  r \hat{Y}_2\rangle = \frac{2(d-1)}{d(d+2)}$ we have
    \[
        T_{S}^{(2*)}\varphi = \frac{1}{2\pi}\begin{bmatrix}
            2 & \frac{2(d-1)}{d(d+2)}\\
            d & 0
        \end{bmatrix} \begin{pmatrix}
            c_0\\
            c_2
        \end{pmatrix} \cdot \begin{pmatrix}
            r Y_0\\
            r \hat{Y}_2
        \end{pmatrix} .
    \]

    Furthermore, because these are non normalized orthogonal eigenfunctions of $T_0$, they perfectly diagonalize the operator. Thus, if $\lambda_{\max}(T_0)$ and $\lambda_2(T_0)$ are the eigenvalues for $r Y_0$ and $rY_2$ respectively, we have the following
    \[
        AT_0 \varphi = \begin{bmatrix}
            A\lambda_{\max}(T_0) & 0\\
            0 & A\lambda_2(T_0)
        \end{bmatrix} \begin{pmatrix}
            c_0\\
            c_2
        \end{pmatrix} \cdot \begin{pmatrix}
            r Y_0\\
            r \hat{Y}_2
        \end{pmatrix},
    \]
    and the combined action is simply the combination of these matrices
    \[
        \left( AT_0 + \frac{B}{2d}T_{S}^{(2*)} \right) \varphi = \begin{bmatrix}
            A\lambda_{\max}(T_0) + \frac{B}{2\pi d} & \frac{B(d-1)}{2 \pi d^2(d+2)}\\
            \frac{B}{4 \pi } & A\lambda_2(T_0)
        \end{bmatrix} \begin{pmatrix}
            c_0\\
            c_2
        \end{pmatrix} \cdot \begin{pmatrix}
            r Y_0\\
            r \hat{Y}_2
        \end{pmatrix} .
    \]
    Note that $\lambda_2$ represents the eigenvalue corresponding to a single spherical harmonic of degree 2. We know the dimension of the degree-2 subspace on $\mathbb {S}^{d-1}$ is given by the degeneracy formula $N(d, 2) = \frac{(d-1)(d+2)}{2}$, thus since the total macroscopic energy is distributed uniformly across the orthogonal basis in the subspace, the individual eigenvalue scales as $\lambda_2 = \mathcal{O}(d^{-2})$. 

    We draw attention to the $\mathcal O(B)$ coefficient on the lower left entry of the matrix which implies a strong coupling between these functions under the combined operator.
    
    Again, solving for the roots of the characteristic polynomial, we find the following two eigenvalues
    \[
        \tilde{\lambda}_{\pm} = \frac{1}{2} \left[ \left( A \lambda_{\max}(T_0) + \frac{B}{2 \pi d} + A \lambda_2(T_0) \right) \pm \sqrt{ \left( A \lambda_{\max}(T_0) + \frac{B}{2 \pi d} - A \lambda_2(T_0) \right)^2 + \frac{B^2(d-1)}{2 \pi^2 d^2(d+2)} } \right],
    \]
    or, more intuitively, using the Taylor expansion of $\sqrt{x + \epsilon}$ we can write
    \[
        \tilde{\lambda}_+ = A\lambda_{\max}(T_0) + \frac{B}{2 \pi d} + o_d\left(\frac{|B|}{d}\right), \quad \tilde{\lambda}_- = \mathcal O\left( A\lambda_2(T_0) + \frac{B^2}{d^3}\right).
    \]
    
    Furthermore, we know that any eigenfunction with this eigenvalue must satisfy
    \[
         \begin{bmatrix}
            A\lambda_{\max}(T_0) + \frac{B}{2\pi d} & \frac{B(d-1)}{2 \pi d^2(d+2)}\\
            \frac{B}{4 \pi } & A\lambda_2(T_0)
        \end{bmatrix}\begin{pmatrix}
            c_0\\
            c_2
        \end{pmatrix} = \lambda_{\pm}\begin{pmatrix}
            c_0\\
            c_2
        \end{pmatrix}
    \]
    which implies the relationship
    \[
        \frac{B}{4\pi}c_0 = (\lambda_{\pm} - A \lambda_2(T_0))c_2.
    \]
    
    Since the eigenfunctions can be arbitrarily scaled, we choose $c_0 = \frac{1}{\sqrt d}$ to reconstruct the isotropic eigenfunction $c_0 \|\bm x\| Y_0 = \frac{\|\bm x\|}{\sqrt d} Y_0$, and solving for $c_2$ gives
    \[
        c_2 = \frac{B}{4 \pi \sqrt d(\tilde{\lambda}_{\pm}-A\lambda_2(T_0))}
    \]
    which lets us define the new eigenfunctions as
    \[
        \hat{\Psi}_{\pm}(\bm x) = \frac{\|\bm x\|}{\sqrt d} Y_0(\bm \omega)+ \left[ \frac{B}{4 \pi( \tilde{\lambda}_{\pm} - A\lambda_2(T_0) )} \right] \frac{\|\bm x\|}{\sqrt d}\hat{Y}_2(\bm \omega, \bm w^*) \, .
    \]

    To ensure this function is normalized, we calculate
    \[
        \|\hat{\Psi}_{\pm}\|^2_{L^2(\rho_{\mathcal X})} = 1 + \frac{B^2}{16 \pi ^2( \tilde{\lambda}_{\pm}-A\lambda_2(T_0))^2} \frac{2d -2}{d^2(d+2)}:= N_{\pm}
    \]
    and therefore the normalized approximate eigenfunctions are given by
    \[
        \tilde{\Psi}_{\pm}(\bm x) = \frac{1}{\sqrt{N_{\pm}}}\left( \frac{\|\bm x\|}{\sqrt d} Y_0(\bm \omega)+ \left[ \frac{B}{4 \pi ( \tilde{\lambda}_{\pm}-A\lambda_2(T_0) )} \right] \frac{\|\bm x\|}{\sqrt d} \hat{Y}_2(\bm \omega, \bm w^*)\right)\, .
    \]

    Lastly, to explicitly show the magnitude of the quadratic feature, if $Y_2$ is the normalized zonal harmonics, we write $\hat{Y}_2 = \|\hat{Y}_2\|_{L^2(\rho_{\mathcal X})} Y_2$ with $\|\hat{Y}_2\|_{L^2(\rho_{\mathcal X})} = \frac{1}{d}\sqrt{\frac{2d - 2}{d+2}}$. Therefore, we have the final form for the eigenfunctions:
    \[
        \tilde{\Psi}_{\pm}(\bm x) = \frac{1}{\sqrt{N_{\pm}}}\frac{\|\bm x\|}{\sqrt d} \Big[ Y_0(\bm \omega)+  \tau_{\pm} Y_2(\bm \omega, \bm w^*) \Big]\,,
    \]
    where we define the alignment magnitudes by the quantity
    \[
        \tau_{\pm} := \frac{1}{d}\sqrt{\frac{2d - 2}{d+2}}\left[\frac{B}{4 \pi (\tilde{\lambda}_{\pm}-A\lambda_2(T_0))}\right].
    \]

    From now on we study the approximate eigenfunction $\tilde{\Psi}_{+}$ with associated approximate eigenvalue $\lambda_{+}$, since we clearly have $\tilde{\lambda}_{+} > \tilde{\lambda}_{-}$.
    
    {\bf The approximate eigenvalue $\tilde{\lambda}_{+}$ dominates every other eigenvalue of $T_1$:}
    From the previous derivations, if we look at the approximate eigenvalue as functions of the dimension $\lambda_{+}(d)$, we established that 
    \[
        \lim_{d \to \infty} \tilde{\lambda}_{+}(d) = A\lambda_{\max}(T_0) > 0
    \]
    Thus, for any chosen $\epsilon > 0$, there exists an integer $d_0$ such that for all $d > d_0$
    \[
        \tilde{\lambda}_{+}(d) > A\lambda_{\max}(T_0) - \epsilon
    \]
    Choose $\epsilon = \frac{A\lambda_{\max}(T_0)}{2}$ such that there exists a $d_0$ such that for all $d > d_0$
    \[
        \tilde{\lambda}_{+}(d) > \frac{A\lambda_{\max}(T_0)}{2}.
    \]
    
    By Theorem \ref{thm:lin-eigenfn}, we know $\lambda_1(T_1) = A\lambda_1(T_0) + \frac{B}{4d} = \mathcal{O}(d^{-1})$. Thus, there exists a real constant $M > 0$ and an integer $d_1$ such that for all $d > d_1$
    \[
        \lambda_1(T_1) \le \frac{M}{d}.
    \]
    Because the ReLU kernel has monotonically decreasing eigenvalues, the linear eigenvalue serves as a strict upper bound for all higher-degree subspaces
    \[
        \lambda_1(T_0) > \lambda_k(T_0) \quad \text{for all } k \ge 2.
    \]
    Furthermore, by Theorem \ref{thm:eigenv-decay} there exists a radius $R > 0$ and $\varepsilon > 0$ such that for constant $C_R > 0$ independent of the dimension we have
    \[
        \lambda_k(T_1) \leq C_R\lambda_k(T_0) + \varepsilon < 2C_R\lambda_1(T_0)
    \]
    for all $k \geq 1$ which implies
    \[
        \lambda_k (T_1) \leq \frac{2C_RM}{d}\, .
    \]
    
    Thus, if we can prove $\tilde{\lambda}_{+}(d) > \lambda_1(T_1)$, we automatically prove it dominates all $\lambda_k(T_1)$ for $k \ge 1$.
    
    We want to guarantee that $\tilde{\lambda}_{+}(d) > \lambda_1(T_1)$ which is true as long as
    \[
        \frac{A\lambda_{\max}(T_0)}{2} > \frac{2C_RM}{d} \implies d > \frac{4C_RM}{A\lambda_{\max}(T_0)}.
    \]
    
    Define 
    \[
        d^* = \max\left(d_0, d_1, \left\lfloor \frac{4C_RM}{A\lambda_{\max}(T_0)} \right\rfloor + 1 \right)
    \]
    then for any dimension $d > d^*$, the following holds
    \[
        \tilde{\lambda}_{+}(d) > \frac{A\lambda_{\max}(T_0)}{2} > \frac{2C_RM}{d} \ge \lambda_1(T_1) \geq  \lambda_{k \ge 2}(T_1).
    \]

    {\bf The approximate eigenfunction and the true eigenfunction are close in norm:}
    Finally, to show $\tilde{\Psi}_{+}$ is close to the original top eigenfunction, we denote by $\Psi$ the top eigenfunction and use the expansion to write
    \[
        T_1 \tilde{\Psi}_{+} = \left(AT_0 + \frac{B}{2d}T_{S_{(1*)}} + \frac{B}{2d}T_{S}^{(2*)}\right)\tilde{\Psi} + e = \tilde{\lambda}_{+}\tilde\Psi_{+} + e
    \]
    with $\| e \| = o_d\left(\frac{|B|}{d}\right)$ and we write this as
    \[
        (T_1 - \tilde{\lambda}_{+} I)\tilde{\Psi}_{+} = e\, .
    \]
    Now, consider the projection onto the top eigenspace of $T_1$ defined by
    \[
        P_{\text{top}}f =  \langle \Psi, f\rangle \Psi,
    \]
    and its orthogonal complement $P_{\perp} = I - P_{\text{top}}$. Expanding the action of the combined operator using these projections we get
    \[
        (T_1 - \tilde{\lambda}_{+} I)\tilde{\Psi}_{+} = (T_1 - \tilde{\lambda}_{+} I)(P_{\text{top}}\tilde{\Psi}_{+} + P_{\perp}\tilde{\Psi}_{+}) = e
    \]
    and since $T_1(P_{\text{top}} f) = \lambda_{\max}(T_1) P_{\text{top}} f$ this simplifies to
    \[
        [\lambda_{\max}(T_1) - \tilde{\lambda}_{+}]P_{\text{top}} \tilde{\Psi}_{+} + (T_1 - \tilde{\lambda}_{+} I)P_{\perp}\tilde{\Psi}_{+} = e.
    \]
    Since we have only orthogonal objects, the norm of this expression is equal to
    \[
        |\lambda_{\max}(T_1)- \lambda_{+}|^2\| P_{\text{top}} \tilde{\Psi}_{+} \|_{L^2(\rho_{\mathcal X})} ^2 + \|(T_1 - \tilde{\lambda}_{+} I)P_{\perp}\tilde{\Psi}_{+}\|_{L^2(\rho_{\mathcal X})} ^2 = o_d\left(\frac{B^2}{d^2}\right)
    \]
    which immediately implies
    \[
        \|(T_1 - \tilde{\lambda}_{+} I)P_{\perp}\tilde{\Psi}_{+}\|_{L^2(\rho_{\mathcal X})} ^2 = o_d\left(\frac{B^2}{d^2}\right)\, ,
    \]
    and we will use this to bound $\| P_{\perp} \tilde{\Psi}\|$. If we define the approximate spectral gap by $\tilde{\delta} = \inf_{\mu \in \sigma(T_1)\setminus\{\lambda_{\max}(T_1)\}}|\mu - \tilde{\lambda}_{+}|$, we have
    \[
        \| (T_1 - \tilde{\lambda}_{+} I)P_{\perp}\tilde{\Psi}_{+} \|_{L^2(\rho_{\mathcal X})}  \geq  \tilde{\delta}\| P_{\perp}\tilde{\Psi}_{+} \|_{L^2(\rho_{\mathcal X})} 
    \]
    and
    \[
        \tilde{\delta} \| P_{\perp}\tilde{\Psi}_{+} \|_{L^2(\rho_{\mathcal X})}  = o_d\left(\frac{|B|}{d}\right).
    \]

    From Theorem \ref{thm:eigenv-decay}, we have the lower bound
    \[
        c\lambda_{\max}(T_0) \leq \lambda_{\max}(T_1)
    \]
    where $c > 0$ is an absolute constant. We know $\lambda_{\max}(T_0) = \Theta(1)$ thus $c\lambda_{\max}(T_0)$ is bounded away from zero. which gives the lower bound
    \[
        \lambda_{\max}(T_1) \geq c\lambda_{\max}(T_0) = \Omega(1)\, .
    \]

    Now if $\delta := \lambda_{\max}(T_1) - \lambda_1(T_1)$ is the true spectral gap, we know from Theorem \ref{thm:lin-eigenfn} that $\lambda_1(T_1) = \Theta(1/d)$. Substituting our lower bound for $\lambda_{\max}(T_1)$ into the gap we have
    \[
        \delta \geq c\lambda_{\max}(T_0) - \lambda_1(T_1) = \Omega(1)\, ,
    \]
    for a high enough dimension $d$. The difference between the gaps is given by
    \[
        |\lambda_{\max}(T_1) - \tilde{\lambda}_{+}| = |\delta - \tilde{\delta}|
    \]
    and since $|\lambda_{\max}(T_1) - \tilde{\lambda}_{+}| = o_d(|B|/d)$ the triangle inequality implies
    \[
        \tilde{\delta} \geq \delta - o_d\left(\frac{|B|}{d}\right) \, ,
    \]
    and we know $\tilde{\delta}$ must bounded away from zero independently of $d$. Thus, going back to bounding the norm of $P_\perp \tilde{\Psi}_{+}$, we obtain
    \[
         \|P_{\perp}\tilde{\Psi}_{+} \|_{L^2(\rho_{\mathcal X})}   = \frac{1}{\tilde{\delta}} o_d\left(\frac{|B|}{d}\right) = o_d\left(\frac{|B|}{d}\right)\, .
    \]
    
    To bound the difference under the norm, we note that $\|\tilde{\Psi}\| = 1$ and write
    \[
        \| \tilde{\Psi}_{+} \|_{L^2(\rho_{\mathcal X})} ^2 = |\langle \tilde{\Psi}_{+}, \Psi\rangle|^2 + \|P_{\perp} \tilde{\Psi}_{+}\|_{L^2(\rho_{\mathcal X})} ^2
    \]
    \[
        |\langle \tilde{\Psi}_{+}, \Psi\rangle|^2 = 1 - o_d\left(\frac{B^2}{d^2}\right)
    \]
    Finally, combining all together we have
    \begin{align*}
        \| \tilde{\Psi}_{+} - \Psi\|_{L^2(\rho_{\mathcal X})} ^2 &= \| \tilde{\Psi}_{+} \|_{L^2(\rho_{\mathcal X})} ^2 -2\langle \tilde{\Psi}_{+}, \Psi\rangle + \|\Psi\|_{L^2(\rho_{\mathcal X})} ^2\\
        &= 2 - 2\left[1- o_d\left(\frac{B^2}{d^2}\right)\right]\\
        &=  o_d\left(\frac{B^2}{d^2}\right)
    \end{align*}
    and the proof is completed.
\end{proof}

%% file: appendixes/appendix-experiments.tex
\label{app:experimental_details}

In this section, we provide comprehensive details regarding the experimental setup, hyperparameter configurations, and the mathematical framework used for the results presented in the main text.

For all experiments, we always consider the matrix $\bm \Gamma = A \bm I_d + B \bm w^* (\bm w^*)^\top$, where we set the artificial scaling $A = 1.2$ and choose multiple values for $B$ to observe of its influence. Denoting by $N$ the number of training samples, we sample the input data matrix $\bm {Z} \in \mathbb{R}^{N \times d}$, where each row $\bm {z}_i$ is drawn independently from a standard multivariate Gaussian distribution $\bm {z}_i \sim \mathcal{N}(0, \bm {I}_d)$.

\subsection{Different models}
Throughout the experiments we consider three different models detailed as follows:
\begin{enumerate}[leftmargin=*,topsep=0.5mm, itemsep=0.mm]
    \item \textbf{Base ReLU Kernel ($k_0$):} The standard ReLU activation kernel with closed form given by the arc-cosine kernel
    \[
        k_0(\bm x, \bm x') = \frac{\|\bm x\|\|\bm x'\|}{2\pi d}\left[\gamma(\pi -\arccos (\gamma)) + \sqrt{1 - \gamma^2}\right]\, ,
    \]
    with $\gamma = \frac{\langle \bm x, \bm x'\rangle}{\|\bm x\| \|\bm x'\|}$.
    \item \textbf{ReLU Kernel ($k_1$):} The kernel from \cref{eq:relu-k1} \[
        k_1(\bm x, \bm x') = \frac{\sqrt{(\bm x^\top \bm \Gamma \bm x)(\bm x'^\top \bm \Gamma \bm x')}}{2\pi d}\left[\gamma_{\bm \Gamma}(\pi -\arccos (\gamma_{\bm \Gamma})) +\sqrt{1 - \gamma_{\bm \Gamma}^2}\right]\,,
    \]
    with $\gamma_{\bm \Gamma} \coloneqq \frac{\bm x^\top \bm \Gamma \bm x'}{\sqrt{(\bm x^\top \bm \Gamma \bm x)(\bm x'^\top \bm \Gamma \bm x')}}$.
    \item \textbf{ReLU MLP:} A two-layer Multi-Layer Perceptron (MLP) with hidden layer width of $400$ and linear output layer. 
\end{enumerate}
\paragraph{Kernels:}
For $i \in \{0, 1\}$, we calculate the kernel matrix $\bm K_i(\bm Z) \in \mathbb R^{N\times N}$ whose entries correspond to
\[ 
    [\bm K_i(\bm Z)]_{k,l} = k_i(\bm z_k, \bm z_l)\, .
\]

\paragraph{ReLU NN:}
The ReLU neural network was trained using the Adam optimizer (\texttt{torch.optim.Adam}) with the following parameters:
\begin{itemize}[leftmargin=*,topsep=0.5mm, itemsep=0.mm]
    \item \textbf{First layer dimension:} $400$
    \item \textbf{Output layer dimension:} $1$
    \item \textbf{Epochs:} 1
    \item \textbf{Batch size:} 64
    \item \textbf{Loss Function:} Mean Squared Error (\texttt{torch.nn.MSELoss})
\end{itemize}

\subsection{Metrics and Significance}
All experiments were executed over 10 independent trials. Shaded regions in figures indicate the variance across these trials. For every experiment routine implemented, we set random seeds to ensure reproducibility across all sources of randomness (e.g. data sampling, NN initialization, K-Fold validation, etc.).

\subsection{Alignment experiment}
For the results in \cref{fig-f} we set the seed $54643$. For the alignment analysis, we vary the dimension $d \in \{50, 100, 200, 400, 800, 1600, 3200\}$, and we also vary the magnitude of $B \in \{5d^{3/10}, 5d^{5/10}, 5d^{7/10}, 5d^{9/10}\}$ obtaining four different kernels. The alignment reported is calculated as $\langle \bm {v}_i^{\text{top}}, Y_2(\bm {Z}) \rangle$, where $\bm {v}_i^{\text{top}}$ is the lead eigenvector of the kernel matrix $\bm {K}_i(\bm {Z})$ for $i \in \{0, 1\}$ and $Y_2$ is the normalized version of the function
\[
    \hat{Y}_2(\bm \omega, \bm w^*) = \langle \bm \omega, \bm w^*\rangle^2 - \frac{1}{d}\, ,
\]
where $\bm \omega = \frac{\bm z}{\|\bm z\|}$.

\subsection{Generalization performance}
For the results in \cref{fig-t} we set the seed $558812$. For the learning performance analysis, we work on the fixed dimension $d=300$ and vary the number of training samples $N \in \{50, 100, 200, 400, 800, 1600, 3200\}$ and again we build three different kernels $A = 1.2$ with $B \in \{5d^{3/10}, 5d^{5/10}, 5d^{7/10}, 5d^{9/10}\}$. All models are trying to learn the target function $g(t) = 2t^2 + 3t + 4\sin(2t)$.

We use Kernel Ridge Regression (KRR) with $k_0$ and $k_1$ to obtain the solutions $\hat{\bm a}_i = (\bm K_i + \lambda \bm I_N)^{-1} \bm y$, with regularization parameter $\lambda$. The regularization is chosen through $5$-fold cross-validation, for each sample size $N$, from a logarithmic grid $\{10^{-3}, 10^{-2}, \dots, 10^3\}$  with random state $38182$.

We sample a test set $\tilde{\bm Z}$, from the same distribution as $\bm Z$, consisting of $M = 600$ independent samples. Given the test samples, for $i \in \{0, 1\}$, we construct the test kernel matrices $\bm K_i(\tilde{\bm Z}, \bm Z) \in \mathbb R^{M\times N}$ whose entries are given by
\[ 
    [\bm K_i(\tilde{\bm Z}, \bm Z)]_{k,l} = k_i(\tilde{\bm z}_k, \bm z_l)\, ,
\]
where $\tilde{\bm z}_k$ is the $k$-th sample from $\tilde{\bm Z}$ and $\bm z_l$ is the $l$-th sample from $\bm Z$. Then, we construct the predictor
\[
    \hat{f}_i(\bm z) = \sum_{j=1}^N (\hat{\bm a}_i)_j k_i(\bm z, \tilde{\bm z}_j)\, ,
\]
for each $i \in \{0, 1\}$. Lastly, given the \texttt{torch.nn.Linear} and \texttt{torch.nn.ReLU} implementations from \texttt{PyTorch}, we implement the predictor $\hat{f}_{\text{NN}} =  \text{Linear}_{2}(\text{ReLU}(\text{Linear}_{1}(\boldsymbol{z})))$. The network was trained with learning rate obtained through $5$-fold cross validation from a logarithmic grid $\{10^{-3}, 10^{-2}, 10^{-1}, 10^0\}$, for every sample size $N$, with random state $123114$.

To obtain the metric from the figure, we calculate the Test Mean Squared Error by computing
\[
  \frac{1}{M} \sum_{i=1}^{M} \Big(y_i - \hat{f}_{p}(\tilde{\bm z}_i)\Big)^2
\]
for all $\hat{f}_{p} \in \{\hat{f}_{0}, \hat{f}_{1}, \hat{f}_{\text{NN}})$.

\subsection{Implementation, Software Stack and Hardware}
The experimental framework was implemented in \texttt{Python} (v3.10.12). We use \texttt{PyTorch} (v.2.11.0) \citep{Adam2019Pytorch}, such that layers and activation are out-of-the-box calls to the methods \texttt{torch.nn.Linear} and \texttt{torch.nn.ReLU}, specifying only the input, first layer, and output sizes. We utilized the \texttt{NumPy} (v.1.26.4) \citep{Harris2020numpy} and \texttt{SciPy} (v.1.15.3) \citep{2020SciPy} libraries for numerical linear algebra operations, specifically for the eigendecomposition (\texttt{scipy.linalg.eigh}) of the kernel matrices, and for the KRR implementation (\texttt{numpy.linalg.solve}). Lastly, for KFold validation we used \texttt{sklearn}'s (v.1.4.1.post1) \texttt{model\_selection.KFold} object \citep{2011Sklearn}.

All experiments were conducted on a laptop with a 11th Gen Intel(R) Core(TM) i7-11800H @ 2.30GHz processor, 8GB of RAM 3200 MHz and a single NVIDIA RTX 3060 Laptop GPU, and should be reproducible, even with a CPU, in less than an hour.